%% file: root.tex
\documentclass{article}

\usepackage{microtype}
\usepackage{amssymb,amsmath, amsthm}
\usepackage[pdftex]{graphicx}
\usepackage{subfigure}
\usepackage{caption}
\usepackage{wrapfig, blindtext}
\usepackage{booktabs} 
\usepackage{amssymb}
\usepackage{framed}
\usepackage{mathrsfs}
\usepackage{multirow}
\usepackage{enumerate,pifont,bigdelim,bbding}
\usepackage{wrapfig}
\usepackage{color}
\usepackage{bbm}

\usepackage[hidelinks]{hyperref}
\usepackage[toc,page,header]{appendix}
\usepackage{minitoc}

\newcommand{\normm}[1]{{\vert\kern-0.25ex\vert\kern-0.25ex\vert #1 \vert\kern-0.25ex\vert\kern-0.25ex\vert}}

\usepackage{hyperref}
\usepackage{bm}
\usepackage{algorithm,algorithmic,multirow,hhline}
\usepackage{enumitem}

\input{mysymbol.sty}
\input{kz_style.sty}

\usepackage{mathtools}

\usepackage{jmlr2e} 

\makeatother

\begin{document}
\doparttoc 
\faketableofcontents 

	\jmlrheading{1}{2022}{pp}{mm/dd}{mm/dd}{Amrit Singh Bedi, Anjaly Parayil, Junyu Zhang,   Mengdi Wang, and Alec Koppel}
	
	
	\ShortHeadings{Heavy-tailed Policy Search in Continuous Control}{Bedi, Parayil, Zhang,  Wang, Koppel}
	\firstpageno{1}
	
	\title{On the Sample Complexity and Metastability of Heavy-tailed Policy Search in Continuous Control}

	\author{%
		\name Amrit Singh Bedi* \email amritbd@umd.edu\\
		\addr 	University of Maryland,\\
		College Park, MD, USA
		\AND
		\name Anjaly Parayil* \email panjaly05@gmail.com\\
		\addr Microsoft Research India 
		\AND
				\name 	Junyu Zhang  \email junyuz@nus.edu.sg  \\
		\addr 	Department of Industrial Systems Engineering and Management \\
		National University of Singapore\\
		Singapore, 119077
		\AND
		\name Mengdi Wang  \email mengdiw@princeton.edu  \\
	\addr 	Department of Electrical Engineering \\
		Center for Statistics and Machine Learning\\
		Princeton University/Deepmind, 
		Princeton, NJ 08544
		\AND	
			\name Alec Koppel$\dagger$ \email alec.koppel@jpmchase.com\\
		\addr JP Morgan AI Research\\ 383 Madison Ave, New York, NY 10017 
	}
	\editor{}
		\thanks{Equal contributions.}
	\thanks{This work is supported by Northrup Grumman Seed Grant. Alec Koppel completed this work in his previous position.}
	\maketitle

	
	

\begin{abstract}
	Reinforcement learning is a framework for interactive decision-making with incentives sequentially revealed across time without a system dynamics model. Due to its scaling to continuous spaces, we focus on policy search where one iteratively improves a parameterized policy with stochastic policy gradient (PG) updates. In tabular Markov Decision Problems (MDPs), under \emph{persistent exploration} and suitable parameterization, global optimality may be obtained. By contrast, in continuous space, the non-convexity poses a pathological challenge as evidenced by existing convergence results being mostly limited to stationarity or arbitrary local extrema.
	To close this gap, we step towards persistent exploration in continuous space through policy parameterizations defined by distributions of heavier tails defined by tail-index parameter $\alpha$, which increases the likelihood of jumping in state space. Doing so invalidates smoothness conditions of the score function common to PG. Thus, we establish how the convergence rate to stationarity depends on the policy's tail index $\alpha$, a H\"{o}lder continuity parameter, integrability conditions, and an exploration tolerance parameter introduced here for the first time.
	Further, we characterize the dependence of the set of local maxima on the tail index through an exit and transition time analysis of a suitably defined Markov chain, identifying that policies associated with L\'{e}vy Processes of a heavier tail converge to wider peaks. This phenomenon yields improved stability to perturbations in supervised learning, which we corroborate also manifests in improved performance of policy search, especially when myopic and farsighted incentives are misaligned.
\end{abstract}
%
\input{Introduction}

\input{problem.tex}
\input{Algorithm.tex}
\input{convergence}
\input{Exittime_amrit_2.tex}

\input{Simulations_new.tex}

\input{Conclusion}

%
\bibliographystyle{plain}
\bibliography{RL_1}
%
\newpage
\renewcommand{\theHsection}{A\arabic{section}}

\appendix
\addcontentsline{toc}{section}{Appendix} 
\part{Appendix} 
\parttoc 

\input{Appendix.tex}

\input{Appendix_D.tex}
\input{Appendix_E.tex}
\input{Appendix_F.tex}

\end{document}

%% file: Introduction.tex

\section{Introduction}
In reinforcement learning (RL), an autonomous agent sequentially interacts with its environment and observes rewards incrementally across time \citep{sutton2017reinforcement}. This framework has gained attention in recent years for its successes in continuous control \citep{schulman2015high,lillicrap2016continuous}, web services \citep{zou2019reinforcement}, personalized medicine \citep{kosorok2015adaptive}, among other contexts. Mathematically, it may be described by a Markov Decision Process (MDP) \citep{puterman2014markov}, in which an agent seeks to select actions so as to maximize the long-term accumulation of rewards, known as the value. The key distinguishing point of RL from classical optimal control is its ability to discern control policies without a system dynamics model.

Algorithms for RL either operate by Monte Carlo tree search \citep{guo2014deep}, approximately solve Bellman's equations \citep{bellman57a,watkins1992q}, or conduct direct policy search \citep{williams1992simple}. While the first two approaches may have lower variance and converge faster \citep{even2003learning,devraj2017zap}, they typically require representing a tree or $Q$-function for every state-action pair, which is intractable in continuous space. For this reason, we focus on PG. 

Policy search hinges upon the Policy Gradient Theorem \citep{sutton2000policy}, which expresses the gradient of the value function with respect to policy parameters as the expected value of the product of the score function of the policy and its associated $Q$ function. Its performance has historically been understood only asymptotically \citep{konda1999actor,konda2000actor,bhatnagar2009natural} via tools from dynamical systems \citep{kushner2003stochastic,borkar2008stochastic}. 
More recently, the non-asymptotic behavior of policy search has come to the fore. In continuous space, its finite-time performance has been linked to stochastic search over non-convex objectives \citep{bottou2018optimization}, whose $\mathcal{O}(1/\sqrt{k})$ convergence rate to stationarity is now clear \citep{bhatt2019policy,zhang2020global}. However, it is challenging to discern the quality of a given limit point under this paradigm. 

By contrast, for tabular MDPs, i.e., those with state and action spaces defined by finite discrete sets, stronger results have appeared \citep{bhandari2019global,zhang2020sample,agarwal2020optimality}: linear convergence to \emph{global} optimality for tabular or softmax parameterizations.  
A critical enabler of these recent innovations in finite MDPs is a persistent exploration condition: the initial distribution over the states is uniformly lower bounded away from null, under which the current policy may be shown to assign strictly positive likelihood to the optimal action over the entire state space \citep{mei2020global}[Lemma 9]. {This concept of exploration is categorically different from notions common to bandits, i.e., optimism in the face of uncertainty \citep{thompson1933likelihood,lai1985asymptotically,jaksch2010near,russo2018learning}, and instead echoes persistence of excitation in systems identification \citep{narendra1987persistent,narendra2012stable}.}
Under this condition, then, a version of gradient dominance \citep{luo1993error} (known also as Polyak-\L ojasiewicz inequality \citep{lojasiewicz1963topological,polyak1963gradient}) holds, as derived in \citep{agarwal2019theory,mei2020global,mei2020escaping}. This result enables such global improvement bounds.
 Unfortunately, translating this condition to continuous space is elusive, as many common distributions in continuous space may fail to be integrable if their likelihood is lower bounded away from null over the entire (not necessarily compact) state space. Thus, the following question is our focus:

\emph{Can one nearly satisfy persistent exploration in MDPs over continuous spaces through appropriate policy parameterizations, and in doing so, mitigate the pathologies of non-convexity?} 

In this work, we step towards an answer by studying policy parameterizations defined by heavy-tailed distributions \citep{bryson1974heavy,focardi2003fat}, which includes the family of L\'{e}vy Processes common to fractal geometry \citep{hutchinson1981fractals,mandelbrot1982fractal}, finance \citep{taleb2007black,taylor2009black}, pattern formation in nature \citep{avnir1998geometry}, and networked systems \citep{clauset2009power}.
By employing a heavy-tailed policy, the induced transition dynamics will be heavy-tailed, and hence at increased likelihood of jumping to non-adjacent states.  That policies \citep{chou2017improving,kobayashi2019student} or stochastic policy gradient estimates \citep{garg2021proximal} associated with heavy-tailed distributions exhibit improved coverage of continuous space is well-documented experimentally. Here we seek a more rigorous understanding of in what sense this impacts performance may be formalized through \emph{metastability}, the study of how a stochastic process transitions between its equilibria. This marks a step towards persistent exploration in continuous space, but satisfying it precisely remains beyond our grasp. 

Historically, heavy-tailed distributions have been recently employed in non-convex optimization to perturb stochastic gradient updates by $\alpha$-stable L\'{e}vy noise \citep{gurbuzbalaban2020heavy,simsekli2020fractional}, inspired by earlier approaches where instead Gaussian noise perturbations are used \citep{pemantle1990nonconvergence,gelfand1991recursive}. Doing so has notably been shown to yield improved 
 stability to perturbations in parameter space since SGD perturbed by heavy-tailed noise can converge to local extrema with more volume, which in supervised learning is experimentally associated with improved generalization \citep{neyshabur2017exploring,zhu2018anisotropic,advani2020high}, and has given rise to a nascent generalization theory based on the tail index of the parameter estimate's limiting distribution \citep{NEURIPS2020_37693cfc,simsekli2020fractional}. 
Rather than perturbing stochastic gradient updates, we directly parameterize policies as heavy-tailed distributions, which induces heavy-tailed gradient noise. Doing so invalidates several aspects of existing analyses of PG in continuous spaces \citep{bhatt2019policy,zhang2020global}.
 Thus, our main results are:
 
 \begin{itemize}
    
  \item We present a few heavy-tailed policy parameterizations that may be used in lieu of a Gaussian policy for continuous space, which can prioritize selecting actions far from the distribution's center (Sec. \ref{sec:prob}), and discuss how policy search manifests for this setting (Sec. \ref{section:algorithm}); 
  
\item   We establish the attenuation rate of the expected gradient norm of the value function when the score function is H\"{o}lder continuous, and may be unbounded but whose moment is integrable with respect to the policy (Theorem \ref{Thm1}). This statement generalizes previous results that break for non-compact spaces \citep{bhatt2019policy,zhang2020global}, and further requires introducing an exploration tolerance parameter (Definition \ref{def:exploration_tolerance}) to quantify the subset of the action space where the score function is absolutely bounded;
 
\item In sec. \ref{subsec:exit_time}, by rewriting the PG under a heavy-tailed policy as a discretization of a L\'{e}vy Process, we establish that the time required to exit a (possibly spurious) local extrema decreases polynomially with heavier tails (smaller $\alpha$), and the width of a peak's neighborhood (Theorem \ref{Theorem:Exit_time}). Further, the proportion of time required to transition from one local extrema to another depends polynomially on its width, which decreases for smaller tail index (Theorem \ref{theorem: transitiontym}). By contrast, lighter-tailed policies exhibit transition times depending exponentially on the volume of an extrema's neighborhood;
 
\item Experimentally, we observe that policies associated with heavy-tailed distributions converge more quickly in problems that are afflicted with multiple spurious stationary points, which are especially common when myopic and farsighted incentives are in conflict with one another (Sec. \ref{sec:simulations}). 

 \end{itemize}
 
 \section{Additional Context and Related Work} 
 Efforts to circumvent the necessity of persistent exploration and obtain rates to global optimality have been considered in both finite and continuous space. In tabular settings, one may incorporate proximal-style updates in order to leverage a performance-difference lemma \citep{kakade2002approximately}, which has given rise to recent analyses of natural policy gradient \citep{schulman2017proximal,tomar2020mirror,lan2021policy}. Translating these results to the continuum remains an open problem. 
 
 Alternatively, in continuous space, one may hypothesize the policy parameterization is a neural network whose size grows unbounded with the number of samples processed \citep{wang2019neural,liu2019neural}. Doing so belies the fact that typically a parameterization has fixed dimension during training. Alternatively, one may impose a ``transferred compatible function approximation error" condition that mandates the ability to sample from the occupancy measure of the optimal policy to ensure sufficient state space coverage \citep{agarwal2019theory,liu2020improved}, which is difficult to perform in practice. 
 
 Two additional lines of effort are pertinent to the objective of this work. The first is state aggregation \citep{michael1995reinforcement}, in which one hypothesizes a large but finite space admits a representation in terms of low-dimensional features, such as tile coding \citep{sutton2017reinforcement} or interpolators \citep{tsitsiklis1996feature}. A long history of works seeks to discern such state aggregations adaptively \citep{bertsekas1989adaptive,michael1995reinforcement,dean1997model,jiang2015abstraction,duan2019state,misra2020kinematic}.
 
 Such representations can be used in, e.g., policy search \citep{agarwal2020pc,russo2020approximation} or value iteration \citep{arumugam2020randomized} to obtain refined convergence behavior that depends only on the properties of the representation rather than the underlying state or action spaces. Finding this representation is itself not necessarily easier than solving the original MDP, however. See, for instance, \citep{agarwal2020flambe,modi2021model}, where a variety of structural assumptions and representations are discussed. In this work, we assume such a feature map is fixed at the outset of training as part of one's specification of a policy parameterization.
 
 The other research thrust broadly related to this work is information-theoretic exploration that seeks comprehensive state-space coverage. The simplest way to achieve this goal is to simply replace the cumulative return with an objective that prioritizes state-space coverage, such as the entropy of the occupancy measure induced by a policy \citep{savas2018entropy,hazan2019provably,zhang2020variational}. This goal does not necessarily result in good performance with respect to the cumulative return, however. Alternatively, exploration bonuses in the form of upper-confidence bound \citep{lai1985asymptotically,jaksch2010near}, Thompson sampling \citep{thompson1933likelihood}, information-directed sampling \citep{russo2018learning}, among other strategies (see \citep{russo2018tutorial} for a thorough review), have percolated into RL in various forms.
 
 For instance, incorporating randomized perturbations/exploration bonuses into value iteration \citep{osband2016deep}, Q-learning \citep{jin2018q}, or augmenting a policy's variance hyper-parameters in policy search in a manner reminiscent of line-search for step-size selection \citep{papini2020balancing}. Alternative approaches based on Thompson sampling \citep{gopalan2015thompson,osband2017posterior} and various Bayesian models of the value function \citep{bellemare2017distributional,azizzadenesheli2018efficient} have been considered, as well as approaches which subsume exploration goals into the choice of the aforementioned state aggregator  \citep{agarwal2020pc,modi2021model}. Our approach contrasts with approaches that inject suitably scaled randomness into an RL update, by searching over a policy class that is itself more inherently random.

\blue{\textit{Notations}: All the norms $\|\cdot\|$ are Euclidean norm unless otherwise stated. }

%% file: problem.tex

\section{Markov Decision Problems}\label{sec:prob}

In RL, an agent evolves through states $s\in \mathcal{S}$ selecting actions $a\in\mathcal{A}$, which causes transitions to another state $s'$ to occur according to a Markov transition density $\mathbb{P}(s'|s,a)$ and a reward $r(s,a)$ is revealed by the environment to inform its merit. Formally, an MDP \citep{puterman2014markov} consists of the tuple $(\mathcal{S}, \, \mathcal{A}, \, \mathbb{P},\,r,\, \gamma )$, where continuous state $\mathcal{S} \subseteq \mathbb{R}^{q}$  and action spaces may be unbounded, i.e., Euclidean space in the appropriate dimension. We hypothesize that actions $a_t\sim \pi(\cdot |s_t)$ are selected according to a time-invariant distribution $\pi(a|s):= \mathbb{P}({a}_t=a| s_t=s )$ called a policy determining the probability of action $a$ when in state $s$.   Define the value as the average long-term accumulation of reward \citep{bertsekas2004stochastic}:
\begin{align}\label{eq:value_func}
V^{\pi}(s) = \mathbb{E} \bigg[ \sum_{t=0}^{\infty} \gamma^{t} r(s_t,a_t) \mid s_0=s, \pi  \bigg]. 
\end{align}
Moreover, $\gamma$ is a discount factor that trades off the future relative to the present, $s_0$ denotes the initial state along trajectory $\{s_u,a_u,r_u\}_{u=0}^\infty$, and we abbreviate the instantaneous reward as $r_t=r(s_t,a_t)$. 
In \eqref{eq:value_func}, the expectation is with respect to randomized policy $a_t \sim \pi(\cdot |s_t)$ and state transition dynamics $s_{t+1} \sim \mathbb{P}(.|s_t,a_t)$ over times $t\geq 0$. We further define the action-value (known also as $Q$) function $Q^{\pi}(s,a)$ as the value conditioning on an initially selected action:
$ Q^{\pi}(s,a)  =    \mathbb{E} \left[ \sum_{t=0}^{\infty} \gamma^{t}r_{t} \mid s_0=s, a_0=a, \pi  \right]$. 
We focus on policy search over policies $ \pi_{{\bbtheta}}(\cdot |s_t)$ parameterized by a vector ${\bbtheta} \in \mathbb{R}^d$, which we estimate via maximizing the expected cumulative returns \citep{sutton2017reinforcement}:
\begin{align}\label{eq:main_prob}
\max_{{\bbtheta}} J({\bbtheta}) := V^{{\pi}_{{\bbtheta}}}(s_0)
\end{align}
One difficulty in RL is that \eqref{eq:main_prob} is non-convex in parameters ${\bbtheta}$. Thus, finding a global optimizer is challenging even if the problem were deterministic. However, in the present context, the search procedure also interacts with the transition dynamics $\mathbb{P}(s'|s,a)$. Before delving into how one may iteratively and approximately solve \eqref{eq:main_prob}, we present a few representative policy  parameterizations.
\begin{example}[\textbf{Gaussian policy}]\normalfont \label{eg:gaussian_fixed_var}
	The Gaussian policy is written as 
	%
		%
		$\pi_{{\bbtheta}}(a|s) = \mathcal{N} (a|\phi(s)^\top {\bbx}, e^{y})\;$,
		%
	%
	where the parameters ${\bbtheta}=[\bbx,y]$ determine the mean (centering) of a Gaussian distribution at $\phi(s)^\top {\bbx}$, and $e^{y} \geq\delta_{0}$ is the variance parametrized by $y$ for some $\delta_0>0$. Here, $\phi(s)$ represents a feature map with $\|\phi(s)\|\leq S<\infty$ which maps continuous state $s$ to a higher-dimension, i.e., $\phi: \mathcal{S} \rightarrow \mathbb{R}^d$.
\end{example}

\begin{example}[\textbf{Moderate-tailed policy}]\label{eg:light_tailed}
A distribution whose tail decays at a slower rate than the Gaussian may better explore environment, which we define as 
$\pi_{{\bbtheta}}(a|s) =\frac{1}{\sigma A_\alpha}\exp\left(-{|a-\phi(s)^T\bbtheta|^\alpha}/{\sigma^\alpha}\right)$
%
with normalizing constant $A_\alpha:=\int\exp\left(-{|x|^\alpha}\right)dx <\infty$, tail index  $\alpha\in[1,2]$ determining the likelihood of tail events, and scale parameter $\sigma>0$.
\end{example}
We next introduce heavy-tailed policies, specifically,  L\'{e}vy processes called $\alpha$-stable distributions, which are historically associated with fractal geometry \citep{hutchinson1981fractals}, finance \citep{focardi2003fat}, and network science \citep{barabasi2003emergence}. 
\begin{example}[\textbf{L\'{e}vy Process Policy}]\normalfont \label{eg:alpha_stable}
	%
	Symmetric $\alpha$ stable, $\mathcal{S}\alpha\mathcal{S}$ distributions generalize Gaussians with $\alpha \in  (0, 2]$  as the tail index determining the decay rate of the distribution's tail  \citep{nguyen2019first}.  Denote random variable $\textbf{X} \sim \mathcal{S}\alpha\mathcal{S} (\sigma)$ with associated characteristic function $\mathbb{E}\left[ e^{i\omega \textbf{X}}\right]  = e^{-\sigma|\omega|^{\alpha}}$ and scale parameter $\sigma  \in (0, \infty) $. 	For non-integer (fractional) value of  $\alpha$, there is no closed form expression but the density decays at a rate $1/|a|^{1+\alpha}$,  and is referred to as fractal \citep{mandelbrot1982fractal}. In finance, $\mathcal{S}\alpha\mathcal{S}$ distributions have been associated with "black swan" events \citep{taleb2007black,taylor2009black}. For $\alpha=2$, it reduces to a Gaussian, and for  
	$\alpha=1$ it is a Cauchy whose parametric form is:
$		\pi_{{\bbtheta}}(a|s) = \frac{1}{e^y \pi (1+((a-\phi(s)^\top {\bbx})/e^y)^2) }$,
	where, ${\bbtheta}=[\bbx,y]$, $\phi(s)^\top {\bbx}$ is the mode of the distribution and $e^y$ is the scaling parameter. 
	%
	
\end{example}

With a few policy choices of detailed, we delve into their relative merits and drawbacks. Intuitively, policies that select actions far from a learned mean parameter over actions may better explore the space, which exhibits outsize importance when near and long-term incentives of the MDP are misaligned \citep{misra2020kinematic}. More formally, persistent exploration has been identified in \emph{tabular} MDPs as key to the ability to converge to the optimal policy using first-order methods \citep{agarwal2019theory,mei2020global,mei2020escaping} and avoid spurious behavior. Persistent exploration formally ensures that under any initial distribution over $s_0$ in \eqref{eq:value_func}, the {current policy} assigns strictly positive likelihood to the optimal action over the entire state space \citep{mei2020global}[Lemma 9], under which a version of gradient dominance (akin to strong convexity) holds (Lemma 8).  Interestingly, these results echo classical persistence of excitation in systems identification \citep{narendra1987persistent,narendra2012stable}. 
%
The stumbling block in translating these conditions from finite to continuous spaces is that many common distributions over {unbounded} continuous space may fail to be integrable if their likelihood is lower bounded away from null.
 As a step towards satisfying this condition, we seek to ensure that the induced transition dynamics under a policy are heavy-tailed, which increases the likelihood of jumping to cover more of the state space. Doing so may be accomplished by specifying a heavy-tailed policy (Example \ref{eg:light_tailed} - 
\ref{eg:alpha_stable}), whose likelihood approaches null slowly while still defining a valid distribution. 
That continuous space necessitates exploration to eventuate in suitable behavior may be illuminated through the Pathological Mountain Car (PMC) (cf. Fig. \ref{fig:env2}) introduced next,  where a car is between two mountains. 
%
  
{\bf \noindent Pathological Mountain Car.}
 The environment consists of two goal posts, a less-rewarding goal at $ s =2.667$ with a reward of $10$ and a bonanza at $s =-4.0$ of $500$ units of reward.\begin{figure}[t]
 	\centering
 	\includegraphics[scale=0.2]{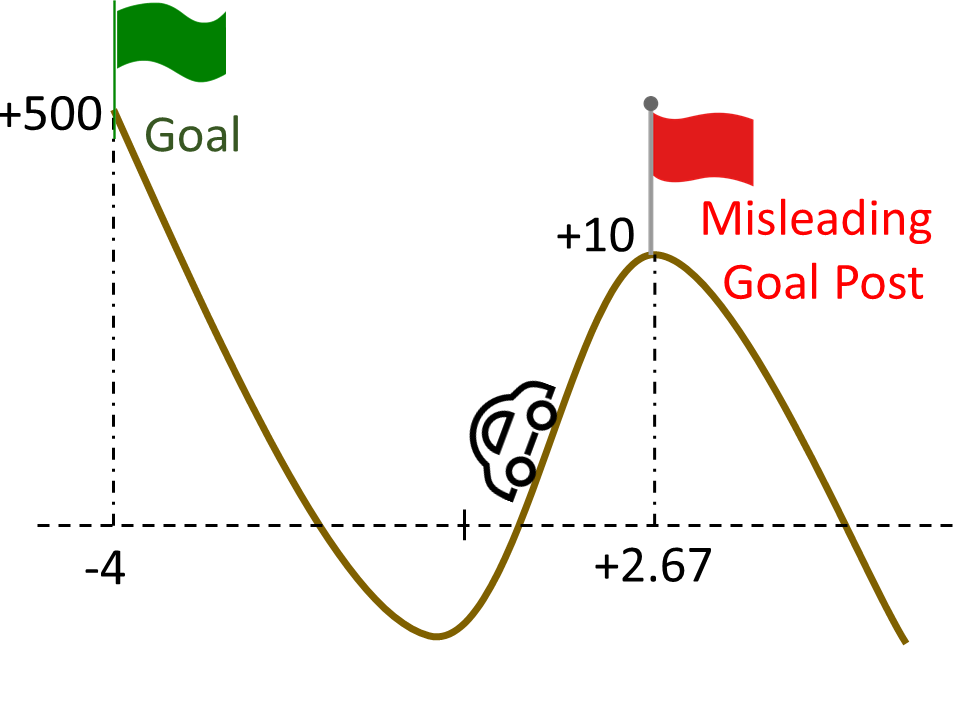}
 	\caption{ A continuous pathological Mountain Car environment with a low reward state and a bonanza atop a higher hill. Policies that do not incentivize exploration get stuck at the misleading goal.} \label{fig:env2}
 \end{figure} In Fig. \ref{fig:env2}, it is possible to get stuck at the lower peak and never reach the jackpot without sufficient exploration. Its potential pitfalls are illuminated experimentally in Sec. \ref{sec:simulations}. 

With the motivation clarified, we shift to illuminating that heavy-tailed policies, while encouraging actions far from the mean, may cause policy search directions to possibly be unbounded and non-smooth. These issues are the focus of Section \ref{section:algorithm}. 

%% file: Algorithm.tex

\section{Policy Gradient Methods}\label{section:algorithm}

\blue{Policy gradient (PG)} is an RL algorithm in which policy parameters in $\mathbb{R}^d$ are iteratively updated as approximate gradient ascent with respect to the value function \eqref{eq:value_func}. Its starting point is the Policy Gradient Theorem \citep{sutton2017reinforcement}, which expresses search directions in parameter space:
\begin{align}\label{equ:policy_grad_main}
	\nabla_{\bbtheta} J(\bbtheta)
	=\frac{1}{1-\gamma}\int_{\cS\times\cA} Q_{\pi_{\bbtheta}}(s,a) \cdot
	 \nabla_{\bbtheta} \log\pi_{{\bbtheta}}(a\given s)\cdot \rho_{{\bbtheta}}(s,a)\cdot dsda
\end{align}
where $\rho_{{\bbtheta}}(s,a)=\rho_{\pi_{\bbtheta}}(s)\cdot \pi_{{\bbtheta}}(a\given s)$ is a distribution called the \emph{discounted state-action  occupancy measure} defined as the product of the discounted state occupancy measure $\rho_{\pi_{\bbtheta}}(s)=(1-\gamma)\sum_{t=0}^\infty\gamma^t {p}(s_k=s\given s_0,\pi_{\bbtheta})$ and policy $\pi_{{\bbtheta}}(a\given s)$. In \citep{sutton2000policy}, both $\rho_{\pi_{\bbtheta}}(s)$ and $\rho_{{\bbtheta}}(s,a)$ are established as valid distributions. Despite this fact, the integral in \eqref{equ:policy_grad_main} may not exist due to the heavy-tailed nature of policy $\pi_{{\bbtheta}}(a\given s)$. Therefore, we first present some preliminaries regarding \eqref{equ:policy_grad_main}

%
\begin{assumption} \label{Assum:Q0}
	The absolute value of the reward is uniformly bounded,
		$\sup_{(s,a) \in \mathcal{S} \times \mathcal{A}} |R(s,a)|  \leq  U_R < \infty$ \; .
	%
\end{assumption}
%
%
\begin{assumption} \label{Assum:integral}
	 For any $s$, ${\bbtheta}$,
	$\int_{\cA}  \|\nabla_{\bbtheta} \log\pi_{{\bbtheta}}(a\given s)\|^{2}\cdot\pi_{{\bbtheta}}(a\given s) da \leq B<\infty$, {where $B$ is a finite constant.}
	%
\end{assumption}
Assumption \ref{Assum:integral} is weaker than the standard almost-sure boundedness of the score function assumed in prior work \cite[Assumption 3.1]{zhang2020global}, which is restrictive, and not valid even for Gaussians (Example \ref{eg:gaussian_fixed_var}). To see this, write the norm of its score function as 
$\|\nabla_{{\bbtheta} } \log \pi_{{\bbtheta}} (s,a)\|\leq \mathcal{O}\left({\|a\|+\|a\|^2}\right)$, which grows unbounded when the support of the action space is infinite. The score function also is unbounded in Example \ref{eg:light_tailed}. These subtleties motivate the relaxed condition in Assumption \ref{Assum:integral} which is valid regardless of a policy's tail index (Examples \ref{eg:gaussian_fixed_var} - \ref{eg:alpha_stable}). 

Next, we shift towards detailing PG. Under Assumptions \ref{Assum:Q0} - \ref{Assum:integral}, we establish that the integral in \eqref{equ:policy_grad_main} is finite (Lemma \ref{lemma:integral_bounded} in Appendix \ref{integral_bounded_proof}).
%
%
%
%
Thus, we employ it to compute search directions, which requires unbiased estimates of \eqref{equ:policy_grad_main}.
To realize this estimate, conduct a Monte Carlo rollout of length  $T_k\sim\text{Geom}(1-\gamma^{1/2})$, collect trajectory data $\tau_k =(s_0,a_0, s_1,a_1, \ldots, s_{T_{k}},a_{T_{k}})$, and form the PG estimate (akin to \citep{baxter2001infinite,liu2020improved}, except with a randomized horizon):
\begin{align}\label{eq:policy_gradient_iteration}
\!\!\!\!\!\hat{\nabla}_{\bbtheta} J({\bbtheta}_k) =&\sum_{t=0}^{T_k}\gamma^{t/2}\cdot R(s_t,a_t)\cdot \bigg(\sum_{\tau=0}^{t}\nabla\log\pi_{{\bbtheta}_k}(a_{\tau}\given s_{\tau})\bigg)  \;, \quad
		{\bbtheta}_{k+1} = {\bbtheta}_k + \eta \hat{\nabla}_{\bbtheta} {J}(\bbtheta_k) \; .
 \end{align}
where the parameter update for $\bbtheta_k$ is defined according to stochastic gradient ascent with step-size  $\eta>0$. The procedure for policy search along a trajectory is summarized as Algorithm \ref{Algorithm2}, where in the pseudo-code, we permit mini-batching with batch-size $B_k$, but subsequently assume $B_k=1$.
%
\begin{algorithm}[t]
	\begin{algorithmic}[1]
		\STATE \textbf{Initialize} :  policy parameters  ${\bbtheta}_0$, discount $\gamma$, step-size $\eta$, gradient $\bbg_0 =\bb0$, starting point $(s_0,a_0)$  \\
		\textbf{Repeat for $k=1,\dots$}
		\STATE Starting from $(s_0,a_0)$ , generate $B_k$ trajectories $\tau_{k,i} =(s_0,a_0, s_1,a_1, \ldots s_{T_{k,i}} , a_{T_{k,i}})$ of length $T_{k,i}\sim\text{Geom}(1-\gamma^{1/2})$ with actions $a_u\sim\pi_{{\bbtheta_k}}(.|s_u)$
		\STATE Compute policy gradient estimate $\bbg_k$ and update parameters $\bbtheta_k$:\vspace{-2mm} \small$$\!\!\!\!\!\bbg_k \leftarrow \frac{1}{B_k}\sum_{i=0}^{B_k} \bigg[\sum_{t=0}^{T_{k,i}}\gamma^{t/2}\cdot R(s_t,a_t)\cdot \bigg(\sum_{\tau=0}^{t}\nabla\log\pi_{{\bbtheta}_k}(a_{\tau}\given s_{\tau})\bigg)\bigg]\; , \quad 	{\bbtheta}_{k+1} \leftarrow {\bbtheta}_k + \eta \bbg_k\; $$ 
		\normalsize\vspace{-2mm}
		\STATE $k \leftarrow k+1$
		\textbf{ Until Convergence}
		\STATE \textbf{Return: ${\bbtheta}_k$} \\ 
	\end{algorithmic}
	\caption{\textbf{H}eavy-tailed \textbf{P}olicy \textbf{G}radient (HPG)}
	\label{Algorithm2}
\end{algorithm}

Next, we establish that the  stochastic gradient $\hat{\nabla}_{\bbtheta} {J}(\bbtheta)$ is an unbiased estimate of the true gradient ${\nabla}_{\bbtheta} {J}(\bbtheta)$ for a given $\bbtheta$. As previously mentioned, almost sure boundedness \cite[Assumption 3.1]{zhang2020global} of the score function $ \nabla_{\bbtheta}\log\pi_{{\bbtheta}}(a \given s) $ does not even hold for the Gaussian (Example \ref{eg:gaussian_fixed_var}), which motivates the moment condition in Assumption \ref{Assum:integral}. This alternate condition is employed to establish unbiasedness of \eqref{eq:policy_gradient_iteration}  formalized next (see Appendix \ref{apx_lemma:J_proof} for proof).

%
%
\begin{lemma}\label{lemma:J}
	Under the Assumptions \ref{Assum:Q0}-\ref{Assum:integral}, it holds that
	%
	$	\EE[\hat{\nabla}_{\bbtheta} J({\bbtheta})\given {\bbtheta}]={\nabla}_{\bbtheta} J({\bbtheta})$.
\end{lemma}
%
%
 An additional condition which is called into question of existing analyses of policy search is Lipschitz continuity \citep{zhang2020global,liu2020improved} of the score function.
%
%
In particular, heavy-tailed policies such as Examples \ref{eg:alpha_stable} -  \ref{eg:light_tailed} necessitate generalizing smoothness conditions to H\"{o}lder-continuity, as formalized next. 
\begin{assumption} \label{Assum:nablaF}
	The score function function $\nabla \log \pi_{{\bbtheta}}(\cdot)$ is H\"{o}lder continuous with constants $M>0$ and $ 0<\beta \leq 1 $, which implies that
	%
	 $\|\nabla \log \pi_{{\bbtheta}_1}(\cdot) -\nabla \log \pi_{{\bbtheta}_2}(\cdot) \| \leq M \|{\bbtheta}_1-{\bbtheta}_2\|^{\beta}$ for all  ${\bbtheta}_1, {\bbtheta}_2 \in \mathbb{R}^d$.
	%
\end{assumption}

%
%
%
%
Observe that the policy parameterization in Example \ref{eg:light_tailed} is not Lipschitz but H\"{o}lder continuous.
In the next section, we formalize the convergence of \eqref{eq:policy_gradient_iteration}, discerning the convergence rate to stationarity and metastability characteristics: the proportion of time the algorithm's limit points spent at wider versus narrower local extrema as a function of the tail index.

%% file: convergence.tex

\section{Convergence Analysis}\label{convergence}

 We analyze the ability of PG [cf. \eqref{eq:policy_gradient_iteration}] to maximize the value function \eqref{eq:main_prob}. As $J(\bbtheta)$ is non-convex in the policy parameter $\bbtheta$, the best pathwise result one may hope for is convergence to stationarity unless additional structure is present. Thus, we first study sample complexity in terms of the rate of decrease of the expected gradient norm  $\mathbb{E}[\|\nabla J(\bbtheta_k) \|]$, which we pursue under Assumptions \ref{Assum:integral}-\ref{Assum:nablaF} regarding the integrability of the norm of the score function with respect to the policy and H\"{o}lder continuity.  This generality is necessitated by heavy-tailed policy parameterizations as previously mentioned, and has not been considered in prior works such as \citep{sutton2000policy,zhang2020global,zhang2020sample,liu2020improved}. 
 
 Assumptions \ref{Assum:integral}-\ref{Assum:nablaF} present unique confounders to the RL setting that do not manifest in vanilla stochastic programming under relaxed smoothness conditions  \citep{shapiro2009lectures,nemirovski2009robust}. Specifically, they cause integrability and smoothness complications with respect to the occupancy measure $\rho_{{\bbtheta}}(s,a)$ induced by the MDP, which upends conditions on the objective and policy gradient in existing analyses. These complications are overcome in Lemmas \ref{Lemma:PGL} - \ref{Lemma:PGL2}, which first require partitioning the action space into sets where the score function is and is not almost surely bounded according to an exploration tolerance parameter (Definition \ref{def:exploration_tolerance}), which is unique to this work. Next, we make precise this discussion, establishing the convergence rate to stationarity of \eqref{eq:policy_gradient_iteration}.
 Later in this section, we formalize that iterates escape narrow extrema, and tend to jump towards wider peaks.

\subsection{Attenuation Rate of the Expected Gradient Norm}\label{subsec:stationarity}

We first focus on convergence rates to stationarity. To do so, we begin by establishing that 
Assumption \ref{Assum:nablaF} regarding the  H\"{o}lder continuity of the score function implies approximate H\"{o}lder continuity on the overall policy gradient. First, we partition the action space according to when the score function is almost surely bounded and where it is integrable according via a constant $\lambda>0$ defined next.
\begin{definition}\label{def:exploration_tolerance}{(\bf Exploration Tolerance)} Define as $\mathcal{A}(\lambda)$ the set of subsets of action space such that
		\begin{align}\label{define_Set}
			\mathcal{A}(\lambda):=\bigg\{\mathcal{C}\subseteq\mathcal{A} ~:~ \int_{\mathcal{\mathcal{A}\backslash\mathcal{C}}}\|\nabla\log\pi_{{\bbtheta}}(a | s)\| \cdot   \pi_{\bbtheta}(a|s) da \leq \lambda, \forall s, \bbtheta\bigg\}.
		\end{align}
	Then, $\lambda$ is the exploration tolerance parameter of a policy in an MDP with unbounded score function. \blue{Intuitively, $\mathcal{A}(\lambda)$ is the collection of all region of action space $\mathcal{C}\subseteq\mathcal{A}$, such that the expectation of score function under policy $\pi_{\bbtheta}(a|s)$ over region $\mathcal{A}\backslash\mathcal{C}$ is upper bounded by $\lambda$. And for the region in $\mathcal{A}(\lambda)$ associated with $\lambda$, we define an upper bound for the score function as }
	%
\begin{align}\label{definition_last}
	B(\lambda)= \inf_{\mathcal{C}\in\mathcal{A}(\lambda)} \sup_{(s,a)\in\mathcal{S}\times\mathcal{C}} \sup_{\bbtheta}\|\nabla\log\pi_{{\bbtheta}}(a | s)\|.  
\end{align}
\end{definition}
Definition \ref{def:exploration_tolerance} induces a tradeoff between the restriction on the range of values an action may take by a subset $\mathcal{C}
\subset\mathcal{A}$ with the scale of $B(\lambda)$. Observe that for the Cauchy (Example \ref{eg:alpha_stable}), constant $B(0)$ exists and is finite. A broader characterization of $\lambda$ as a function of the policy is given in Appendix \ref{apx_delta}.

\begin{lemma}\label{Lemma:PGL}
	Under Assumptions \ref{Assum:Q0} -  \ref{Assum:nablaF}, with $\lambda$ as in Definition \ref{def:exploration_tolerance}, the policy gradient \eqref{eq:policy_gradient_iteration} satisfies
\begin{align}
		\| {\nabla} _{{\bbtheta}} J ({\bbtheta}_1) - {\nabla} _{{\bbtheta}} J ({\bbtheta}_2 )  \| & \leq M_J\left[\|   {\bbtheta}_1 -{\bbtheta}_2  \|^{\beta}+\|   {\bbtheta}_1 -{\bbtheta}_2  \|+\lambda\right],
\end{align}
	for all ${\bbtheta}_1, {\bbtheta}_2 \in \mathbb{R}^d$ with $$	M_J:=\max\Bigg\{\frac{2U_R M }{(1-\gamma)^2}, \frac{M_QB^{1/2}}{1-\gamma} + \frac{U_R B(\lambda) M_\rho}{(1-\gamma)^2},\frac{2U_R}{(1-\gamma)^2}\Bigg\},$$ and $0<\beta\leq 1$. \blue{Here, $U_R$ denotes the reward upper bound from Assumption \ref{Assum:Q0}, $M_\rho= \frac{\sqrt{B}}{1-\gamma}$, and $M_Q=\frac{ {\gamma U_R M_{\rho}}}{1-\gamma}$. }

\end{lemma}
See Appendix \ref{lemma_1_proof} for proof. Lemma \ref{Lemma:PGL} generalizes a comparable statement regarding the  Lipschitz continuity of the score function typically imposed to establish a Lipschitz property of the policy gradient. Next, we provide an intermediate Lemma \ref{Lemma:PGL2} (see Appendix \ref{apx_lemma_2} for proof) crucial to establishing the main convergence rates to stationarity of Algorithm \ref{Algorithm2}.
\begin{lemma}\label{Lemma:PGL2}
Under Assumptions \ref{Assum:Q0} -  \ref{Assum:nablaF}, value function $J(\bbtheta)$ satisfies the smoothness condition
\begin{align}\label{eq:gradient_smoothness}
	\!\!\!\!\!\!\!\left| J(\bbtheta_1)-J(\bbtheta_2) \!-\!\left< \nabla J(\bbtheta_2),\bbtheta_1\!-\! \bbtheta_2\right> \right| \!\leq \!M_J \! \left[\|  {\bbtheta}_1 \!-\!{\bbtheta}_2 \|^{1+\beta}\!+\!\|  {\bbtheta}_1 -{\bbtheta}_2  \|^2 +\lambda\|  {\bbtheta}_1 -{\bbtheta}_2  \|\right]
\end{align}
for all $ \bbtheta_1, \, \bbtheta_2 \in \mathbb{R}^d$ and $M_J$ is as defined in Lemma \eqref{Lemma:PGL} with exploration tolerance $\lambda$ as in \eqref{define_Set}. 
\end{lemma}
Now, we formalize the convergence rate for Algorithm \ref{Algorithm2} as Theorem \ref{Thm1}. 
\begin{theorem} \label{Thm1}
Under Assumptions \ref{Assum:Q0}-\ref{Assum:nablaF}, with objective $J$  bounded above by $J^*$,  and H\"{o}lder continuity parameter $\beta$ bounded by the tail-index $\alpha$ as  $\beta \in (0, \alpha-1]$, under constant step-size selection $\eta = 1/ K^{\frac{\beta}{\beta+1}}$, the policy gradient updates  of  $\bbtheta_k$ in Algorithm \ref{Algorithm2} [cf. \eqref{eq:policy_gradient_iteration}] converges to stationarity:
%
%
	\begin{align}\label{eq:theorem_stationarity}
		\!\!\!\!\!\frac{1}{K} \!\!\sum_{k=0}^{K-1}\! \mathbb{E}\! \left[  \left\| {\nabla} J ({{\bbtheta}}_k)\right\|_2^2\right] \leq \frac{a_{\beta}}{K^{\frac{\beta}{1+\beta}}} +\mathcal{O}(\lambda), \; \quad a_{\beta} = \big( \left(L_J\right)^{1/(\beta+1)} \big) \left(J^* \!-\!J ({{\bbtheta}}_{1})\right)^{\beta/(\beta+1)}
	\end{align}
 with problem-dependent constant $L_J$ defined in \eqref{eq:theorem41_constant}, and exploration tolerance $\lambda$ as in Def. \ref{def:exploration_tolerance}.
\end{theorem}
Theorem \ref{Thm1} (proof in Appendix \ref{proof_theorem}) establishes that the iteration complexity of Algorithm \ref{Algorithm2} is $\mathcal{O}\big({1}/{\zeta^{1+({1}/{\beta})}}\big)$ {when $\lambda=\mathcal{O}(\zeta)$}, where $\zeta$ is the accuracy parameter. This result contrasts the standard rate of $\mathcal{O}\left({1}/{\zeta^{2}}\right)$ for non-convex optimization \citep{bottou2018optimization}, which restricts the policy parameterization to be Gaussian  \citep{bhatt2019policy,zhang2020global,liu2020improved}, i.e., $\alpha=2$. This means that heavy-tailed parameterizations result in slower convergence; however, we note that the rate of decrease of the expected gradient norm may not comprehensively encapsulate the non-convex landscape of value function. An additional subtlety is the effect of continuous action spaces, which are partitioned into sets where the score function is and is not bounded in accordance with the exploration tolerance parameter $\lambda$ (Def. \ref{def:exploration_tolerance}). In existing analyses of literature \citep{zhang2020global,paternain2020stochastic,liu2020improved}, the effect of $\lambda$ is assumed null { ($\lambda=0$)}, which overlooks the effect of action space coverage during policy search.
 Next, we establish that this perceived slower rate of heavy-tailed policies is overruled by their tendency towards local extrema with wider peaks, under a hypothesis that they admit a representation as a discretization of a L\'{e}vy Process.


%% file: Exittime_amrit_2.tex

\subsection{Metastability and Convergence to Wide Peaks}\label{subsec:exit_time}

{In the previous subsection, we established that the attenuation rate of the expected gradient norm for heavy-tailed policies is actually \emph{slower} than the rate associated smoother policies. This fact seemingly contradicts prior experimental results which demonstrate that they tend towards policies that achieve higher reward more quickly \citep{garg2021proximal}. The nature of this confounder has to do with the fact that expected gradient norm may only characterize how close a policy is to stationarity, but not how quickly a policy moves from one stationary point to another.}

To make sense of this quandary, we turn to characterizing (i) the time that Algorithm \ref{Algorithm2} takes to escape a (possibly spurious) local extremum, and (ii) how the proportion of time spent at a local maxima depends on its width and the policy's tail index. These results hinge upon introducing into  RL for the first time of \emph{metastability} of dynamical systems under the influence of weak random perturbations \citep{tzen2018local}. Similar results have been employed for SGD in the context of training neural networks in supervised learning \citep{nguyen2019first,gurbuzbalaban2020heavy}; however, it is unclear how one neural parameterization induces gradient noise whose distribution has a heavier from another. By contrast, here, this aspect is directly determined by the policy parameterization's tail index, which \emph{we choose} in Algorithm \ref{Algorithm2}. Moreover, in the aforementioned works, the analysis is only for the scalar-dimensional case, whereas here we consider dimension $d>1$.
%

We begin then by rewriting \eqref{eq:policy_gradient_iteration} in terms of the true policy gradient and the stochastic error $\hat{\nabla} J(\bbtheta_k)-\nabla {J}(\bbtheta_k)$, with the noise process hypothesized as an $\alpha$-tailed distribution, given by 
%
%
\begin{align} \label{eqn:SGD_L}
	\bbtheta_{k+1} = \bbtheta_k + \eta \nabla_{\bbtheta}  J(\bbtheta_k) + \eta^{1/\alpha}  \eta^{\frac{\alpha-1}{\alpha}}  S_k\;,
\end{align}
%
where, $S_k \in \mathbb{R}^d$  is   $ \mathcal{S}\alpha \mathcal{S} $ distributed random vector. Subsequently, we impose that the score function [cf. \eqref{equ:policy_grad_main}] is dissipative (Assumption \ref{Assump: disspitat}).

Hereafter, we rewrite discrete-time process $\bbtheta_k$ as $\bbtheta^k$ with superscript to disambiguate between continuous and discrete time. \eqref{eqn:SGD_L} holds under a hypothesis that the stochastic errors associated with policy gradient steps are heavy-tailed, which is observed experimentally in  \citep{garg2021proximal}. In Sec. \ref{sec:simulations}, we experimentally corroborate that policies induce gradient noise with a proportionate tail index  (Fig. \ref{figure_estimation_alpha}). 
\begin{table*}
\centering
		\vspace{-0mm}
		\begin{tabular}{|c|c|c|c|} 
			\hline
			Algorithm & Iter. complexity & Exit time (Def. \ref{def:exit_time0}) & Trans. time (Def. \ref{def:transition_time})\\
			\hline
			PG   & $\mathcal{O}\left({1}/{\zeta^{2}}\right)$  &  $\mathcal{O} \left(   e^{-2 J(a)/\epsilon^2} \right) $  
			&   $  \mathcal{O} \left( e^{2 (J(\bar{\bbtheta}_i)-J(\bar{\bbtheta}_i ))/\epsilon^{2}}\right) $  
			\\
			\hline
			HPG& $\mathcal{O}\left({1}/{\zeta^{1+\frac{1}{\beta}}}\right)$  & $\mathcal{O} \left(  \frac{\alpha}{2} \frac{a^{\alpha}}{\epsilon^{\alpha}} \right) $  & $ \mathcal{O} \left( 1/\epsilon^{\alpha}\right) $   \\
			\hline
		\end{tabular}
	\caption{ Summary of iteration complexity, exit time, and transition time results  for vanilla PG and heavy-tailed PG, with $\epsilon$ as the jump process coefficient, and $\zeta$ as accuracy parameter for $\mathbb{E}[\|\nabla J(\bbtheta_k) \|]\leq \zeta$. Employing a policy with a faster tail probability decay rate such as a Gaussian (larger $\alpha$) may take exponential time to escape a spurious local extrema, whereas a heavy-tailed policy escapes in polynomial time, as a function of the width $a$ of the set containing a local maxima \eqref{eqn:_Gi} and its escape direction \eqref{eqn:distm}.\vspace{-2mm}}
	\label{table_1}
\end{table*}
%
The continuous-time analogue of \eqref{eqn:SGD_L}, i.e., $(\bbtheta_{k+1}-\bbtheta_k)/\eta$ as $\eta\rightarrow 0$, defines Stochastic Differential Equation (SDE) driven by an $\alpha-$stable  L\'{e}vy process as \citep{tzen2018local}
%
\begin{align}  \label{eqn:SGD_L1}
	d\bbtheta_t^{\epsilon}  = \nabla_{\bbtheta} J(\bbtheta_t^{\epsilon}) dt +  \epsilon  d \bbL_t^{\alpha}\;,
\end{align}
where, $\epsilon: = \eta^{\frac{\alpha-1}{\alpha}}$  is a coefficient of the jump process (similar to diffusion coefficient in Brownian motion), and  $\bbL_t^\alpha$ denotes the multi-dimensional $\alpha$-stable L\'{e}vy motion in $\mathbb{R}^d$. 
%
%
%
%
With these details in place, we impose some additional structure (Assumption \ref{Assump:fun}) on the non-convex landscape of the objective $J(\bbtheta)$ in \eqref{eq:value_func}, namely, within the region of the objective's assumed finitely many local maxima, each one is separated by only a local minimum and no saddle points.
%
%
%
%
With the operating hypothesis that there are finitely many extrema of the objective, denote as   $\mathcal{G}_i \subset \mathbb{R}^d$  the neighborhood of the $i$-th local (arbitrary) maximizer $\bar{\bbtheta}_i$: 
\begin{align} \label{eqn:_Gi}
	\mathcal{G}_i:= \{ \mathbb{\bbtheta} \in \mathbb{R}^d: \| \bbtheta - \bar{\bbtheta}_i\| < a+ \xi \} \; , \qquad\partial \mathcal{G}_i:= \{ \mathbb{\bbtheta} \in \mathbb{R}^d: \| \bbtheta - \bar{\bbtheta}_i\| = a+ \xi \}
\end{align} 
where, $a, \, \xi >0$ are scalar radius parameters, and $\partial \mathcal{G}_i$ denotes the boundary of this neighborhood. \vspace{2mm}

{\bf \noindent Exit Time and Transition Time.} We next define the metastability quantities of exit and transition time in both continuous and discrete-time, assuming that  \eqref{eqn:SGD_L1}  and \eqref{eqn:SGD_L} are initialized at ${\bbtheta_0}\in\mathcal{G}_i$.

\begin{definition}\label{def:exit_time0}
	({\bf Exit time  from $\mathcal{G}_i$}) The time required  for the continuous-time process \eqref{eqn:SGD_L1} and discrete-time process \eqref{eqn:SGD_L}, respectively, to exit  $\mathcal{G}_i$ along standard basis vector $\bbr\in\mathbb{R}^d$  is defined by 
	%
	\begin{align}\label{exit_contins}
		%
		\hat{\tau}_{\xi,a}(\epsilon) \triangleq \inf \{ t \geq 0:  \bbtheta_t^{\epsilon} \in  \Omega_i^+(\bar{\delta})   \}, \qquad
		%
		%
		%
		%
		\bar{\tau}_{\xi,a}(\epsilon) \triangleq \inf \{ K \in \mathbb{N} :   \bbtheta^K  \in  \Omega_i^+(\bar{\delta})  \}.
	\end{align}
\end{definition}
Here, $a$ and $\xi$ denote scalar radius parameters (cf. \eqref{eqn:_Gi}). For all   $(\bbtheta - \bar{\bbtheta}_i)$ at a distance $\bar{\delta}$ from $\partial \mathcal{G}_i$, we define  its distance to  $\partial \mathcal{G}_i$ along standard basis vector $\bbr\in\mathbb{R}^d$, where $\bbr$ is as in \eqref{eqn:SGD_L1} with $\bbL_t^\alpha=\bbr L_t$ in terms of the lines in $\mathbb{R}^d$ as $g_{i_{\bbtheta}}(t)= {\bbtheta} -\bar{\bbtheta}_i+t \cdot \bbr $ for $t \in \mathbb{R}$. Then, for all $\|\bbtheta-\bar{\bbtheta}_i\|	 < \bar{\delta}, \, \bar{\delta} \in (0, a+\xi)$,  the distance function to the boundary along $\bbr$ is defined as   
\begin{align}\label{eqn:distm}
	d^+(\bbtheta):= \inf \{ t >0: g_{i_{\bbtheta}}(t) \in \partial \mathcal{G}_i \}\; ,\qquad 
	d^-(\bbtheta):= \sup \{ t <0: g_{i_{\bbtheta}}(t) \in \partial \mathcal{G}_i \},
\end{align}
where \eqref{eqn:distm}   define distance between   any point of interest and the  boundary of domain along  the  unit vector, $\bbr$, we have  $g_{i_{\bbtheta}}(t) \notin \mathcal{G}_i$ for $t \notin (d^-(\bbtheta), \, d^+(\bbtheta))$ for all $i$.  We say the point exits the domain $\mathcal{G}_i$ in the direction $\bbr$ when it enters the $\bar{\delta}$-tubes outside $\mathcal{G}_i$ defined by 
%
%
\begin{align} \label{eqn:omegatube}
	\Omega_i^+(\bar{\delta}):= \{\bbtheta \in \mathbb{R}^d: 	\|\left<(\bbtheta- \bar{\bbtheta}_i), \bbr \right>\bbr -(\bbtheta- \bar{\bbtheta}_i)\| 	 <\bar{\delta},\, \left<(\bbtheta- \bar{\bbtheta}_i), \bbr \right> > 0\} \cap \mathcal{G}_i^c.
\end{align}
We underscore that $\bar{\tau}(\cdot)$ represents the exit time of discrete-time process $\bbtheta^k$, whereas $\hat{\tau} (\cdot) $ denotes that of continuous-time stochastic process $\bbtheta_t^{\epsilon}$. 

\begin{definition}\label{def:transition_time}
	({\bf Transition time from $\mathcal{G}_i$ to $\mathcal{G}_j$})
	%
	Under the existence of a unit vector $\bbr$ along the direction connecting the domains $\mathcal{G}_i$ and $\mathcal{G}_{i+1}$ between two distinct local maxima, we define the transition time  from a neighborhood of one local maxima to another, i.e., from $\mathcal{G}_i$ to $\mathcal{G}_j, \, i\neq j$ in respective continuous-time [cf. \eqref{eqn:SGD_L1}] and discrete-time \eqref{eqn:SGD_L} 
	\begin{align}\label{eq:transition_time}
		\hat{T}_{i} (\epsilon) =  \inf \{   t>0: \bbtheta_{t}^{\epsilon}  \in \cup_{i \neq j} \mathcal{G}_j  \}\;,\qquad
		%
		%
		%
		%
		\bar{T}_{i} (\epsilon) = \inf\{   K \in \mathbb{N}: \bbtheta^{K}   \in \cup_{i \neq j} \mathcal{G}_j  \}\;. 
	\end{align}
\end{definition}

We begin by stating a technical assumptions which are required for the theorems presented in Section \ref{subsec:exit_time}. The first is regarding the non-convex landscape of $J$ and the later is regarding the L\'{e}vy jump process in \eqref{eqn:SGD_L1}.
\begin{assumption}\label{Assump:fun}
	Following statements holds for function, $J$:
	\begin{enumerate}
		%
		\item The set of local maxima of the value function $J$ consists of  $r$ distinct points
		$\{m_i\}=\{J(\bar{\bbtheta}_i)\}$ separated by $r-1$  local minima $\{s_i\}$.
		\item The function $J$ possesses the strict-saddle property,  i.e., all its local maxima satisfy $\nabla^2 J(\bbtheta) \prec 0$  and all its other stationary points satisfy $\lambda_{\min} (\nabla^2 J(\bbtheta))  > 0$. 
		\item The value function $J(\bbtheta)$ satisfies the growth condition; $J'(\bbtheta) > |\bbtheta|^{1+c}$ for $c>0$ and $|\bbtheta|$ sufficiently large, i.e. the function increases to infinity with infinite $\bbtheta$.
	\end{enumerate}
\end{assumption}

%
\begin{assumption}\label{Assump_Levy}
	\begin{enumerate}
		\item $L_0^\alpha$ = 0 almost surely.
		\item For $t_0 < t_1 < \ldots  < t_N$, the increments ($L_{t_i}^{\alpha}$
		) are independent ($i = 1, \ldots, N$).
		\item The difference ($L_t^{\alpha}-
		L_s^{\alpha}$) and $L_{t-s}^{\alpha}$ have the same distribution: $\mathcal{S}\alpha\mathcal{S}(t-s)^{1/\alpha}$ for $s<t$.
		\item $L_t^{\alpha}$
		is continuous in probability: for all $\delta> 0$ and $s \geq 0$, $\mathcal{P}(|L_t^\alpha -L_s^\alpha| > \delta   )  \to 0 \, \textrm{as} \,  t \to s $.
	\end{enumerate}
\end{assumption}

We also first present an additional condition we require on the score function.  
\begin{assumption}\label{Assump: disspitat}
	For some $m > 0$ and $b \geq 0$, $\nabla_{\bbtheta} \log \pi_{\bbtheta}(\cdot)$ is $(m,b,c)$-dissipative, which implies that
	$	c_{\alpha} \left< \bbtheta, \nabla_{\bbtheta} \log \pi_{\bbtheta}(\cdot) \right>  \geq    m \|\bbtheta\|^{1+c} -b$, for all $\bbtheta \in \mathbb{R}^d$.
\end{assumption}

We impose the following structural assumption on $\mathcal{G}_i$ [cf. \eqref{eqn:_Gi}] such that  desired properties for a domain perturbed by a L\'{e}vy noise  in multi-dimensional space   holds \cite{imkeller2010first}. 
\begin{assumption}\label{assmp:domaing}
	The following  assumptions hold for  $\mathcal{G}_i$:
	\begin{enumerate}
		\item We denote by $\Omega_i := \{ \bbtheta \in  \mathbb{R}^d: \bbtheta = t \cdot  \bbr_i \ \text{for a}\   t \in \mathbb{R}\}$ the straight
		line in the direction of $\bbr_i$. Let  $\nabla J(\cdot):
		\bar{\mathcal{G}} \to \mathbb{R}^d$ and  the set $\bar{\mathcal{G}} \cap \Omega$ is connected. There exists  numbers $a, b\,>0$  and a closed interval
		$I: = [-b,a]$ such that for all $t \in (-b, a)$ we have: $t\cdot \bbr \in \mathcal{G}_i$.  Since $\bar{\bbtheta}_i \in \mathcal{G}$, $\mathcal{G}_i$ is open, and $\mathcal{G}_i \cap \Omega \neq \emptyset $. 
		\item The boundary of $ \mathcal{G}_i$ defined by  $\partial \mathcal{G}_i$ is a $\mathcal{C}^1$-manifold so that the vector field $n$ of the outer normals on
		the boundary exists. We assume 
		%
			%
			$\left<   \nabla J(\bbtheta),  n(\bbtheta) \right>  \leq -\frac{1}{C}$,
			%
		%
		for all $\bbtheta \in \mathcal{G}_i$. This means that $\nabla J(\cdot)$  points into $\mathcal{G}_i$.
		\item Local extrema, $\bar{\bbtheta}_i$ is an attractor of the domain, i.e.   for every starting value $\bbtheta \in \mathcal{G}_i$, the
		deterministic solution  vanishes asymptotically.
		\item There exists atleast one set of  domains,  $\mathcal{G}_{i-1}$,  $\mathcal{G}_{i}$  and $\mathcal{G}_{i+1}$  such that $\mathcal{G}_{i-1}$,  $\mathcal{G}_{i}$  and $\mathcal{G}_{i+1}$ are connected,  $\partial \mathcal{G}_i \cap \partial \mathcal{G}_{j} \neq \emptyset, $ $j \in \{i, i-1\}$.    We  assume   existence of a  local minima  at the intersection of $\partial \mathcal{G}_i$ and $\partial \mathcal{G}_{j}$.
		\item There exists a discrete instant $K$ such that exit time $\hat{\tau}_{\xi,a}$ [c.f. \eqref{exit_contins}] greater than $K \eta,\, K>0$ and $\bbtheta_{\hat{\tau}}^{\epsilon}  \in \Omega^+$ for $\hat{\tau}_{\xi,a}(\epsilon) \geq K \eta$. 
	\end{enumerate}
\end{assumption}
%
\blue{Assumption \ref{Assump:fun} is regarding the level sets of the value function within the vicinity of stationary points versus local extrema. Assumption \ref{Assump:fun}.1 ensures that there is positive volume separating distinct extrema, which imposes that the value function, and hence reward, cannot be extremely similar for policies whose relative merits are different. Observe that the strict saddle property (Assumption \ref{Assump:fun}.2) has been studied before in the context of policy gradient method, as it is a sufficient condition for the correlated negative curvature condition \cite{zhang2020global}, which holds whenever the policy parameterization is associated with a positive definite Fischer information matrix, and the reward function is strictly positive or strictly negative \cite[Assumption 4.5]{zhang2020global}. Assumption \ref{Assump:fun}.3 is easy to satisfy for any policy that does not threshold large values of the derivative, such as the Gaussian or Cauchy -- direct calculation reveals that it holds for these cases, but it does not hold for a \emph{truncated} Gaussian.   

Assumption \ref{Assump_Levy} imposes conditions on the L\'{e}vy processes that drive the heavy-tailed noise. Theoretically they are difficult to verify, but we note that they are strictly more general than standard assumptions in the ODE analysis of stochastic approximation that underlies the stability analysis of reinforcement learning -- see \cite{borkar2000ode}. Moreover, we empirically verify that the noise satisfies the conditions required to be jump process with index $\alpha$ in Figure \ref{figure_estimation_alpha}, due to the fact that if the gradient is heavy tailed, then the noise associated with the stochastic errors is heavy-tailed.

Assumption \ref{Assump: disspitat} holds for any policy parameterization which is an increasing function of the norm. Observe that it holds for the policy in Example \ref{eg:light_tailed} directly when the policy parameter $\theta$ lies in compact space. Assumption \ref{assmp:domaing} imposes structure on the landscape of the value function. Assumption \ref{assmp:domaing}.2 imposes that the gradient is negatively correlated with the normal vector pointing away from a neighborhood of a stationary point, which usually holds. Assumption \ref{assmp:domaing}.3 ensures that the gradient is null near a local extrema, i.e., the policy gradient becomes null at a local extrema. Assumption \ref{assmp:domaing}.4 imposes that there is some intersection between the neighborhoods of extrema, which means that one locally optimal policy may have similar cumulative return to another of comparable quality. Assumption \ref{assmp:domaing}.5 imposes that the transition time between the neighborhoods of local extrema is governed by choice of learning rate up to a constant factor, which typically holds in practice.  } 

The following theorems present the first exit time and transition time probabilities of  the proposed heavy-tailed policy gradient setting, \eqref{eq:policy_gradient_iteration} when initialized within  $\mathcal{G}_i \subset \mathbb{R}^d$ such that \eqref{eqn:_Gi} holds.
%
%
\begin{theorem}(Exit Time Dependence on Tail Index) \label{Theorem:Exit_time}
	Suppose Assumptions \ref{Assum:Q0}- \ref{assmp:domaing} hold, the value function $J$  is initialized near local maxima $\bar{\bbtheta}_i$, and the policy gradient update in \eqref{eq:policy_gradient_iteration} is run under a heavy-tailed policy parameterization that induces tail index $\alpha$ in its stochastic error \eqref{eqn:SGD_L}. Then, the likelihood of its exit time from neighborhood $\mathcal{G}_i$ [cf. \eqref{eqn:_Gi}] of $\bar{\bbtheta}_i$ larger than $K$ is upper bounded as 
	%
	\begin{align}\label{eq:exit_time_theorem}
		\!\!\! \mathcal{P}^{\bbtheta_0} \!\left(  \bar{\tau}_{0,a} (\epsilon) >\!K\right) 
		&\leq ({2}/{\epsilon^{\rho \alpha}}) \epsilon^{\alpha} (d^+)^{-\alpha}   \!+\!   \mathcal{O} \big(  \big({d} K \eta^{2-(1/\alpha)}\big)^{\beta} \big) \!\!\! \nonumber 
		\\ & \quad+ \mathcal{O}  \big(1- \big(1-C_{\alpha}d^{1+({\alpha}/{2})}\eta \exp{(\alpha M_J\eta)}\epsilon^{\alpha}\left(({\xi}/{3})\right)^{-\alpha} \big)^{K} +\delta   \big)  
	\end{align}
	%
	%
	with initialization $\bbtheta_0$, $d^+$ [cf. \eqref{eqn:distm}]  denotes distance  between $\bbtheta_0$ and the boundary of $\mathcal{G}_i$,  $\rho\in (0,1)$ is a positive constant and $\bbtheta\in\mathbb{R}^d$. Moreover, the H\"{o}lder continuity constant satisfies $\beta \in (0,\alpha-1)$, $\delta >0$,  $\xi>0$,  $\eta$ is the step-size,  $k_1: = 1/\eta^{\alpha-1}$, and $\epsilon$ [cf. \eqref{eqn:SGD_L1}] is the jump process coefficient.
\end{theorem}
See proof in Appendix \ref{Proof:exittime}.
Observe that as $\epsilon \to 0$ in  \eqref{eq:exit_time_theorem}, the right hand side of \eqref{eq:exit_time_theorem} depends on the distance of $\bbtheta_0$ from the boundary $\partial \mathcal{G}$ and tail-index $\alpha$. Further, the dependence on $d^+$ (cf. \eqref{eqn:distm}) implicitly hinges upon the width $a$ of the neighborhood of the extrema \eqref{eqn:_Gi}, which noticeably decreases with heavier tails (smaller $\alpha$), meaning that heavier-tailed policies increase the likelihood of escape and tend towards wider maxima. The intricacy of the expression precludes easy interpretation. Thus, consider the average exit time for the single dimensional case in Table \ref{table_1}, in which there exists only a single direction of exit, which coincides with  \citep{imkeller2006first,nguyen2019first}. In contrast to proposed heavy-tailed setting wherein exit time is a function of width of the  neighborhood, exit time for PG under, e.g., a Gaussian parameterization, depends exponentially on the value at the extrema. 
Next, we discuss the transition time from one extrema to another.
\begin{theorem} (Transition Time Dependence on Tail Index)\label{theorem: transitiontym}
	Suppose Assumptions \ref{Assum:Q0}- \ref{assmp:domaing} hold and the value function $J$  is initialized near a local maxima $\bar{\bbtheta}_i$. Then in the limit $\epsilon \to 0$,  the policy gradient update in \eqref{eq:policy_gradient_iteration} under a heavy-tailed parameterization with tail index $\alpha$ associated with its induced stochastic error \eqref{eqn:SGD_L}, transitions from $\mathcal{G}_i$ to the boundary of $(i+1)$-th local maxima with probability, $\mathcal{P}^{\bbtheta_0}( \bbtheta^{k} \in \Omega_i^+ (\bar{\delta}) \cap \partial \mathcal{G}_{i+1}) $  lower bounded as a function of tail index $\alpha$:
	\begin{align} \label{eqn:thm4}
		\lim_{\epsilon \to 0} \mathcal{P}^{\mathbf{\theta}_0}( {\mathbf{\theta}}^{k} \in \Omega_i^+ (\bar{\delta}) \cap \partial \mathcal{G}_{i+1}) \geq \frac{d_{ij}^{-\alpha}}{((d_{ij}^+)^{-\alpha} + (-d_{ij}^-)^{-\alpha}}  - \delta,
	\end{align}
	where $\delta >0$, escape distance from extrema are defined as $d_{ij}^+(\bbtheta)\!:=\! \inf \{ t\! >\!0\!:\! g_{{i}_{\bbtheta}}(t) \!\in\! \Omega_i^+ (\bar{\delta}) \cap \partial \mathcal{G}_{i+1} \!\}$ and $d_{ij}^-(\bbtheta)\!:=\! \sup \{ t \!<\!0\!:\! g_{{i}_{\bbtheta}}(t) \!\in\! \Omega_i^- (\bar{\delta}) \cap \partial \mathcal{G}_{i-1} \!\}$, and $\Omega_i^+(\bar{\delta})$ is defined before Def. \ref{def:transition_time}.
\end{theorem}

Similar to exit time, the transition time probability [cf. \eqref{eqn:thm4}] (proof in Appendix \ref{sec:proofthm4}) depends on the width of boundary and the tail index, which noticeably also decreases for heavier tails (smaller $\alpha$), and depends on the width of the neighborhood containing a local maxima. For ease of interpretation, the single-dimensional case  for both vanilla PG and HPG are given in Table \ref{table_1}.  Transition times are asymptotically exponentially distributed in the limit of small noise and scale with $1/\epsilon^{\alpha}$ for HPG, whereas transition time for Brownian is exponentially distributed with  $\epsilon^\alpha$ replaced by \emph{exponential dependence} $e^{2 J(\cdot)/\epsilon^2}$ for a Gaussian policy.
\begin{figure}
	\centering
	\subfigure[ 1D Mario environment.]
	{\includegraphics[scale=0.16]{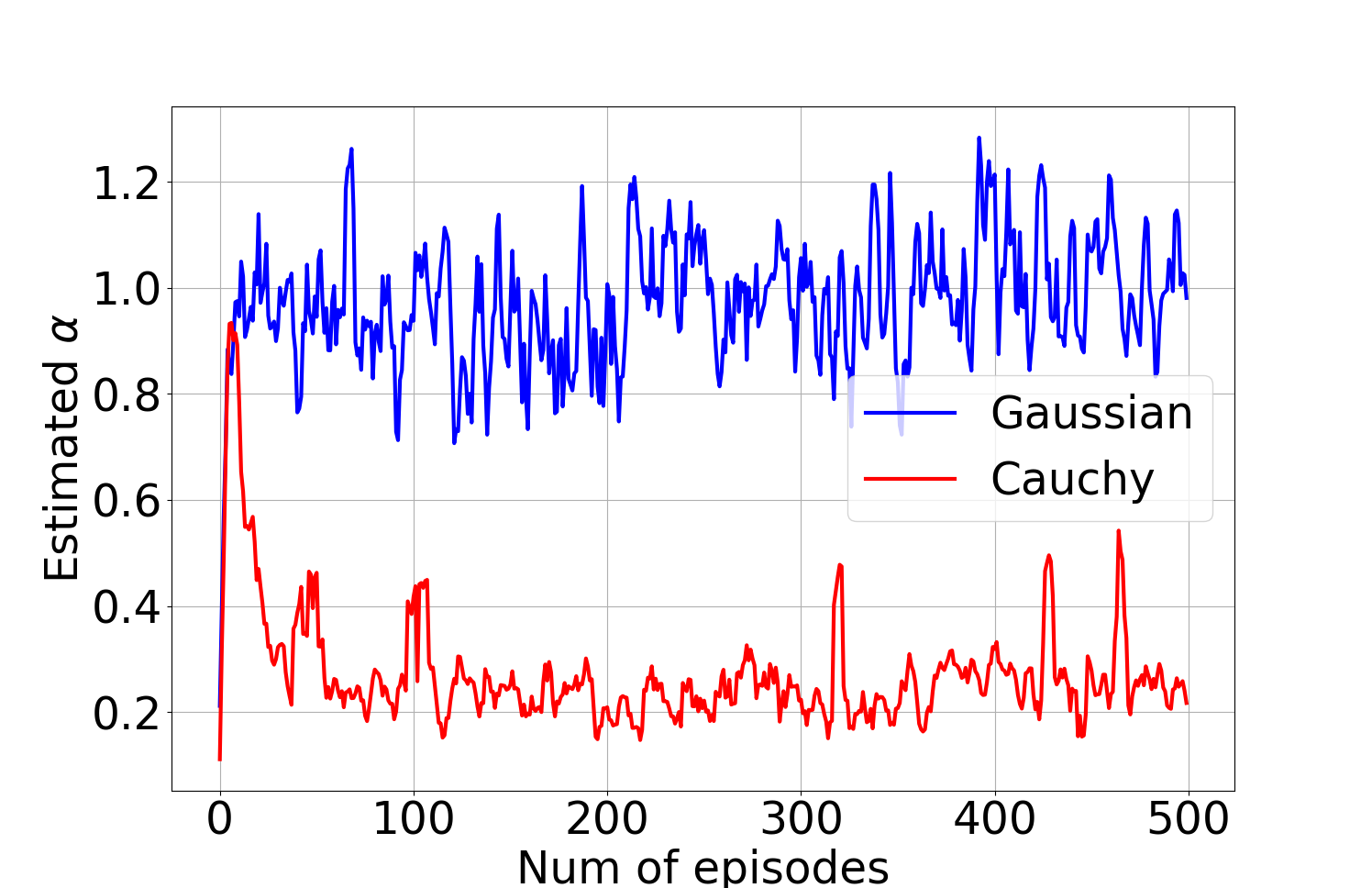}\label{fig:30}}
	\subfigure[ PMC (Sec. \ref{sec:prob})]
	{\includegraphics[scale=0.16]{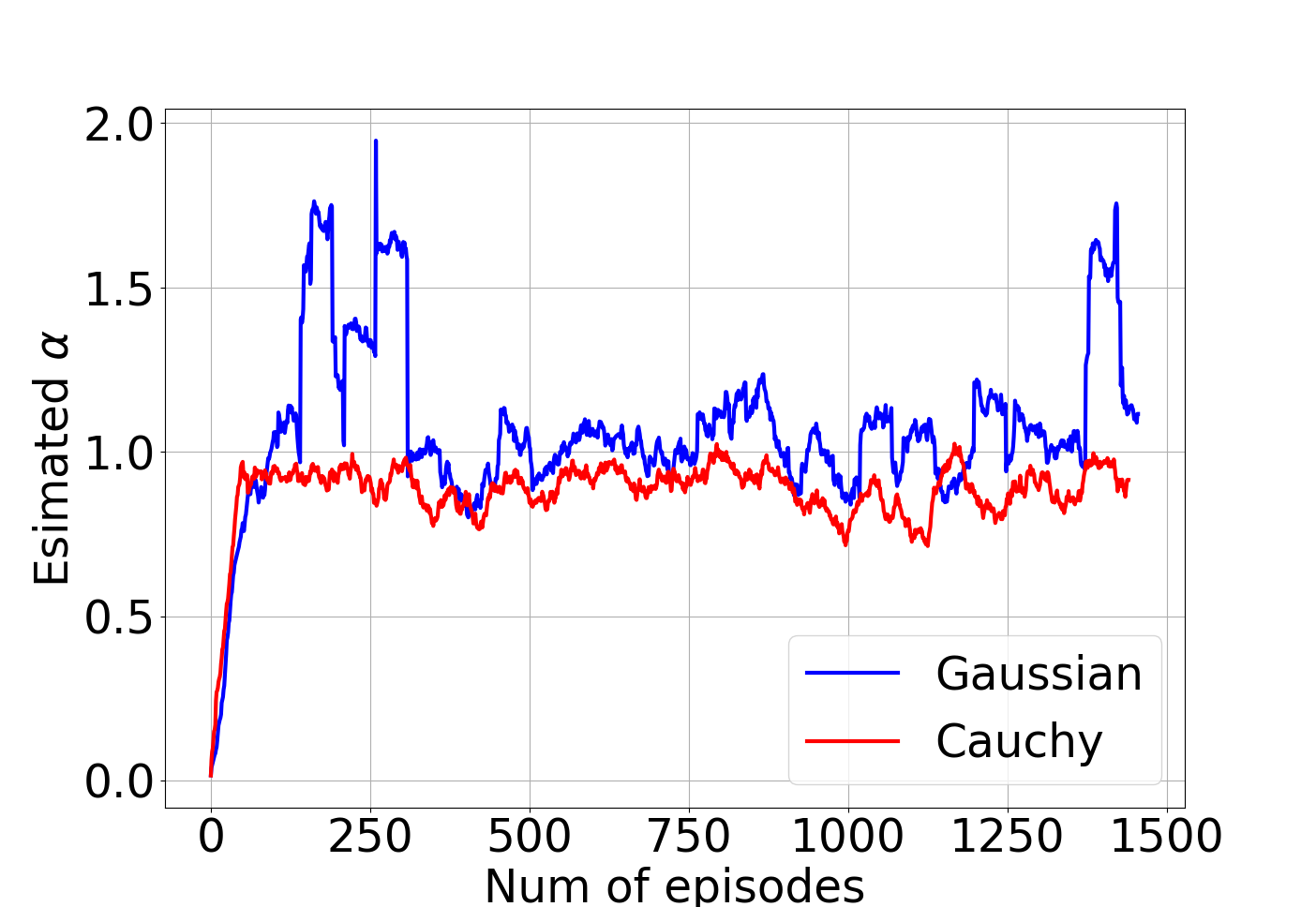} \label{fig:3}}
	\caption{ {\bf({a}) } Tail index  estimation of HPG updates for a 1D Mario \citep{matheron2019problem}.  
		{\bf({b}) }Tail index  estimation for Pathological Mountain car. In both, estimates are averaged over latest 50 episodes. Observe that a Cauchy policy induces a tail index lower than the Gaussian policy, and the volatility of the blue sample path stems from training being uncompleted during estimation.}
	\label{figure_estimation_alpha}
\end{figure}
Thus, in the small noise limit, Brownian-motion driven PG needs exponential time to transition from one peak to another, whereas the L\'{e}vy-driven process requires polynomial time, illuminating that heavy-tailed policies quickly jump away from spurious extrema.

%% file: Simulations_new.tex
\section{Experiments} \label{sec:simulations}

In this section, we evaluate the proposed HPG (Algorithm \ref{Algorithm2}) as compared to some common approaches for policy search. Before doing so, we demonstrate experimentally evidence that the heavy-tailed policies results in heavy tailed policy gradients. Then, we provide experiments for the Pathological Mountain Car (PMC) (Sec. \ref{sec:prob}) and 1D Mario environment \cite{matheron2019problem}. 
 For PMC, we consider an incentive structure in which the amount of energy expenditure, i.e., the action squared, at each time-step is negatively penalized and the reward structure is given by
 %
 $$ r(s_t,a_t) = -  a_t^2\mathbbm{1}_{\{-4.0 < s < 3.709, s \neq 2.667\}} + (500 -  a_t^2)\mathbbm{1}_{\{s = -4.0\}}+ (10-   a_t^2)\mathbbm{1}_{\{s = 2.67\}}.$$
 %
 %
 Here,  $s$ denotes the state space, and the action $a_t$ is a one-dimensional scalar representing the speed of the vehicle $\dot{s}_t$. 
 \begin{figure}[h]
 	\centering
 	\includegraphics[scale=0.07]{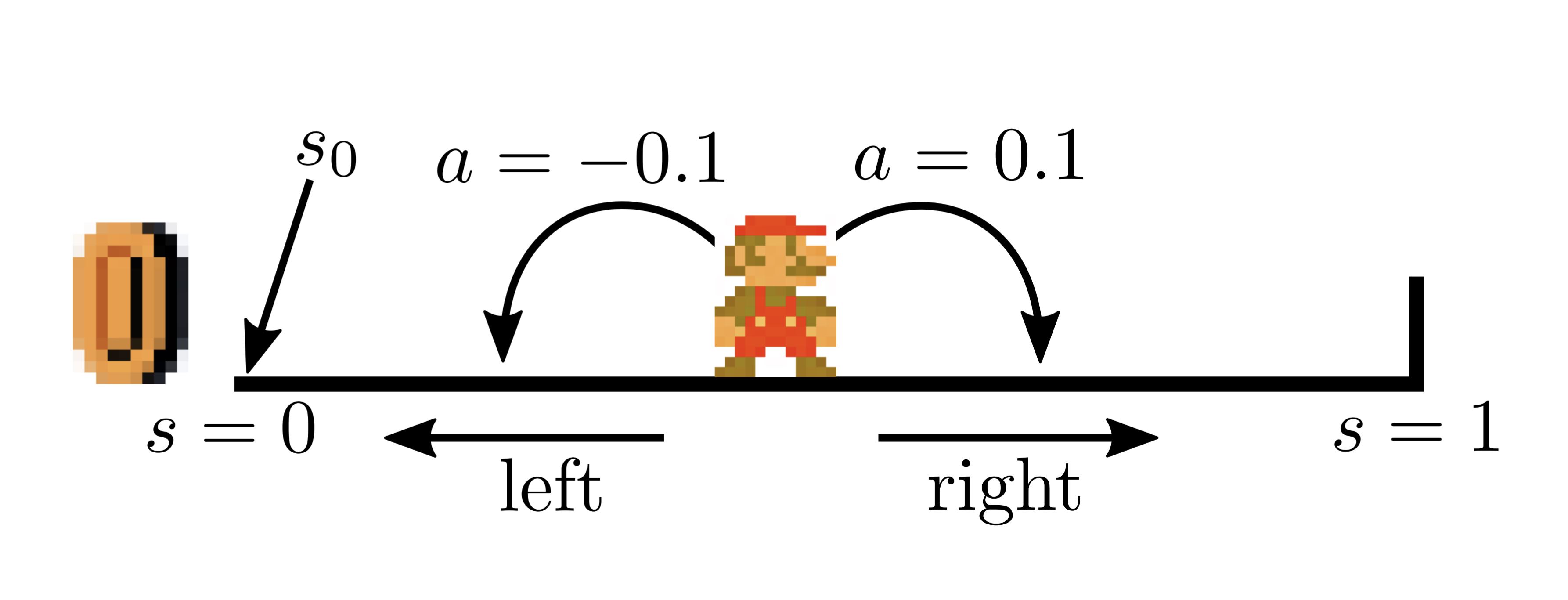}
 	\caption{ 1D Mario environment \cite{matheron2019problem}.}
 \end{figure} In 1D Mario environments, state $s\in[-4.0, \,  3.709]$ and the actions are confined to $[-20, \, 20]$. On the other hand, as the name suggests, the 1D Mario environment is one-dimensional with continuous state and action spaces with incentive structure  and state transition  defined as 
 $r(s_t,a_t)=\mathbbm{1}_{\{s_t+a_t <0\}}$, and $s_{t+1}=\min\{1,\max{\{0,s_t+a_t\}}\}$ 
 where, state, $s\in [0,1]$ and action $a\in[-0.1, 0.1]$. Each episodes are initialized at $s_0=0$.  

{Before presenting the experiments, first in Fig. \ref{figure_estimation_alpha}, we depict the estimation of tail index $\alpha$ (using method in \cite{mohammadi2015estimating}) for gradient estimates [cf. \eqref{eq:policy_gradient_iteration}] with a Cauchy and Gaussian policy. The lower the value of $\alpha$ the heavier the tail is of the policy gradient. In Fig. \ref{fig:30}, we observe that the average estimate for the Gaussian policy settles to a value of one, while  the corresponding value for Cauchy values settles around $0.2$ for 1D Mario environment. A similar plot for PMC is in  Fig. \ref{fig:3}: note that the tail-index estimate of Cauchy settles around unity and the corresponding value for Gaussian exhibits volatility since the policy has yet to converge.}  For the tail index estimation, we utilized the logic presented in \cite{mohammadi2015estimating} for the $\alpha$ estimation reiterated here in the form of Theorem \ref{them_appendix_exp} for quick reference. 
\begin{theorem} \label{them_appendix_exp} \cite{mohammadi2015estimating}
	Let  $\{ \bbX_i\}_{i=1}^{K}$ be the collection of random variables with $\bbX_i \sim \mathcal{S}\alpha\mathcal{S}(\sigma)$ and $K =K_1 \times K_2$. Define $Y_i \triangleq \sum_{i=1}^{K_1} \bbX_{j+(i-1)K_{1}}$ for $i \in [1, K_2].$ Then the estimator
	\begin{align}
		\hat{\frac{1}{\alpha}}  \triangleq \frac{1}{\log K_1} \left( \frac{1}{K_2}   \sum_{i=1}^{K_2}  \log |\bbY_i|  - \frac{1}{K} \sum_{i=1}^{K} \log |\bbX_i| \right) 
	\end{align}
	converges to $\frac{1}{\alpha}$ almost surely as $K_2 \to \infty$. 
\end{theorem}

\begin{figure*}[t!]
	\centering
	\subfigure[Constant scale $\sigma$ PMC]
	{\includegraphics[scale=0.14]{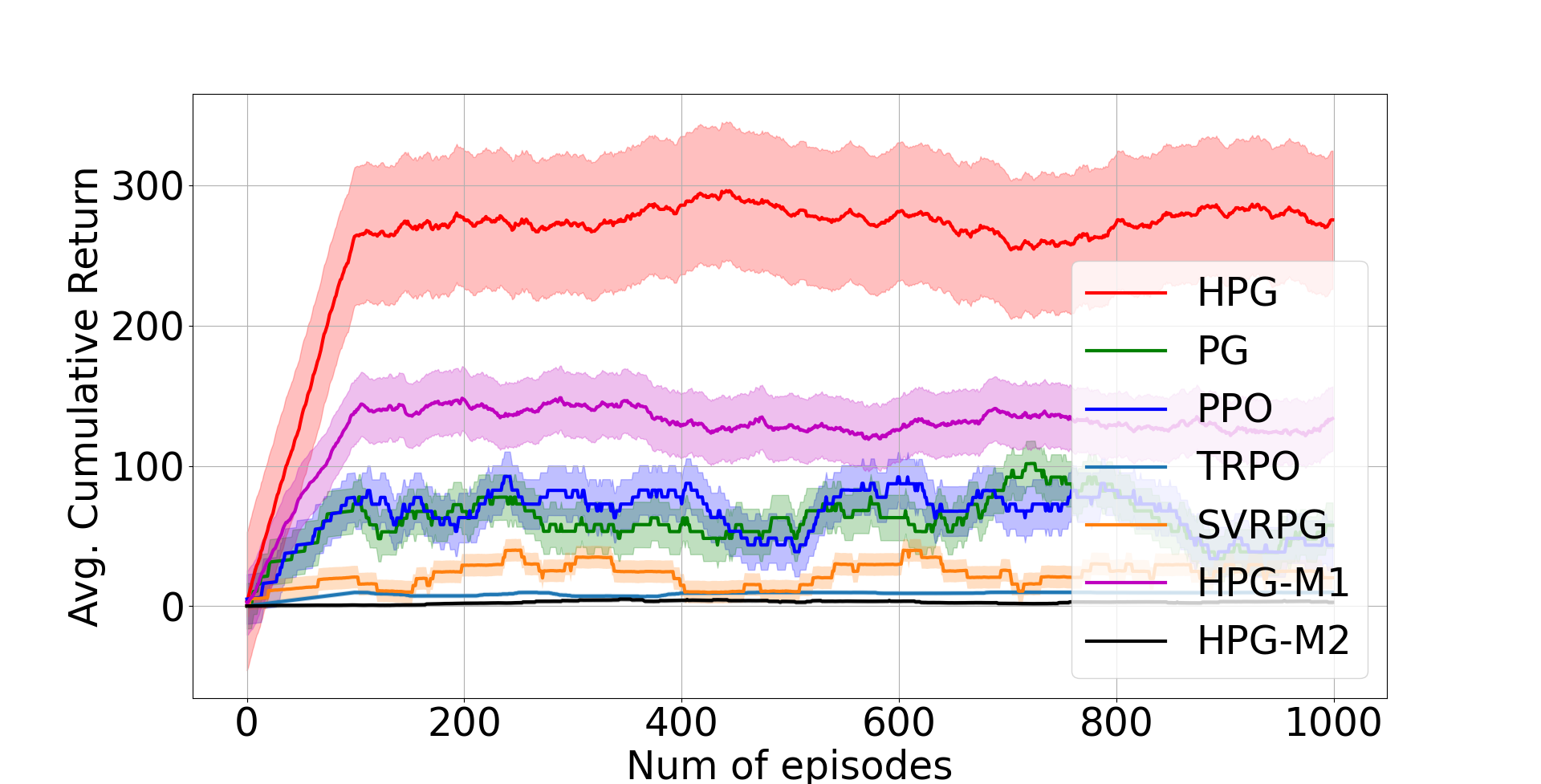} \label{fig:00}}\hspace{-2mm}
	\subfigure[Tuned scale $\sigma$ PMC.]
	{\includegraphics[scale=0.14]{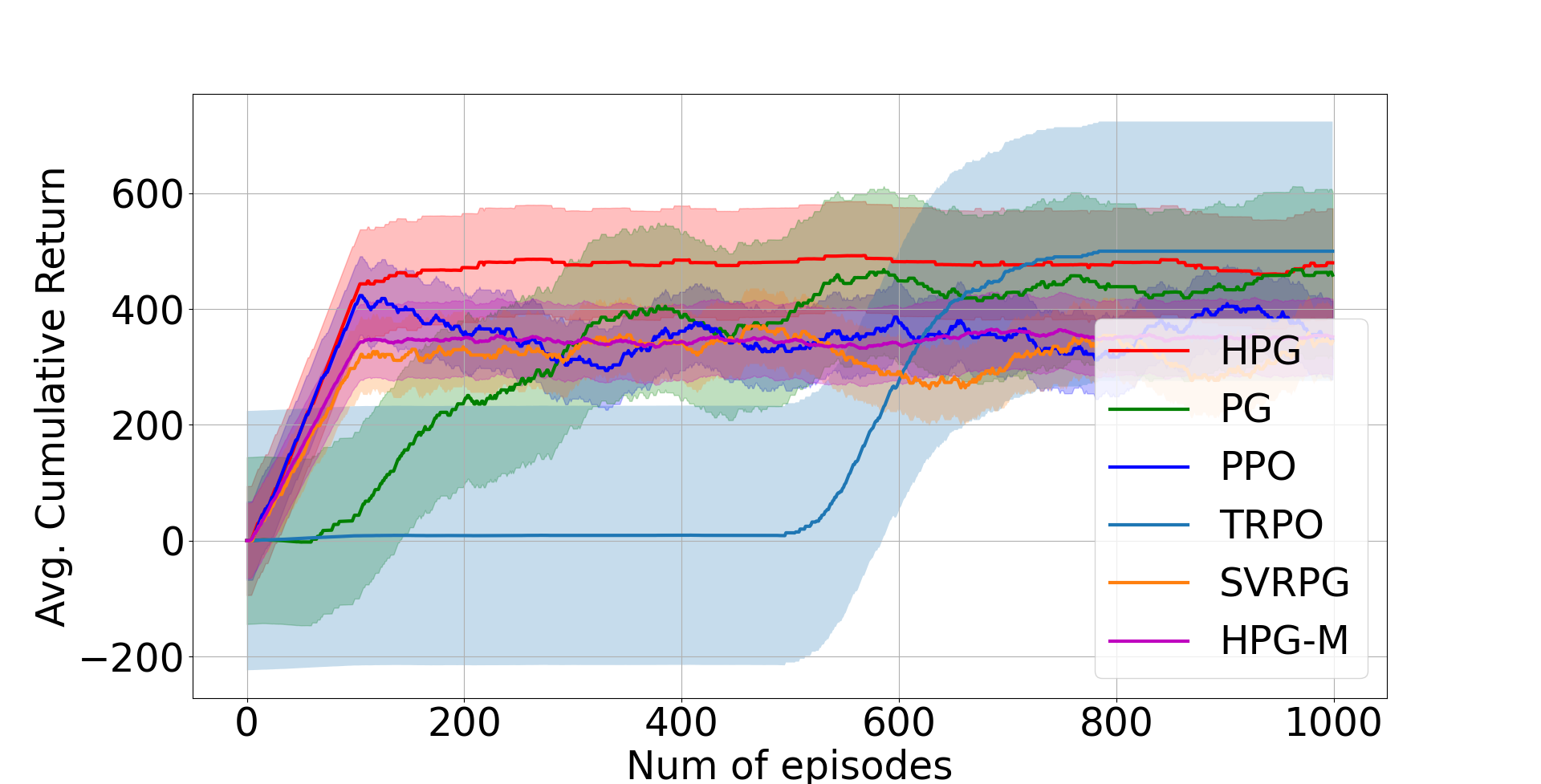} \label{fig:10}}\hspace{-2mm}
	\subfigure[Tuned scale $\sigma$ Mario.]
	{\includegraphics[scale=0.14]{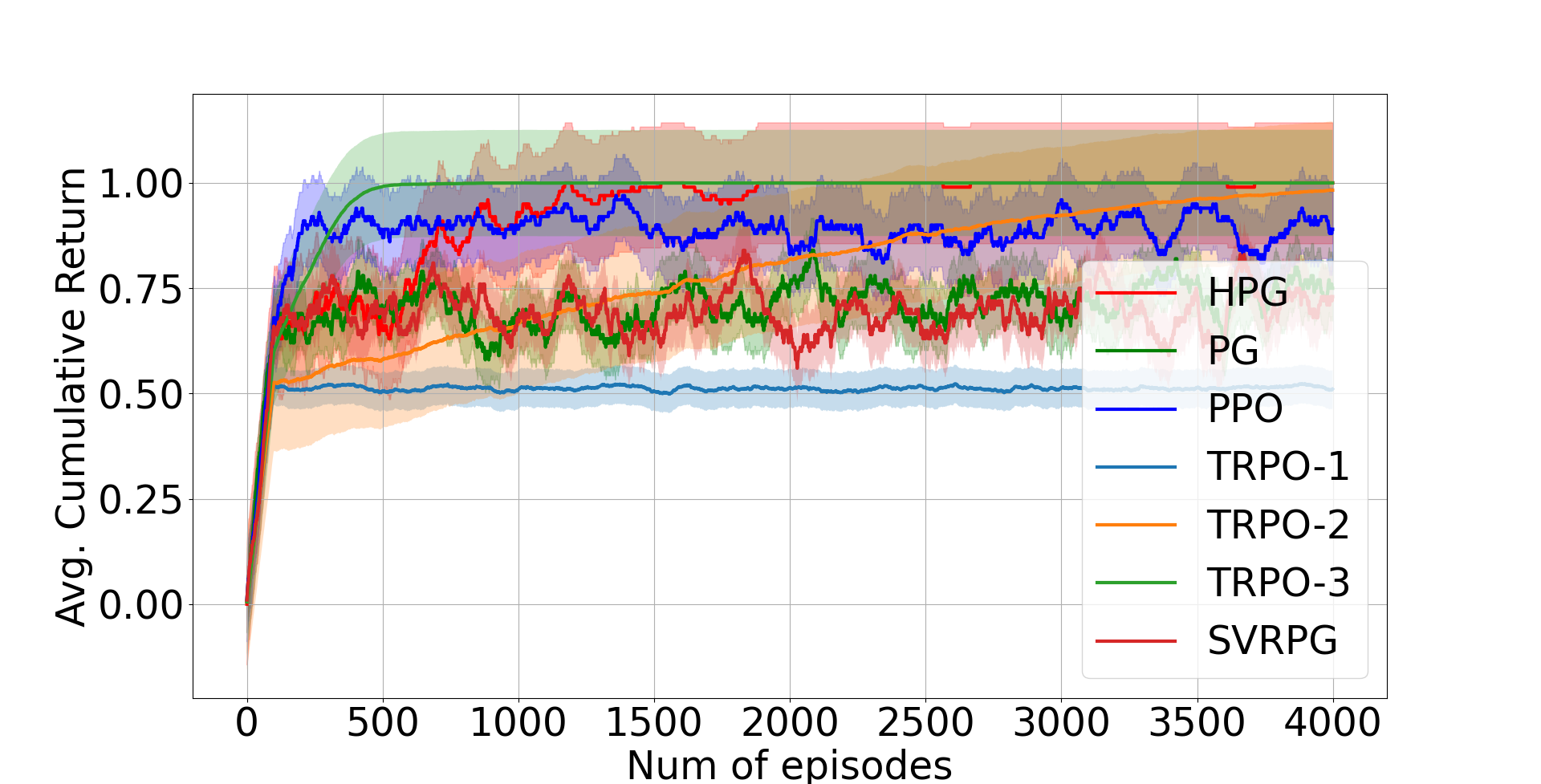} \label{fig:20}}
	\caption{ {\bf (a)} We plot the average cumulative returns for PMC environment over latest 100 episodes for Gaussian and Cauchy policies with constant $\sigma$. The importance of searching over a heavy-tailed (Cauchy) distribution is clear, as the Gaussian policy converges to spurious behavior.
		\textbf{(b)}  We plot average cumulative returns for PMC environment with variable $\sigma$ over latest 100 episodes. 
		{\textbf{(c)} Average commutative return for 1D Mario with variable sigma}. For TRPO, TRPO-1, 2, and 3, are respectively for trust region parameters $10^{-10}$, $10^{-6}$, and $10^{-5}$.
		\vspace{-0mm}}
	\label{comparisons2}
\end{figure*} 

Note that $\{ \bbX_i \}$ corresponds to the samples from policy gradient estimates.  The aforementioned  approach  has been  employed recently  for estimating the tail-index of stochastic gradients  \cite{garg2021proximal, simsekli2019tail}. Next, we present the main experiments results corroborating the findings in the main paper. Additional experiments with continuous control environments are provided in Appendix \ref{subsec:experiments_supplement}. 

 Fig \ref{comparisons2} compares the average commutative reward performance of HPG (for a Cauchy policy) to GPOMDP \cite{baxter2001infinite} with a Gaussian policy with fixed and tuned variance parameters \cite{papini2020balancing} (which we abbreviate as PG), as well as Proximal Policy Optimization (PPO) \cite{schulman2017proximal}, Trust Region Policy Optimization \cite{schulman2015trust}, 
 and Stochastic Variance Reduced PG (SVRPG) \cite{papini2018stochastic}, 
 for constant variance as well as variable variance.  In order to evaluate the Meta stable characteristics of the algorithms, we initialize each episodes at $s=2.26$, in the neighborhood of the local minima, $s = 2.67$.  Firstly in Fig. \ref{fig:00}, we evaluate the performance on PMC environment when the scale of the HPG is a constant $\sigma =3.0$ and the variance of Gaussian policy is also fixed $\sigma =3.0$. Secondly, we present the results with variable scale $\sigma$ for PMC and 1D Mario environments in Fig. \ref{fig:10}-\ref{fig:20}.  All the experiments use a  discounted factor of  $\gamma=0.97$  and we use   a diminishing step-size ranging from 0.005 to $5 \times 10^{-9}$. All the simulations are performed for $1000$ episodes using a  batch size of $B_k=5$ and with cumulative returns averaged over $100$ episodes.  For the comparison with PPO,  the policy ratio for  PPO is allowed to vary in the interval $[1-\epsilon_1, 1+\epsilon_1]$ with $\epsilon_1=0.2$.  Note that a fined tuned value of $\epsilon_1$ can result in a better performance as shown in Fig. \ref{fig:10}. However, note that the best performance feasible for PPO is same as that of PG.   The trust region parameters for TRPO, aka. maximum KL- divergence allowed is set to $0.001$.  In  addition, here we also evaluate performance of the HPG  against Stochastic Variance-Reduced Policy Gradient (SVRPG) (\cite{papini2018stochastic}). The number of epochs for SVRPG is fixed to $1000$ and epoch size, $m=10$. Further for PMC environment, we have included the comparisons with HPG-M which denotes the moderate tailed policy gradient with $\alpha=1.3$. In Fig. \ref{fig:00}, HPG-M1, HPG-M2 \blue{denotes the different instance} of moderate tailed policy gradient with $\sigma=5,3$, respectively.   For the experiments in  Fig \ref{comparisons2}, we use a simple  network without hidden layers.

From meta stability results of Section \ref{subsec:exit_time},  it is the  nature of jumps initiated by heavy tailed policies which results in better meta stable characteristics and  results in faster escape from spurious local extrema. In order to establish this fact, we plot the single test episode state visitation frequency (aka single episode occupancy measures) of HPG and PG once the training is done. The test episode is initialized at  $s=2.26$ (neighborhood of local extrema) and states of the environment  $s \in [-4, 3.709]$ are discretized into $100$ states and the heatmap \blue{of} the state visitation frequency is shown in Fig. \ref{comparisons4}. 
\begin{figure}[t]
	\centering
	\subfigure[PMC: Occupancy measure, HPG]
	{\includegraphics[width=.45\columnwidth, height=4.5cm]{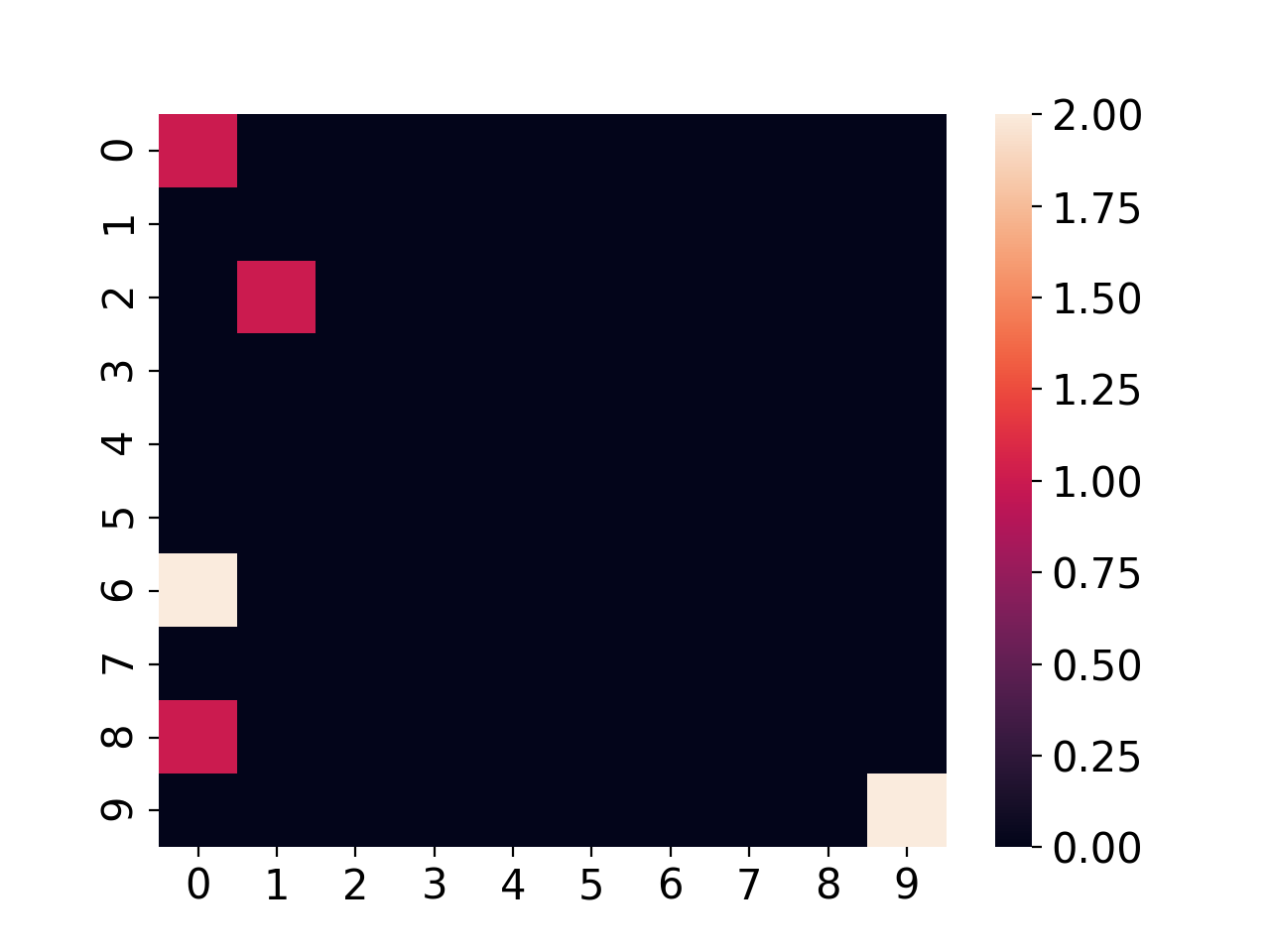} \label{fig:J2}}\hspace{-3mm}
	\subfigure[PMC: Occupancy measure, PG]
	{\includegraphics[width=.45\columnwidth, height=4.5cm]{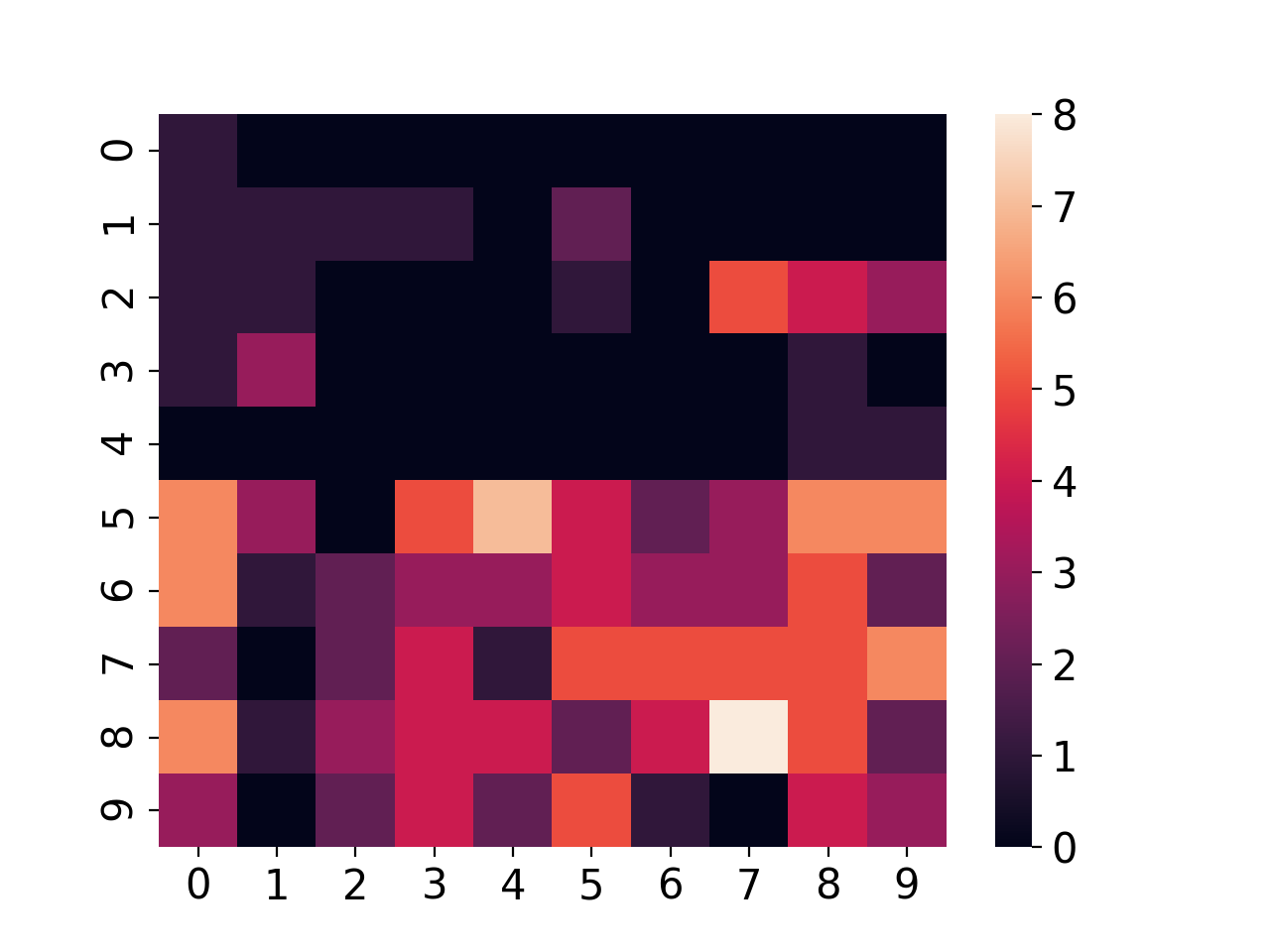} \label{fig:J3}}
	\caption{ Single episode occupancy measure for PMC when initialized in the neighborhood of spurious local extrema. We plot it for a test episode after network is trained for $1000$ episodes  with HPG policy and PG policy.  States from $[-4.0, 3.709]$ are discretized into $100$ states to calculate  the state visitation frequency.  Desired extrema of $s=-4.0$ corresponds to $(0,0)$ and the initial state $s=2.26$ is $(0,8)$.  Dark color in the map corresponds  to  region not visited during the test episode.  {\bf (a)} Single episode occupancy measure for PMC with HPG.   The importance of searching over a heavy-tailed (Cauchy) distribution is clear, as the  policy  takes heavier jumps and reach to desired goal faster.
		\textbf{(b)}  Single episode occupancy measure for  for PMC with PG.  Overall, we may observe that a Gaussian policy results in an occupancy measure which exhibits diffuse probability across the state space, failing to concentrate around actions associated with higher reward, whereas the heavy-tailed distribution results in an occupancy measure that assigns high likelihood to a small number of extreme actions in a manner reminiscent of the ``black swan" phenomenon \cite{taleb2007black}.}
	\label{comparisons4}
\end{figure} 
%

%% file: Conclusion.tex
\section{Conclusion} \label{sec:conclusion}

We focused on PG method in infinite-horizon RL problems. Inspired by persistent exploration that mitigates the tendency of policies to become mired at spurious behavior, we sought to nearly satisfy it in continuous settings through heavy-tailed policies. Doing so invalidated several aspects of existing analyses, which motivated studying the sample complexity of policy search when the score function is H\"{o}lder continuous and its norm is integrable with respect to the policy, and introducing an exploration tolerance parameter to quantify the degree to which the score function may be unbounded.

Moreover, we established that heavy-tailed policies induce heavy-tailed transition dynamics, which jump away from local extrema as formally quantified by the metastability characteristics of its L\'{e}vy process representation. We discerned that policies a heavier tail induce transitions away from a local extrema more quickly than one with a lighter tail, and tend towards extrema with more volume, which  we empirically associated with more stable policies for a few RL problems in practice. The characterization of jumps defined by metastability provides a lens through which approximate persistent exploration may be satisfied in continuous space.

%% file: Appendix.tex
\section*{\centering Appendix}

\section{Technical Details of Policy Search} \label{technical_details}
\subsection{Finiteness of Integral in \eqref{equ:policy_grad_main}} \label{integral_bounded_proof}

\begin{lemma}\label{lemma:integral_bounded}
The integral in \eqref{equ:policy_grad_main} is finite.
\end{lemma}
\begin{proof}
	Consider the integral for the policy gradient from \eqref{equ:policy_grad_main}, we obtain
	$$	I:=\frac{1}{1-\gamma}\int_{\cS\times\cA}\rho_{\pi_{\bbtheta}}(s)\cdot\pi_{{\bbtheta}}(a\given s)\cdot\|\nabla \label{proof_1} \log[\pi_{{\bbtheta}}(a\given s)]\|\cdot |Q_{\pi_{\bbtheta}}\cdot(s,a)|dsda\nonumber
	$$
	From Assumption \eqref{Assum:Q0} and the definition of $Q$ function, we note that $Q_{\pi_{\bbtheta}}(s,a)\leq \frac{U_R}{1-\gamma}$ for any ${\bbtheta},s$ and $a$. Therefore we could upper bound the above integral as follows
	\begin{align}
		I\leq \frac{U_R}{(1-\gamma)^2}&\int_{\cS\times\cA}\rho_{\pi_{\bbtheta}}(s)\cdot \|\nabla \log\pi_{{\bbtheta}}(a\given s)\|\cdot \pi_{{\bbtheta}}(a\given s)ds da\nonumber
		\\
		\leq \frac{U_R \sqrt{B}}{(1-\gamma)^2}&\int_{s\in\cS}\rho_{\pi_{\bbtheta}}(s) \cdot \blue{ds}\label{proof_3}
	\end{align}
	where we obtain \eqref{proof_3} from Assumption \ref{Assum:integral} after applying the 	{ Jensen's inequality which implies that
		\begin{equation*}
			\int_{\cA}  \|\nabla_{\bbtheta} \log\pi_{{\bbtheta}}(a\given s)\|^x\cdot\pi_{{\bbtheta}}(a\given s)\cdot da \leq B^{\frac{x}{2}}
		\end{equation*}
		for $x\in(1,2]$.} Since $\rho_{\pi_{\bbtheta}}(a)$ is the occupancy measure distribution for states $s$, we can obtain the  following upper bound for the integral as
	%
	$I\leq \frac{U_R \sqrt{B}}{(1-\gamma)^2}$.
\end{proof}

\subsection{Proof of Lemma \ref{lemma:J}}\label{apx_lemma:J_proof}
\begin{proof}
	Let us start by considering the stochastic gradient $\hat{\nabla} J({\bbtheta})$ as defined in \eqref{eq:policy_gradient_iteration}
	\#\label{equ:RPG_unbiased}
	&\EE[\hat{\nabla} J({\bbtheta})\given {\bbtheta}]=\EE\bigg\{\sum_{t=0}^{T}\gamma^{t/2}\cdot R(s_t,a_t)\cdot \bigg(\sum_{\tau=0}^{t}\nabla\log\pi_{{\bbtheta}_k}(a_{\tau}\given s_{\tau})\bigg)  \bigg\}.
	\#
	In order to relate the above expression to the true gradient, we introduce the infinite sum via the identity function notation $\mathbbm{1}_{T \geq t}$  and modify \eqref{equ:RPG_unbiased} as
	\#\label{equ:RPG_exchange_0}
	\EE[\hat{\nabla} J({\bbtheta})\given {\bbtheta}]= \EE\bigg\{\lim_{N\rightarrow \infty}\sum_{t=0}^{N}\mathbbm{1}_{T \geq t}\cdot \gamma^{t/2}\cdot R(s_t,a_t)\cdot \bigg(\sum_{\tau=0}^{t}\nabla\log\pi_{{\bbtheta}_k}(a_{\tau}\given s_{\tau})\bigg)\bigg\}.
	\#
	To interchange the limit and expectation in the above expression via Dominated Convergence Theorem in \eqref{equ:RPG_exchange_0}, we need first to ensure that the individual terms are dominated by an integrable function. To do so, we consider the term inside the expectation in \eqref{equ:RPG_exchange_0} as
	\begin{align}
		\bigg\|\sum_{t=0}^{N}&\mathbbm{1}_{T \geq t}\cdot \gamma^{t/2}\cdot R(s_t,a_t)\cdot \bigg(\sum_{\tau=0}^{t}\nabla\log\pi_{{\bbtheta}_k}(a_{\tau}\given s_{\tau})\bigg)\bigg\|
		\nonumber
		\\
		&\leq 		U_R\sum_{t=0}^N\mathbbm{1}_{T \geq t}\cdot \gamma^{t/2}\cdot \bigg(\sum_{\tau=0}^{t}\|\nabla\log\pi_{{\bbtheta}_k}(a_{\tau}\given s_{\tau})\|\bigg)\label{dom_1},
	\end{align}
	which follows from the bound $R(s_t,a_t)\leq {U_R}$. The sum on the right-hand side of \eqref{dom_1} is integrable with respect to the occupancy measure via Assumption \ref{Assum:integral}. Therefore, we may apply Dominated Convergence Theorem in \eqref{equ:RPG_exchange_0} in order to exchange expectation and limit as follows  
	\#\label{equ:RPG_exchange_00}
	\EE[\hat{\nabla} J({\bbtheta})\given {\bbtheta}]=& \lim_{N\rightarrow \infty}\sum_{t=0}^{N}\EE\bigg[\EE_T\left[\mathbbm{1}_{T \geq t}\right]\cdot \gamma^{t/2}\cdot R(s_t,a_t)\cdot \bigg(\sum_{\tau=0}^{t}\nabla\log\pi_{{\bbtheta}_k}(a_{\tau}\given s_{\tau})\bigg)\bigg]\nonumber
	\\
	= &\lim_{N\rightarrow \infty}\sum_{t=0}^{N}\EE\bigg[ \mathbb{P}[\blue{T} \geq t]\cdot \gamma^{t/2}\cdot R(s_t,a_t)\cdot \bigg(\sum_{\tau=0}^{t}\nabla\log\pi_{{\bbtheta}_k}(a_{\tau}\given s_{\tau})\bigg)\bigg].
	\#
	where \eqref{equ:RPG_exchange_00} holds since \blue{$\EE_{T}[\mathbbm{1}_{T \geq t}] = \mathbb{P}[T \geq t]$}. Next, we note that since $T\sim\tx{Geom}(1-\gamma^{1/2})$ which implies that $\mathbb{P}[T\geq t]=\gamma^{t/2}$, hence we can write
	\#\label{equ:RPG_exchange_000}
	\EE[\hat{\nabla} J({\bbtheta})\given {\bbtheta}]
	= &\EE\bigg[\sum_{t=0}^{\infty}\gamma^{t}\cdot R(s_t,a_t)\cdot \bigg(\sum_{\tau=0}^{t}\nabla\log\pi_{{\bbtheta}_k}(a_{\tau}\given s_{\tau})\bigg)\bigg].
	\#
	After rearranging the order of summation in the above expression, we could write
	\#\label{equ:RPG_exchange_0000}
	\EE[\hat{\nabla} J({\bbtheta})\given {\bbtheta}]
	= &\EE\bigg[\sum_{\tau=0}^{\infty} \sum_{t=\tau}^{\infty}\gamma^{t}\cdot R(s_t,a_t)\cdot \nabla\log\pi_{{\bbtheta}_k}(a_{\tau}\given s_{\tau})\bigg]
	\nonumber
	\\
	= &\EE\bigg[\sum_{\tau=0}^{\infty} \gamma^\tau\cdot\sum_{t=\tau}^{\infty}\gamma^{t-\tau}\cdot R(s_t,a_t)\cdot \nabla\log\pi_{{\bbtheta}_k}(a_{\tau}\given s_{\tau})\bigg]\nonumber
	\\
	= &\EE\bigg[\sum_{\tau=0}^{\infty} \gamma^\tau\cdot{Q}_{\pi_{{\bbtheta}}}(s_\tau,a_\tau)\cdot \nabla\log\pi_{{\bbtheta}_k}(a_{\tau}\given s_{\tau})\bigg]
	\nonumber
	\\
	= &{\nabla} J({\bbtheta}).
	\#
	%
	%
	%
	%
	which is as stated in Lemma \ref{lemma:J}.
\end{proof}
\section{Proof of Lemma \ref{Lemma:PGL}}\label{lemma_1_proof}
Before providing the proof for the statement of Lemma \ref{Lemma:PGL}, we discuss an  intermediate Lemma \ref{lemma_occu_meas} provided next. 
\begin{lemma}\label{lemma_occu_meas}
	For any given $(\bbtheta,\bbtheta')\in\mathbb{R}^d$, it holds that 
	
	1) the occupancy measure is Lipschitz continuous, which implies that  
	\begin{align}\label{lips_occu}
		\norm{\rho_{\bbtheta}(s,a)-\rho_{\bbtheta'}(s,a)}_{1}\leq M_\rho \cdot \|\bbtheta-\bbtheta'\|,
	\end{align}
	
	2) and the $Q_{\pi_{{\bbtheta}}}$ is also  Lipschitz continuous which satisfies 
	\begin{align}\label{Q_function_lips}
		\norm{Q_{\pi_{{\bbtheta}}}(s,a)-Q_{\pi_{{\bbtheta'}}}(s,a)}_{1}\leq M_Q\cdot \|\bbtheta-\bbtheta'\|,
	\end{align}
	where $M_\rho= \frac{\sqrt{B}}{1-\gamma}$ and $M_Q=\frac{ {\gamma U_R M_{\rho}}}{1-\gamma}$.
\end{lemma}
\begin{proof}
	\textbf{Proof of statement (1)}. In order to bound the term $	\norm{\rho_{\bbtheta}(s,a)-\rho_{\bbtheta'}(s,a)}_{1}$, let us define a function $d(\bbtheta,\bbtheta')$ as follows
	\begin{align}\label{def_d}
		d(\bbtheta,\bbtheta')
		=&  \|\rho_{\bbtheta}(s,a)-\rho_{\bbtheta'}(s,a)\|_1
		\\
		=& \int_{ \mathcal{S} \times \mathcal{A}}|\rho_{\bbtheta}(s,a)-\rho_{\bbtheta'}(s,a)|\cdot ds\ da .
	\end{align} 
	Next, we evaluate the gradient of $d(\bbtheta,\bbtheta')$ with respect to $\bbtheta$ and tale norm as 
	\begin{align}\label{main_eqution}
		\|\nabla_{\bbtheta} d(\bbtheta,\bbtheta')\| \leq &\bigg\|\int_{ \mathcal{S} \times \mathcal{A}}\text{sign}[\rho_{\bbtheta}(s,a)-\rho_{\bbtheta'}(s,a)]\cdot\nabla_{\bbtheta} \rho_{\bbtheta}(s,a)\cdot ds\ da \bigg\|\nonumber
		\\
		\leq& \int_{ \mathcal{S} \times \mathcal{A}}\|\nabla_{\bbtheta} \rho_{\bbtheta}(s,a)\| ds\ da,
	\end{align}
	%
%
	where $\text{sign}[x] = + 1$ if $x\geq0$ and $\text{sign}[x] = - 1$ if $x<0$. We recall the definition of the occupancy measure and write 
	\begin{align}\label{occupancy_measure}
		\rho_{\bbtheta}(s,a) = (1-\gamma)\sum_{t=0}^{\infty} \gamma^t p(s_t=s,a_t=a~|~\pi_{\bbtheta}, \xi(s_0))
	\end{align}
	where $p(s_t=s,a_t=a~|~\pi_{\bbtheta}, \xi(s_0))$ is the probability of visiting state $s$ and $a$ at the $t^{th}$ instant and $\xi(s_0)$ denotes the initial state distribution. We write the explicit form of $p(s_t=s,a_t=a~|~\pi_{\bbtheta}, \xi(s_0))$ as
	\begin{align}\label{probability_measure}
		p(s_t\!=\!s,a_t\!=\!a~|~\pi_{\bbtheta}, \xi(s_0))=&\int_{ \mathcal{S} \times \mathcal{A}}\hspace{-0mm}\xi(s_0)\cdot \pi_{\bbtheta}(a_0|s_0)p(s_1|s_0,a_0)\times
		\nonumber
		\\
		&\hspace{1.6cm}\times \pi_{\bbtheta}(a_1|s_1)p(s_2|s_1,a_1)\times		\nonumber
		\\
		&\hspace{3cm}\vdots\nonumber
		\\		\nonumber
		&\hspace{1.6cm}\times \pi_{\bbtheta}(a_{t-1}|s_{t-1})p(s_t|s_{t-1},a_{t-1})\times		\nonumber
		\\
		&\hspace{1.6cm}\times\pi_{\bbtheta}(a_{t}\!=\!a|s_{t}\!=\!s)\cdot d\bbs_{t-1}d\bba_{t-1}
	\end{align}
	where $d\bbs_{t-1}= ds_{0}ds_{1}\cdots ds_{t-1}$, and $d\bba_{t-1}= da_{0}da_{1}\cdots da_{t-1}$ denotes the integration of the state action pairs so far till $t$. Let us collect the state action pair trajectory till $t$ as~$\mathcal{T}_t:=\{(s_0,a_0), (s_1,a_1), \cdots, (\blue{s_{t},a_{t}})\}$ and write
	\begin{align}\label{def_prob}
		p_{\bbtheta}(\mathcal{T}_t)=\xi(s_0)\cdot \pi_{\bbtheta}(a_0|s_0)p(s_1|s_0,a_0)\times\cdots\times p(s_t|s_{t-1},a_{t-1})\cdot\pi_{\bbtheta}(a_{t}|s_{t}).
	\end{align}
	Using the notation in \eqref{def_prob}, we could rewrite \eqref{probability_measure} as 
\begin{align}\label{probability_measure2}
	p(s_t\!=\!s,a_t\!=\!a~|~\pi_{\bbtheta}, \xi(s_0))=&\int_{ \mathcal{T}_{t-1}}	p_{\bbtheta}(\mathcal{T}_{t-1})\cdot p(s_t|s_{t-1},a_{t-1})\cdot\pi_{\bbtheta}(a_{t}\!=\!a|s_{t}\!=\!s)\cdot d\mathcal{T}_{t-1}.
\end{align}

	Calculating the gradient on both sides of \eqref{occupancy_measure}, we get
	\begin{align}\label{occupancy_measure221}
		\nabla_{\bbtheta}\rho_{\bbtheta}(s,a) = (1-\gamma)\sum_{t=0}^{\infty} \gamma^t\nabla_{\bbtheta}p(s_t=s,a_t=a~|~\pi_{\bbtheta}, \xi(s_0)).
	\end{align}
	Using the definition \eqref{probability_measure2}, we could write the gradient in \eqref{occupancy_measure221} as follows
	\begin{align}\label{occupancy_measure2}
		\nabla_{\bbtheta}\rho_{\bbtheta}(s,a) = (1-\gamma)\sum_{t=0}^{\infty} \gamma^t\bigg[\int_{\mathcal{T}_{t-1}}\nabla_{\bbtheta}\big\{p_{\bbtheta}(\mathcal{T}_{t-1})\cdot p(s|s_{t-1},a_{t-1})\cdot\pi_{\bbtheta}(a|s)\big\}\cdot d\mathcal{T}_{t-1}\bigg],
	\end{align}
	\blue{ where by default, for the term $t=0$, we let $p_{\bbtheta}(\mathcal{T}_{-1})\equiv1,$ and $p(s|s_{t-1},a_{t-1}) = \xi(s)$.}
	Now we shift to upper bound the right hand side of \eqref{main_eqution} using the simplified definition in \eqref{occupancy_measure2}. 
		\blue{	Let us rewrite the inequality in \eqref{main_eqution} as
  \begin{align}
    	&\|\nabla_{\bbtheta} d(\bbtheta,\bbtheta')\|
    	\leq \int_{ \mathcal{S} \times \mathcal{A}}(1-\gamma)\sum_{t=0}^{\infty} \gamma^t\bigg\|\int_{\mathcal{T}_{t-1}}\nabla_{\bbtheta}\big\{p_{\bbtheta}(\mathcal{T}_{t-1})\cdot p(s|s_{t-1},a_{t-1})\cdot\pi_{\bbtheta}(a|s)\big\}\cdot d\mathcal{T}_{t-1}\bigg\|\cdot ds\ da, \tag{B.12} \nonumber
    \end{align}
where note that the integration over $\mathcal{S}\times \mathcal{A}$ \blue{\textit{is now outside the norm as correctly pointed out by the reviewer}}. 
	For this correct version of (B.12) the argument from the reviewer will no longer hold and will not result in any confusion. Note that $\nabla_{\bbtheta}\log\big\{p_{\bbtheta}(\mathcal{T}_{t-1})\cdot p(s|s_{t-1},a_{t-1})\cdot\pi_{\bbtheta}(a|s)\big\} = \sum_{i=0}^t\nabla_{\bbtheta}\log\pi_\theta(a_i|s_i)$, with $(s_t,a_t) = (s,a)$, then we have
	\begin{align}
		 \bigg\|\int_{\mathcal{T}_{t-1}}&\nabla_{\bbtheta}\big\{p_{\bbtheta}(\mathcal{T}_{t-1})\cdot p(s|s_{t-1},a_{t-1})\cdot\pi_{\bbtheta}(a|s)\big\} d\mathcal{T}_{t-1}\bigg\|\\
		 = & \bigg\|\int_{\mathcal{T}_{t-1}} p_{\bbtheta}(\mathcal{T}_{t-1})\cdot p(s|s_{t-1},a_{t-1})\cdot\pi_{\bbtheta}(a|s)\cdot \sum_{i=0}^t\nabla_{\bbtheta}\log\pi_\theta(a_i|s_i)  d\mathcal{T}_{t-1}\bigg\|\nonumber\\
		 \leq & \int_{\mathcal{T}_{t-1}} p_{\bbtheta}(\mathcal{T}_{t-1})\cdot p(s|s_{t-1},a_{t-1})\cdot\pi_{\bbtheta}(a|s)\cdot \
		\Big\|\sum_{i=0}^t\nabla_{\bbtheta}\log\pi_\theta(a_i|s_i)\Big\|  d\mathcal{T}_{t-1} \nonumber\\
	     \leq & \int_{\mathcal{T}_{t-1}} p_{\bbtheta}(\mathcal{T}_{t-1})\cdot p(s|s_{t-1},a_{t-1})\cdot\pi_{\bbtheta}(a|s)\cdot (t+1)\sqrt{B}  d\mathcal{T}_{t-1} \nonumber
	\end{align}
    Substitute the above inequality to B.12 gives us 
        \begin{align}
    	&\|\nabla_{\bbtheta} d(\bbtheta,\bbtheta')\|
    	\leq \int_{ \mathcal{S} \times \mathcal{A}}(1-\gamma)\sum_{t=0}^{\infty} \gamma^t\cdot \int_{\mathcal{T}_{t-1}} p_{\bbtheta}(\mathcal{T}_{t-1})\cdot p(s|s_{t-1},a_{t-1})\cdot\pi_{\bbtheta}(a|s)\cdot (t+1)\sqrt{B}\cdot  d\mathcal{T}_{t-1}\cdot ds\ da,  \nonumber
    \end{align}
		We note that we can take $(1-\gamma)\sum_{t=0}^{\infty} \gamma^t$ outside the integration $\int_{ \mathcal{S} \times \mathcal{A}}$, hence we get
		\begin{align}
			&\|\nabla_{\bbtheta} d(\bbtheta,\bbtheta')\|
			\nonumber
			\\
			&
	\leq (1-\gamma)\sum_{t=0}^{\infty} \gamma^t\cdot \int_{ \mathcal{S} \times \mathcal{A}} \int_{\mathcal{T}_{t-1}} p_{\bbtheta}(\mathcal{T}_{t-1})\cdot p(s|s_{t-1},a_{t-1})\cdot\pi_{\bbtheta}(a|s)\cdot (t+1)\sqrt{B}\cdot  d\mathcal{T}_{t-1}\cdot ds\ da  \nonumber. \nonumber
		\end{align}
		%
		%
		After adjusting the limits of integration and from the definition of $	p_{\bbtheta}(\mathcal{T}_t)=\xi(s_0)\cdot \pi_{\bbtheta}(a_0|s_0)p(s_1|s_0,a_0)\times\cdots\times p(s_t|s_{t-1},a_{t-1})\cdot\pi_{\bbtheta}(a_{t}|s_{t})$ from (B.8) in the main paper,  we can write
		\begin{align}
			&\|\nabla_{\bbtheta} d(\bbtheta,\bbtheta')\|
			\leq (1-\gamma)\sum_{t=0}^{\infty} \gamma^t \int_{\mathcal{T}_{t}}  \big((t+1)\sqrt{B}\big)\cdot p_{\bbtheta}(\mathcal{T}_t)\cdot d\mathcal{T}_{t}. \nonumber
		\end{align}
		From Assumption 2, it holds that
		%
		$\int_{\mathcal{A}}  \|\nabla_{\bbtheta} \log\pi_{{\bbtheta}}(a| s)\|\cdot\pi_{{\bbtheta}}(a| s) da \leq \sqrt{B}$. 
		%
	We can write  the above inequality as 
		\begin{align}\label{main_eqution2222}
			\|\nabla_{\bbtheta} d(\bbtheta,\bbtheta')\|
			\leq  &(1-\gamma)\sum_{t=0}^{\infty} \gamma^t (t+1)\sqrt{B}=\frac{\sqrt{B}}{1-\gamma}:=M_\rho,  \tag{B.15}
		\end{align}
		where we have used the fact that arithmetic–geometric sequence $\sum_{t=0}^{\infty}(t+1) \gamma^t = \frac{\gamma}{(1-\gamma)^2}$. Since the gradient of function $d(\bbtheta,\bbtheta')$ with respect to $\bbtheta$ is bounded, it implies that  $d(\bbtheta,\bbtheta')$ is Lipschitz with respect to $\bbtheta$, which further implies that
		\begin{align}\label{lipschitz}
			|d(\bbtheta,\bbtheta') - d(\bbtheta',\bbtheta')|\leq M_\rho\|\bbtheta-\bbtheta'\| ,  \tag{B.16}
		\end{align}
		for all $\bbtheta,\bbtheta'\in\mathbb{R}^d$. 	Next,  substituting the definition $d(\bbtheta,\bbtheta')=\|\rho_{\bbtheta}(s,a)-\rho_{\bbtheta'}(s,a)\|_1$ into \eqref{lipschitz} and noting that $d(\bbtheta',\bbtheta')=0$ , we get
		\begin{align}\label{lipschitz2}
			\|\rho_{\bbtheta}(s,a)-\rho_{\bbtheta'}(s,a)\|_1
			\leq M_\rho\|\bbtheta-\bbtheta'\|,
		\end{align}
		which is as stated in Lemma \ref{lemma_occu_meas}(1).}

	\textbf{Proof of statement (2)}. Let us define the occupancy measure, 
	\begin{align}\label{occupancy_measure222}
		\mu_{\bbtheta}^{sa}(s',a') = (1-\gamma)\sum_{t=0}^{\infty} \gamma^t\cdot p(s_t=s',a_t=a'~|~\pi_{\bbtheta}, s_0 \sim p(\cdot|s,a))
	\end{align}
	where the initial state $s_0 \sim p(\cdot|s,a)$. With the help of occupancy measure $\mu_{\bbtheta}^{sa}(s',a')$, we could write the $Q_{\pi_{{\bbtheta}}}$ function for state action pair $(s,a)$ as 
	\begin{align}\label{q_function}
		Q_{\pi_{{\bbtheta}}}(s,a)=R(s,a)+\frac{\gamma}{1-\gamma} \int R(s',a')\cdot	\mu_{\bbtheta}^{sa}(s',a')\cdot ds'da'.
	\end{align}
	Using the definition in \eqref{q_function}, we can write
	\begin{align}
		|Q_{\pi_{{\bbtheta}}}(s,a)-Q_{\pi_{{\bbtheta'}}}(s,a)|\leq \frac{\gamma U_R}{1-\gamma}	\|\mu_{\bbtheta}^{sa}(s',a')-\mu_{\bbtheta'}^{sa}(s',a')\|_1.
	\end{align}	
	Utilizing the upper bound in \eqref{lipschitz2}, we can write
	Using the definition in \eqref{q_function}, we can write
	\begin{align}
		|Q_{\pi_{{\bbtheta}}}(s,a)-Q_{\pi_{{\bbtheta'}}}(s,a)|\leq \frac{ {\gamma U_R M_{\rho}}}{1-\gamma}\|\bbtheta-\bbtheta'\|.
	\end{align}	
as stated in Lemma \ref{lemma_occu_meas}(2).
	
\end{proof}

Now we shift focus to prove the statement of Lemma \ref{Lemma:PGL} as follows.
\begin{proof}[Proof of Lemma \ref{Lemma:PGL}]
	To obtain a bound on the gradient norm difference for $J$, we start by considering the  term $\| {\nabla} _{{\bbtheta}} J ({\bbtheta}_1) - {\nabla} _{{\bbtheta}} J ({\bbtheta}_2 )  \|$ and expanding it using the definition in \eqref{equ:policy_grad_main}, we get 
	\begin{align}\label{first}
		\| {\nabla} _{{\bbtheta}} J ({\bbtheta}_1) - {\nabla} _{{\bbtheta}} J ({\bbtheta}_2 )  \|
		& \leq  \frac{1}{1-\gamma}\left\| \int_{ \mathcal{S} \times \mathcal{A}} Q_{\pi_{{\bbtheta}_1}}(s,a) \nabla\log\pi_{{\bbtheta_1}}(a | s)    {\rho_{\pi_{{\bbtheta}_1}}(s,a) }  ds  \, da \right.  
		\\ 
		& \hspace{2cm}\left.  - \int_{ \mathcal{S} \times \mathcal{A}}    Q_{\pi_{{\bbtheta}_2}}(s,a)  \nabla\log\pi_{{\bbtheta_2}}(a | s)   \rho_{\pi_{{\bbtheta}_2}}(s,a)  ds  \, da \right\| .\nonumber
	\end{align}
	Add and subtract the terms $Q_{\pi_{{\bbtheta}_1}}(s,a) \nabla\log\pi_{{\bbtheta_2}}(a | s)  {\rho_{\pi_{{\bbtheta}_1}}(s)} \pi_{{\bbtheta}_1}(a|s) $ and $Q_{\pi_{{\bbtheta}_1}}(s,a) \nabla\log\pi_{{\bbtheta_2}}(a | s)  {\rho_{\pi_{{\bbtheta}_2}}(s)} \pi_{{\bbtheta}_1}(a|s) $  inside the first  integral of \eqref{first} and use triangle inequality to obtain
	%
	\begin{align} \label{Eqn:int5}
		&\| {\nabla} _{{\bbtheta}} J ({\bbtheta}_1) - {\nabla} _{{\bbtheta}} J ({\bbtheta}_2 )  \|  
		\\
		&  \leq  \underbrace{ \frac{1}{1-\gamma}\bigg\|\int_{ \mathcal{S} \times \mathcal{A}}  Q_{\pi_{{\bbtheta}_1}}(s,a) \bigg(\nabla\log\pi_{{\bbtheta_1}}(a | s) -\nabla\log\pi_{{\bbtheta_2}}(a | s)\bigg)  {\rho_{\pi_{{\bbtheta}_1}}(s,a) } ds  \, da \bigg\|}_{\bbI_1} \nonumber 
		\\ 
		& \hspace{2cm}  +\underbrace{\frac{1}{1-\gamma}\bigg\|\int_{ \mathcal{S} \times \mathcal{A}}  \hspace{0mm} \bigg(Q_{\pi_{{\bbtheta}_1}}(s,a)-Q_{\pi_{{\bbtheta}_2}}(s,a)\bigg)\nabla\log\pi_{{\bbtheta_2}}(a | s)   
			\rho_{\pi_{{\bbtheta}_2}}(s,a)  ds  \, da\bigg\|}_{\bbI_2}  \nonumber 
		\\ 
		& \hspace{2.5cm}  +\underbrace{\frac{1}{1-\gamma}\bigg\|\int_{ \mathcal{S} \times \mathcal{A}}  \hspace{0mm} Q_{\pi_{{\bbtheta}_1}}(s,a) \nabla\log\pi_{{\bbtheta_2}}(a | s)  \bigg(  {\rho_{\pi_{{\bbtheta}_1}}(s,a) }  -\rho_{\pi_{{\bbtheta}_2}}(s,a)    \bigg)  ds  \, da\bigg\|}_{\bbI_3} .\nonumber
	\end{align}
	Hence, we could write the equation in \eqref{Eqn:int5} as follows using the triangle inequality:
	\begin{align}
		\| {\nabla} _{{\bbtheta}} J ({\bbtheta}_1) -{\nabla} _{{\bbtheta}} J ({\bbtheta}_2 )  \|  \leq &   \bbI_1+\bbI_2+\bbI_3.\label{second}
	\end{align}
	Now we will bound each of the above terms separately. Let us start with $\bbI_1$ and take the norm inside the integral to get 
	\begin{align}\label{fouth}
		\bbI_1 \leq \frac{1}{1-\gamma}\int_{ \mathcal{S} \times \mathcal{A}}  |Q_{\pi_{{\bbtheta}_1}}(s,a)|  \left\| \nabla\log\pi_{{\bbtheta_1}}(a | s) -\nabla\log\pi_{{\bbtheta_2}}(a | s) \right\|     {\rho_{\pi_{{\bbtheta}_1}}(s,a) }   ds  \, da,
	\end{align}
	From Assumption \ref{Assum:Q0}, we have $|Q_{\pi_{{\bbtheta}_1}}(s,a)|  \leq \frac{U_R}{(1-\gamma)}$ which may be substituted into the right-hand side of \eqref{fouth} as follows
	\begin{align}\label{fouth2}
		\bbI_1 \leq \frac{U_R}{(1-\gamma)^2}\int_{ \mathcal{S} \times \mathcal{A}}  {\left\| \nabla\log\pi_{{\bbtheta_1}}(a | s)-\nabla\log\pi_{{\bbtheta_2}}(a | s)  \right\| }   {\rho_{\pi_{{\bbtheta}_1}}(s,a) } ds  \, da.
	\end{align} 
	From Assumption \eqref{Assum:nablaF} regarding the H\"{o}lder continuity of the score function, we have
	\begin{align}\label{inequality}
		\|\nabla\log\pi_{{\bbtheta_1}}(a | s)-\nabla\log\pi_{{\bbtheta_2}}(a | s) \|\leq M \|{\bbtheta}_1-{\bbtheta}_2\|^{\beta},
	\end{align} 
	where $\beta\in(0,1]$. 
	This expression \eqref{inequality} may be substituted into \eqref{fouth2} to write
	\begin{align}
		\bbI_1
		\leq &    \frac{U_R M}{(1-\gamma)^2} \left( \|{\bbtheta}_1-{\bbtheta}_2\|^{\beta}  \int_{ \mathcal{S} \times \mathcal{A}}    {\rho_{\pi_{{\bbtheta}_1}}(s,a) } \cdot    \,  ds  \, da \right).
	\end{align}
	The above  integral is a valid probability measure which further implies that it integrates to unit. Therefore, we have that
	\begin{align}\label{first_bound}
		\bbI_1\leq   \frac{U_R M  \|{\bbtheta}_1-{\bbtheta}_2\|^\beta }{(1-\gamma)^2}.
	\end{align}
	Now let us consider the expression associated with $\bbI_2$  in \eqref{Eqn:int5} and take the norm inside the integral , we write
	\begin{align}\label{fifth2}
		\bbI_2 \leq\frac{1}{1-\gamma}\int_{ \mathcal{S} \times \mathcal{A}}  \hspace{0mm} \big|Q_{\pi_{{\bbtheta}_1}}(s,a)-Q_{\pi_{{\bbtheta}_2}}(s,a)\big|\cdot \|\nabla\log\pi_{{\bbtheta_2}}(a | s)\|   
		\rho_{\pi_{{\bbtheta}_2}}(s,a)  ds  \, da.
	\end{align}
	Note that Q function is Lipschitz as given in \eqref{Q_function_lips}, hence we can upper bound \eqref{fifth2} as follows
	\begin{align}\label{fifth22}
		\bbI_2 \leq& \frac{M_Q}{1-\gamma}\|\bbtheta_1-\bbtheta_2\|\int_{ \mathcal{S} \times \mathcal{A}}  \|\nabla\log\pi_{{\bbtheta_2}}(a | s)\|   
		\rho_{\pi_{{\bbtheta}_2}}(s,a)  ds  \, da
		\\
		\leq& \frac{M_QB^{1/2}}{1-\gamma}\|\bbtheta_1-\bbtheta_2\|.
	\end{align}
	Now, we are only left to bound $\bbI_3$ in \eqref{second}. Let us rewrite $\bbI_3$  as follows
	\begin{align}
		\bbI_3=&\frac{1}{1-\gamma}\bigg\|\int_{ \mathcal{S} \times \mathcal{A}}  \hspace{0mm} Q_{\pi_{{\bbtheta}_1}}(s,a) \nabla\log\pi_{{\bbtheta_2}}(a | s)  \bigg(  {\rho_{\pi_{{\bbtheta}_1}}(s,a) }  -\rho_{\pi_{{\bbtheta}_2}}(s,a)    \bigg)  ds  \, da\bigg\|
		\nonumber
		\\
		\leq & \frac{1}{1-\gamma}\int_{ \mathcal{S} \times \mathcal{A}}  \hspace{0mm} |Q_{\pi_{{\bbtheta}_1}}(s,a)|\cdot \|\nabla\log\pi_{{\bbtheta_2}}(a | s)\| \cdot   |{\rho_{\pi_{{\bbtheta}_1}}(s,a) }  -\rho_{\pi_{{\bbtheta}_2}}(s,a)|      ds  \, da.
	\end{align}
	Using the bound $|Q_{\pi_{{\bbtheta}_1}}(s,a)|  \leq \frac{U_R}{(1-\gamma)}$, we can write
	\begin{align}
		\bbI_3
		\leq & \frac{U_R}{(1-\gamma)^2}\int_{ \mathcal{S} \times \mathcal{A}}   \|\nabla\log\pi_{{\bbtheta_2}}(a | s)\| \cdot   |{\rho_{\pi_{{\bbtheta}_1}}(s,a) }  -\rho_{\pi_{{\bbtheta}_2}}(s,a)|      ds  \, da. \label{bound_score}
	\end{align}
	{Next, we need to bound the right-hand side of \eqref{bound_score}. 
%
%
%
According to definitions provided in \eqref{define_Set}-\eqref{definition_last}, for any $s$, there exists a set $\mathcal{C}\in\mathcal{A}(\lambda)$ s.t. 
\begin{align}\label{definition_2}
			\int_{\mathcal{S}}\int_{\mathcal{A}\backslash\mathcal{C}}\|\nabla\log\pi_{{\bbtheta}}(a | s)\|\cdot\rho_{{\bbtheta}}(s,a) \cdot  ds da \leq \lambda,
\end{align}
and 
$\sup_{a\in\mathcal{C}} \|\nabla\log\pi_{{\bbtheta}}(a | s)\|\leq B(\lambda)$.
%
We proceed to upper bound the right-hand side of \eqref{bound_score} by splitting this integral over the action space into two parts and employing the quantities in Definition \ref{def:exploration_tolerance}, specifically, \eqref{define_Set}, as
	%
	\begin{align}
		\bbI_3
		\leq & \frac{U_R}{(1-\gamma)^2}\int_{ \mathcal{S}}  \int_{\mathcal{C}}    \|\nabla\log\pi_{{\bbtheta_2}}(a | s)\| \cdot|{\rho_{\pi_{{\bbtheta}_1}}(s,a) }  -\rho_{\pi_{{\bbtheta}_2}}(s,a)|      ds  \, da.
		\nonumber
		\\
		&+\frac{U_R}{(1-\gamma)^2}\int_{ \mathcal{S}}\int_{ \mathcal{A}\backslash\mathcal{C}}   \|\nabla\log\pi_{{\bbtheta_2}}(a | s)\| \cdot   |{\rho_{\pi_{{\bbtheta}_1}}(s,a) }  -\rho_{\pi_{{\bbtheta}_2}}(s,a)|      ds  \, da.
	\end{align}
	From \eqref{definition_last} and the definition of $\|\cdot\|_1$, we can write
	\begin{align}
		\bbI_3
		\leq & \frac{U_R B(\lambda)}{(1-\gamma)^2}{\|{\rho_{\pi_{{\bbtheta}_1}}(s,a) }-{\rho_{\pi_{{\bbtheta}_2}}(s,a) }\|_1}
		\nonumber
		\\
		&+\frac{U_R}{(1-\gamma)^2}\int_{ \mathcal{S}}\int_{ \mathcal{A}\backslash\mathcal{C}}  \|\nabla\log\pi_{{\bbtheta_2}}(a | s)\| \cdot   |{\rho_{\pi_{{\bbtheta}_1}}(s,a) }  -\rho_{\pi_{{\bbtheta}_2}}(s,a)|      ds  \, da. \label{bound_score2}
	\end{align}
	From the Lipschitz continuity of the occupancy measure in \eqref{lips_occu} and triangle inequality, we may write
	\begin{align}
		\bbI_3
		\leq & \frac{U_R B(\lambda) M_\rho}{(1-\gamma)^2}\|\bbtheta_1-\bbtheta_2\| 
		\nonumber
		\\
		&+\frac{U_R}{(1-\gamma)^2}\underbrace{\int_{ \mathcal{S}}\int_{\mathcal{A}\backslash\mathcal{C}}  \|\nabla\log\pi_{{\bbtheta_2}}(a | s)\| \cdot   \left({\rho_{\pi_{{\bbtheta}_1}}(s,a) }  +\rho_{\pi_{{\bbtheta}_2}}(s,a)\right)      ds  \, da}_{\mathcal{Z}}. \label{bound_score22}
	\end{align}
	Let us focus on the second term $\mathcal{Z}$ of the right-hand side of \eqref{bound_score22}. We expand its expression using \eqref{define_Set} as
	\begin{align}
		\mathcal{Z}=& \int_{ \mathcal{S}}\int_{\mathcal{A}\backslash\mathcal{C}} \|\nabla\log\pi_{{\bbtheta_2}}(a | s)\| \cdot   {\rho_{\pi_{{\bbtheta}_1}}(s,a) }      ds  \, da
		\nonumber
		\\
		&+\int_{ \mathcal{S}}\int_{\mathcal{A}\backslash\mathcal{C}} \|\nabla\log\pi_{{\bbtheta_2}}(a | s)\| \cdot   \rho_{\pi_{{\bbtheta}_2}}(s,a)    ds  \, da.\label{bound_score222}
	\end{align}
	Next, from the H\"{o}lder continuity of the score function, by adding and subtracting $\nabla\log\pi_{{\bbtheta_1}}(a | s)$ inside the norm for the first term, we may write
	\begin{align}
		\mathcal{Z}
		\leq & M\|\bbtheta_1-\bbtheta_2\|^\beta + \int_{ \mathcal{S}}\int_{ \mathcal{A}\backslash\mathcal{C}} \|\nabla\log\pi_{{\bbtheta_1}}(a | s)\| \cdot   {\rho_{\pi_{{\bbtheta}_1}}(s,a) }      ds  \, da
		\nonumber
		\\
		&\hspace{2.5cm}+\int_{\mathcal{S}}\int_{\mathcal{A}\backslash\mathcal{C}} \|\nabla\log\pi_{{\bbtheta_2}}(a | s)\| \cdot   \rho_{\pi_{{\bbtheta}_2}}(s,a)    ds  \, da\nonumber\\
		\leq & M\|\bbtheta_1-\bbtheta_2\|^\beta + 2\lambda.    \label{bound_score2222}
	\end{align}}
	Substituting \eqref{bound_score2222} into the right hand side of \eqref{bound_score22}, we get
	\begin{align}
		\bbI_3
		\leq & \frac{U_RM}{(1-\gamma)^2}\|\bbtheta_1-\bbtheta_2\|^\beta+\frac{U_R B(\lambda) M_\rho}{(1-\gamma)^2}\|\bbtheta_1-\bbtheta_2\|+ \frac{2U_R\lambda}{(1-\gamma)^2}. \label{bound_score2222022}
	\end{align}
	%
	%
	%
	%
	Next, substituting the upper bounds for $\bbI_1$, $\bbI_2$, and $\bbI_3$ into the right hand side of \eqref{second}, we get
	\begin{align}
		\| {\nabla} _{{\bbtheta}} J ({\bbtheta}_1)\! -\!{\nabla} _{{\bbtheta}} J ({\bbtheta}_2 )  \|  \leq &    M_J[\|{\bbtheta}_1-{\bbtheta}_2\|^\beta\!+\!\|\bbtheta_1-\bbtheta_2\|\!+\!\lambda]  \nonumber
	\end{align}
	where $M_J$ is defined as 
	\begin{align}\label{eq:M_J}
		M_J:=\max\bigg\{\frac{2U_R M }{(1-\gamma)^2}, \frac{M_QB^{1/2}}{1-\gamma} + \frac{U_R B(\lambda) M_\rho}{(1-\gamma)^2},\frac{2U_R}{(1-\gamma)^2}\bigg\}.
	\end{align}
	which concludes the proof of Lemma \ref{Lemma:PGL}.
\end{proof}

\section{Proof of Lemma \ref{Lemma:PGL2}}\label{apx_lemma_2}
\begin{proof}
	The proof technique is motivated from the analysis in \cite{nguyen2019non}. Consider a curve $g(t) \triangleq J(\bbtheta_2+ t(\bbtheta_1-\bbtheta_2))$. Then $g'(t) = \left< \nabla J(\bbtheta_2 + t(\bbtheta_1 - \bbtheta_2)), \bbtheta_1 - \bbtheta_2 \right>$. The integral of $g'(t)$ from $t=0$ to $t=1$ can be expressed as 
	$$
	\int_{t=0}^{t=1} g'(t) dt = g(1)-g(0) =  J(\bbtheta_1) -J(\bbtheta_2).
	$$
	by the \blue{Fundamental theorem of calculus}.
	Now subtracting $\left< \nabla J(\bbtheta_2), \bbtheta_1-\bbtheta_2 \right>$ on the both sides of  the above expression, we get
	\begin{align}
		\left| J(\bbtheta_1) -J(\bbtheta_2) -\left< \nabla J(\bbtheta_2), \bbtheta_1-\bbtheta_2 \right> \right|  = \left| \int_{t=0}^{t=1} g'(t) dt  -\left< \nabla J(\bbtheta_2), \bbtheta_1-\bbtheta_2 \right> \right|
	\end{align}
	Using the expression for $g'(t)$ above expression takes the form
	\begin{align}
		\big| J(\bbtheta_1) -&J(\bbtheta_2) -\left< \nabla J(\bbtheta_2), \bbtheta_1-\bbtheta_2 \right> \big|  \nonumber
		\\
		= & \left| \int_{t=0}^{t=1}\left< \nabla J(\bbtheta_2 + t(\bbtheta_1 - \bbtheta_2)), \bbtheta_1 - \bbtheta_2 \right> dt  -\left< \nabla J(\bbtheta_2), \bbtheta_1-\bbtheta_2 \right> \right| \nonumber \\
		=& \left| \int_{t=0}^{t=1}\left< \nabla J(\bbtheta_2 + t(\bbtheta_1 - \bbtheta_2))-\nabla J(\bbtheta_2), \bbtheta_1 - \bbtheta_2 \right> dt   \right|.
	\end{align}
	Using Cauchy-Schwartz inequality for inner product on the right-hand side of the previous expression
	\begin{align}
		\left| J(\bbtheta_1) -J(\bbtheta_2) -\left< \nabla J(\bbtheta_2), \bbtheta_1-\bbtheta_2 \right> \right|  \leq  \int_{t=0}^{t=1} \left\| \nabla J(\bbtheta_2 + t(\bbtheta_1 - \bbtheta_2))-\nabla J(\bbtheta_2)\right\|  \left\| \bbtheta_1 - \bbtheta_2  \right\| dt .
	\end{align}
	Apply Lemma \ref{Lemma:PGL} to the first factor of the integrand on the right-hand side, which simplifies the preceding expression to
	\begin{align}
		\big| J(\bbtheta_1) -& J(\bbtheta_2) -\left< \nabla J(\bbtheta_2), \bbtheta_1-\bbtheta_2 \right> \big| \nonumber
		\\
		& \leq  \left\| \bbtheta_1 - \bbtheta_2  \right\| \int_{t=0}^{t=1} M_J\left[\|  t( {\bbtheta}_1 -{\bbtheta}_2)  \|^{\beta}+\|   t( {\bbtheta}_1 -{\bbtheta}_2)  \| +\lambda\right] dt
		\nonumber
		\\
		& \leq   M_J\left[\|  {\bbtheta}_1 -{\bbtheta}_2 \|^{1+\beta}+\|  {\bbtheta}_1 -{\bbtheta}_2  \|^2 + \lambda \|  {\bbtheta}_1 -{\bbtheta}_2 \|\right] ,
	\end{align}
%
	which is as stated in \eqref{eq:gradient_smoothness}.
\end{proof}
\section{Proof of Theorem \ref{Thm1}}\label{proof_theorem}

\begin{proof}
	We begin by unraveling the statement of Lemma \eqref{Lemma:PGL2} to write an approximate ascent relationship on the objective $J(\bbtheta)$ as:
	\begin{align}\label{thm_first}
		J ({{\bbtheta}}_{k+1}) \geq &  J ({{\bbtheta}}_{k}) +   \left<{\nabla} J ({{\bbtheta}}_k),{\bbtheta}_{k+1} - {\bbtheta}_k  \right> 
		\nonumber
		\\
		&- M_J\left[ \| {\bbtheta}_{k+1}- {\bbtheta}_k \|^{1+\beta}+\|  {\bbtheta}_{k+1}- {\bbtheta}_k  \|^2+\lambda\|  {\bbtheta}_{k+1}- {\bbtheta}_k  \|\right]. 
	\end{align}
	Substitute the expression for the policy gradient \eqref{eq:policy_gradient_iteration} in place of $\bbtheta_{k+1} - \bbtheta_k$ into \eqref{thm_first} as
	\begin{align}
		\label{thm3}
		J ({{\bbtheta}}_{k+1}) \geq &  J ({{\bbtheta}}_{k}) + \eta\left<{\nabla} J ({{\bbtheta}}_k), \hat{\nabla} J({\bbtheta}_k)  \right> \nonumber
		\\
		&-   M_J\left[\|  \eta \hat{\nabla} J({\bbtheta}_k) \|^{1+\beta}+\|   \eta \hat{\nabla} J({\bbtheta}_k) \|^2+\lambda\|   \eta \hat{\nabla} J({\bbtheta}_k) \|\right].  
	\end{align}
For $c = 1, 1+\beta, 2$, using Assumption \ref{Assum:Q0} along with the Jensen's inequality indicates that 
\begin{align}
	\|\eta\hat{\nabla} J({\bbtheta}_k)\|^c  = &\left\|\eta\sum_{t=0}^{T_k}\gamma^{t/2}\cdot R(s_t,a_t)\cdot \bigg(\sum_{\tau=0}^{t}\nabla\log\pi_{{\bbtheta}_k}(a_{\tau}\given s_{\tau})\bigg)\right\|^c\nonumber
	\\
	\leq& \eta^cU_R^c\cdot\left(\sum_{t=0}^{T_k}\gamma^{t/2}\right)^c\cdot\left\|\sum_{t=0}^{T_k}\frac{\gamma^{t/2}}{\sum_{t=0}^{T_k}\gamma^{t/2}}\sum_{\tau=0}^{t}\nabla\log\pi_{{\bbtheta}_k}(a_{\tau}\given s_{\tau})\right\|^c\nonumber
	\\
	 \leq & \eta^cU_R^c\cdot\left(\sum_{t=0}^{T_k}\gamma^{t/2}\right)^{c-1}\cdot\sum_{t=0}^{T_k}\gamma^{t/2}\left\|\sum_{\tau=0}^{t}\nabla\log\pi_{{\bbtheta}_k}(a_{\tau}\given s_{\tau})\right\|^c.
\end{align}	 
Applying the Jensen's inequality again, we can write
\begin{align}
	\|\eta\hat{\nabla} J({\bbtheta}_k)\|^c 	\leq&\eta^cU_R^c\cdot\left(\sum_{t=0}^{T_k}\gamma^{t/2}\right)^{c-1}\cdot\sum_{t=0}^{T_k}\gamma^{t/2}(t+1)^{c-1}\sum_{\tau=0}^{t}\left\|\nabla\log\pi_{{\bbtheta}_k}(a_{\tau}\given s_{\tau})\right\|^c\nonumber
	\\
	\leq&\eta^cU_R^c\cdot T_k\cdot\sum_{t=0}^{T_k}\gamma^{t/2}(t+1)^{c-1}\sum_{\tau=0}^{t}\left\|\nabla\log\pi_{{\bbtheta}_k}(a_{\tau}\given s_{\tau})\right\|^c.\label{here0}
\end{align}
Taking the expectation on both sides in \eqref{here0}, we have 
\begin{align}
	\EE[\|\eta\hat{\nabla} J({\bbtheta}_k)\|^c]  \leq & \eta^cU_R^c\cdot \EE\left[T_k\cdot\sum_{t=0}^{T_k}\gamma^{t/2}(t+1)^{c-1}\sum_{\tau=0}^{t}\left\|\nabla\log\pi_{{\bbtheta}_k}(a_{\tau}\given s_{\tau})\right\|^c\right]\nonumber
	\\
	 \leq & \eta^cU_R^c\cdot \sum_{T=0}^{+\infty}(1-\gamma^{1/2})\gamma^{T/2}\cdot T\cdot\sum_{t=0}^T\gamma^{t/2}(t+1)^cB^{c/2}\nonumber\\
 \leq & \eta^cU_R^cB^{c/2}\cdot \sum_{T=0}^{+\infty}(1-\gamma^{1/2})\gamma^{T/2}\cdot (T+1)\cdot\sum_{t=0}^{+\infty}\gamma^{t/2}(t+1)(t+2)\nonumber\\
	 = & \frac{2\eta^cU_R^cB^{c/2}}{(1-\gamma^{1/2})^4}.\label{here22}
\end{align}
Taking expectation on the both sides of \eqref{thm3} conditioning on $\bbtheta_k$, denoted as $\EE_k$ and utilizing the bound in \eqref{here22}, gives  
	\begin{align}\label{thm6}
		\mathbb{E}_k \left[   J ({{\bbtheta}}_{k+1})\right] \geq &  J ({{\bbtheta}}_{k}) +  \eta\|{\nabla} J ({{\bbtheta}}_k)\|^2 -  \frac{2M_J}{(1-\gamma^{1/2})^4}\Big(\eta^2U_R^2B+\eta^{1+\beta}U_R^{1+\beta}B^{\frac{1+\beta}{2}}+\eta \lambda U_RB^{1/2}\Big)
	\end{align}
	After rearranging the term and defining, 
	\begin{align}\label{eq:theorem41_constant}
		L_J:=\frac{2M_J}{(1-\gamma^{1/2})^4}\cdot\max\bigg\{U_R^2B,U_R^{1+\beta}B^{\frac{1+\beta}{2}}, U_RB^{1/2}\bigg\},
	\end{align}
	we can write \eqref{thm6} as 
	\begin{align}\label{thm77}
		\mathbb{E}_k \left[   J ({{\bbtheta}}_{k+1})\right] \geq &  J ({{\bbtheta}}_{k}) +  \eta\|{\nabla} J ({{\bbtheta}}_k)\|^2 -     L_J \left(\eta^{1+\beta}+\eta^{2}+\eta \lambda \right)
		\\
		\geq &  J ({{\bbtheta}}_{k}) +  \eta\|{\nabla} J ({{\bbtheta}}_k)\|^2 -     2L_J \eta^{1+\beta}-     L_J \eta \lambda.
	\end{align}
	Let $J^*$ be the optimal function value, then it holds that $J ({\bbtheta}_{k+1} ) \leq J^*$. Calculating the total expectation in \eqref{thm6} and taking sum from $k =0, \ldots, K-1$, we get
	\begin{align}\label{eq:gradient_sum}
		\sum_{k=0}^{K-1} \mathbb{E} \left[  \left\| {\nabla} J ({{\bbtheta}}_k)\right\|_2^2\right] \leq  \frac{J^*- J ({{\bbtheta}}_{0}) }{\eta } + 2K L_J \eta^{\beta} + K L_J  \lambda .
	\end{align}
	Divide both sides by $K$, we get
	\begin{align}\label{thm7}
		&\frac{1}{K} \sum_{k=0}^{K-1} \mathbb{E} \left[  \left\| {\nabla} J ({{\bbtheta}}_k)\right\|_2^2\right] \leq  \frac{  J^* -J ({{\bbtheta}}_{0})}{\eta K}  +  2L_J \eta^{\beta} + L_J  \lambda  .
	\end{align}
	Now we specify the step-size as a constant $\eta = c_{\beta}/ K^{\frac{1}{1+\beta}}$ with $c_{\beta} =   \left(  \frac{1}{2\beta L_J} \left(J^* -J ({{\bbtheta}}_{0})\right) \right)^{1/(1+\beta)} $. Doing so permits us to rewrite \eqref{thm7} as follows
	\begin{align}
		&\frac{1}{K} \sum_{k=0}^{K-1} \mathbb{E} \left[  \left\| {\nabla} J ({{\bbtheta}}_k)\right\|_2^2\right] \leq \frac{a_{\beta}}{K^{\frac{\beta}{1+\beta}}} +\mathcal{O}(\lambda),
	\end{align}
	where we define the constant
	$a_{\beta} = \left(2L_J\right)^{1/(\beta+1)} \left(J^* -J ({{\bbtheta}}_{0})\right)^{\beta/(\beta+1)}
	\; ,$ as stated in \eqref{Thm1}. Observe that in existing analyses in the literature\cite{zhang2020global}, the $\mathcal{O}(\lambda)$ term is assumed to be null. Thus, standard rates are recovered as a special case.
\end{proof}

\section{Instantiations of $\lambda$ in Example \ref{eg:gaussian_fixed_var}-\ref{eg:alpha_stable} }\label{apx_delta}
In this section, we discuss about the parameter $\lambda$ in detail. Note that the specific value of \blue{$\lambda$} would depend upon the policy class being considered. Therefore, we derive the values of $\lambda$ for Example \ref{eg:light_tailed}-\ref{eg:alpha_stable} (Example \ref{eg:gaussian_fixed_var} is special case of Example \ref{eg:light_tailed} for $\alpha=2$). \blue{For the sake of analysis in this section, we assume that  $\theta$ belongs to some compact set $\Theta$.}
%
	
	(1) We start with the Example \ref{eg:alpha_stable}, for which we note that the score function is absolutely bounded over the full action space $\mathcal{A}$. For this case, then, $\lambda=0$ and $B(\lambda)$ exists and is finite.
	
	(2) For the moderate tail case (Example \ref{eg:light_tailed}), note that the policy distribution is given by 
	\begin{align}\label{policy}
		\pi_{\bbtheta}(a|s)=\frac{1}{\sigma \mathcal{A}_\alpha}\exp\bigg\{-\frac{\|a-\phi(s)^T\bbtheta\|^\alpha}{\sigma^\alpha}\bigg\}.
	\end{align}
Therefore the score function could be written as
\begin{align}
\|\nabla_{\bbtheta}\log \pi_{\bbtheta}(a|s)\| = \sigma^{-\alpha} \|a-\phi(s)^T\bbtheta\|^{\alpha-1}\|\phi(s)\| 
\leq \sigma^{-\alpha} \mathcal{D}_\phi\|a-\phi(s)^T\bbtheta\|^{\alpha-1}\label{score_function}
\end{align}
as long as the feature map is bounded as $\|\phi(s)\|\leq \mathcal{D}_\phi$. 
Suppose $\theta$ belongs to some compact set $\Theta$. 
Let us  construct the set $\mathcal{C}$ as  $$\mathcal{C}:= \{a\in\mathcal{A}: \exists\bbtheta\in\Theta\,\,s.t.\,\, |a-\phi(s)^T\bbtheta|\leq R\}$$ where $R$ is a finite positive constant. 
Now, let us look at the following integral
\begin{align}
		\int_{\mathcal{A}\backslash\mathcal{C}}&\|\nabla\log\pi_{{\bbtheta}}(a | s)\|\cdot\pi_{{\bbtheta}}(a|s) \cdot  da \nonumber
		\\
		\leq&\sigma^{-\alpha} \mathcal{D}_\phi\int_{\mathcal{A}\backslash\mathcal{C}}\|a-\phi(s)^T\bbtheta\|^{\alpha-1} \cdot\pi_{{\bbtheta}}(a|s) \cdot  da
\end{align}
which follows from the upper bound in  \eqref{score_function}. From the definition in \eqref{policy}, we can write
\begin{align}
	\int_{\mathcal{A}\backslash\mathcal{C}}&\|\nabla\log\pi_{{\bbtheta}}(a | s)\|\cdot\pi_{{\bbtheta}}(a|s) \cdot  da \nonumber
	\\
	\leq& \frac{\mathcal{D}_\phi}{\sigma^{1+\alpha}\!\!\mathcal{A}_\alpha}\int_{\mathcal{A}\backslash\mathcal{C}}\|a-\phi(s)^T\bbtheta\|^{\alpha-1} \cdot\exp\bigg\{\!\!-\!\frac{\|a\!-\!\phi(s)^T\bbtheta\|^\alpha}{2\sigma^\alpha}\bigg\}\cdot\exp\!\bigg\{\!\!-\!\!\frac{\|a-\phi(s)^T\bbtheta\|^\alpha}{2\sigma^\alpha}\bigg\}   da\nonumber
	\\
	\leq& \frac{\mathcal{D}_\phi}{\sigma^{1+\alpha}\mathcal{A}_\alpha}\int_{\mathcal{A}\backslash\mathcal{C}} \|a-\phi(s)^T\bbtheta\|^{\alpha-1} \cdot\exp\bigg\{-\frac{\|a-\phi(s)^T\bbtheta\|^\alpha}{2\sigma^\alpha}\bigg\}\cdot  da\cdot\exp\bigg\{-\frac{R^\alpha}{2\sigma^\alpha}\bigg\} 
	\nonumber
	\\
	\leq& \frac{\mathcal{D}_\phi}{\sigma^{1+\alpha}\mathcal{A}_\alpha} \cdot\sigma^{\alpha}\cdot B_\alpha\cdot\exp\bigg\{-\frac{R^\alpha}{2\sigma^\alpha}\bigg\}
		\nonumber
	\\
	\leq& \frac{\mathcal{D}_\phi}{\sigma\mathcal{A}_\alpha} \cdot B_\alpha\cdot\exp\bigg\{-\frac{R^\alpha}{2\sigma^\alpha}\bigg\},
\end{align}
where $B_\alpha := \int |a|^{\alpha-1}\exp\{-\frac{|a|^\alpha}{2}\}<\infty$.
The above equation will be less than $\lambda$ if we have
\begin{align}
	\frac{\mathcal{D}_\phi}{\sigma\mathcal{A}_\alpha} \cdot B_\alpha\cdot\exp\bigg\{-\frac{R^\alpha}{2\sigma^\alpha}\bigg\}  \leq \lambda,
\end{align}
which implies that 
\begin{align}
	\left(\frac{R}{\sigma}\right)^\alpha \geq 2\log \left(\frac{\mathcal{D}_\phi B_\alpha}{\sigma\mathcal{A}_\alpha \lambda} \right).
\end{align}
The above expression provides the bound for $B(\lambda)$ as 
\begin{align}
	B(\lambda)\leq &\max_{a\in\mathcal{C}}\max_{s\in\cS}\max_{\bbtheta\in\Theta} \|\nabla_{\bbtheta}\log \pi_{\bbtheta}(a|s)\|\leq 	\max_{a\in\mathcal{C}}\max_{s\in\cS}\max_{\bbtheta\in\Theta}\frac{\mathcal{D}_\phi}{\sigma}\left(\frac{\|a-\phi(s)^T\bbtheta\|}{\sigma}\right) ^{\alpha-1}
	\nonumber
	\\
	\leq & 	\frac{\mathcal{D}_\phi}{\sigma}\left(\frac{\mathcal{D}_\Theta\mathcal{D}_\phi}{\sigma}+2\log \left(\frac{\mathcal{D}_\phi B_\alpha}{\sigma\mathcal{A}_\alpha \lambda} \right)\right) ^{\frac{\alpha-1}{\alpha}}
		\nonumber
	\\
	= & 	\mathcal{O}\left(\log\frac{1}{\lambda}\right) ^{\frac{\alpha-1}{\alpha}},\label{bound}
\end{align}
where $\mathcal{D}_\Theta:=\max_{\bbtheta,\bbtheta'\in\Theta}\|\bbtheta-\bbtheta'\|$ is the diameter of $\Theta$. Note that the bound in \eqref{bound} is small even for a very small value of $\lambda$. This permits us to relax the standard assumption of absolutely bounded score function for continuous action spaces.  

%% file: Appendix_D.tex
\section{Exit Time Analysis}
We first present the technical preliminaries required for the analysis in this section as follows.
\subsection{Technical Preliminaries}\label{subsec:exit_time_preliminaries}

{\bf \noindent Technical Results for Univariate Case $d=1$.}
We continue then with the formal definition of  first exit time for a SDE in continuous one dimensional case for simplicity. Consider a  neighborhood  $B_i: =  [-b,\, a]$ around the $i$-th local extrema $\bar{\bbtheta}_i$ in a single dimension. Let the process is initialized at $\bbtheta_0$ which is   inside $B_i$.  We are interested in the first
exit time from $B_i$ starting from a point $\bbtheta_0  \in B_i$. The  first exit time from  $B_i: =  [-b,\, a]$ for a process defined by continuous SDE  is defined as (also stated as Definition \ref{def:exit_time00})
%
\begin{align*}
\hat{\tau}(\epsilon) =\inf\{  t \geq 0: \bbtheta_t^{\epsilon} \notin [-b,a] \},
\end{align*}
as random perturbation  in the SDE, $\epsilon \to 0$, [cf. \eqref{eqn:SGD_L1}].  We denote the first exit time for continuous SDE using $\hat{\tau}(\cdot)$. Under the Assumption \ref{Assump:fun} (3.), we invoke results from \cite{imkeller2006first} for the first exit times  of  continuous SDEs \eqref{eqn:SGD_L1} in the univariate case $d=1$.

Further, we impose that the multi-dimensional L\'{e}vy motion $\bbL_t^\alpha$ in \eqref{eqn:SGD_L1} admits a representation as a $\bbL_t^\alpha=\bbr L_t$ with $\bbr\in\mathbb{R}^d$ as a standard basis vector in $d$-dimensions, which determines the direction of the jump process, and $L_t$ is a scalar  $\alpha$-stable L\'{e}vy motion. This restriction is needed in order to tractably study the transient behavior of \eqref{eqn:SGD_L1} in terms of its exit time from regions of attraction \cite{imkeller2010first}, specifically, in applying Lemma \ref{lemma:bound}, as well as characterizing the proportion of time jumping between its limit points, to be discussed next. We note that such analyses for general $d$-dimensional  L\'{e}vy motion is an open problem in stochastic processes. 

\begin{theorem}\label{th:2}  \cite{imkeller2006first}
Consider the SDE \eqref{eqn:SGD_L1}, in the univariate case $d = 1$ ($\theta \leftarrow \bbtheta $) and
assume that it has a unique strong solution. Assume further that the there exists an objective J with  a global
maximum at zero, satisfying the conditions $J'(\theta) 
\theta < 0$ for every $\theta \in \mathbb{R}$, $J(0) = 0$, $J'(\theta) = 0 $ if and only if $\theta=0$ and $J^{''}(0) < 0$. Then, there exist positive constants $\epsilon_0$, $\gamma$,  $\delta$, and $C > 0$ such that for
$0 <  \epsilon \leq \epsilon_0$, the following holds in the limit of small $\epsilon$:
\begin{align}
\exp\bigg(-u\epsilon^{\alpha}\bigg(\frac{1}{a^\alpha}\bigg)\frac{(1+C\epsilon^\delta)}{\alpha}\bigg)&(1-C\epsilon^\delta) \leq \mathcal{P} (\hat{\tau}(\epsilon) >u)  \\
&\hspace{2.5cm}\leq \exp\bigg(-u\epsilon^{\alpha}\bigg(\frac{1}{a^\alpha}\bigg)\frac{(1+C\epsilon^\delta)}{\alpha}\bigg) (1+C\epsilon^\delta)\nonumber
\end{align}
uniformly for all $\theta \leq a-\epsilon^{\gamma}  $  and $u \geq 0$. Consequently
\begin{align}
\mathbb{E} \left[ \hat{\tau}_a  (\epsilon)  \right] = \frac{\alpha a^{\alpha} }{\epsilon^{\alpha}}  (1+\mathcal{O}(\epsilon^\delta))
\end{align}
uniformly for all $\theta \leq a-\epsilon^{\gamma}  $.
\end{theorem} 

\begin{theorem}\cite{imkeller2006first},
Consider the SDE \eqref{eqn:SGD_L1}, in dimension $d = 1$ and
assume that it has a unique strong solution. Assume further that there exists an  objective $J$ with a global
maximum  at zero, satisfying the conditions $J'(\theta)
\theta < 0$ for every $\theta \in \mathbb{R}$, $J(0) = 0$, $J'(\theta) = 0 $ if and only if $\theta=0$ and $J^{''}(0)<0$, the  following results hold in the limit if small $\epsilon$:
\begin{enumerate}
\item First exit time is exponentially large in $\epsilon^{-2}$. Assume for definiteness $J(a) >J(-b)$. Then for any $\delta >0$ , $ \theta  \in B_i$.
\begin{align}
\mathcal{P}_{\bbtheta} (\exp{(-2J (a)-\delta)/\epsilon^2 }< \hat{\tau} (\epsilon) < \exp{(-2J (a)+\delta)/\epsilon^2 )} \to 1 \,  \textrm{as} \, \epsilon \to 0 
\end{align}
Moreover, $\epsilon^2 \log \mathbb{E}_{\bbtheta} [\hat{\tau}(\epsilon)] \to 2J (a)$.
\item The mean of first exit time is given by 
\begin{align}\label{SGD:extg}
\mathbb{E}_{\theta} (\hat{\tau}(\epsilon) \approx \frac{\epsilon \sqrt{\pi}}{J'(a)\sqrt{J^{''}(0)}} \exp{(2J(a)/\epsilon^2)}
\end{align}
\item Normalized first exit time is exponentially distributed: for $u \geq 0$
\begin{align}
\mathcal{P}_{\bbtheta} \left(  \frac{\hat{\tau}(\epsilon)}{\mathbb{E}_{\theta} (\hat{\tau}(\epsilon)} >u \right) \to \exp{(-u)}\, \textrm{as}\, \epsilon \to 0
\end{align}
uniformly in $\theta$ on compact subsets of $(-b,a)$.
\end{enumerate}
\end{theorem}
Note that the above results hold for   continuous SDEs in one dimensional space. 
\\
Next we  present extend the exit time results  from   a domain $ \mathcal{G} _i \subset \mathbb{R}^d$ around $i$-th  local maxima of $J(\cdot)$, $\bar{\bbtheta}_i$  \cite{imkeller2010first} with an assumption that the system is perturbed by a single one-dimensional L\'{e}vy process with $\alpha$-stable component. 

{\bf \noindent Multi-Dimensional Case $d>1$.}
Before proceeding to the the statement of results and proofs, we define assumptions and  terminologies associated with the multi-dimensional space, i.e., subsequently $\bbtheta\in\mathbb{R}^d$. The reason for separately stating the scalar case and the multi-dimensional case, is that results and conditions for the scalar-dimensional case are invoked in generalizing to the multi-dimensional case, specifically, in Lemmas \ref{Lemma:Levy1} and \ref{Lemma:Gurbuzb}.
 For simplicity, we assume the tail-index ($\alpha$)  of perturbations are identical in all the directions.  The dynamical system of \eqref{eqn:SGD_L1} when perturbed by single dimensional L\'{e}vy process  is given by 
\begin{align}\label{eqn:Levygeneric}
\bbtheta_t^{\epsilon}(\bbtheta_0) = \bbtheta_0 +\int_{0}^{t} b(\bbtheta_s^{\epsilon} (\bbtheta_0))  ds + \epsilon  \bbr L_t^i, \epsilon>0, \bbtheta \in \mathcal{G}, t \geq 0 
\end{align}
where, $\bbr \in \mathbb{R}^d$  is the  unit vector.

We define the inner parts of $\mathcal{G}_i$ by $\mathcal{G}_{i_{\bar{\delta}}}:= \{z \in \mathcal{G}_i: \textrm{dist}(z, \partial \mathcal{G}_i) \geq \bar{\delta}  \}$.  Therefore, the following holds: Sets $\mathcal{G}_{\bar{\delta}}$ are positively invariant for all $\bar{\delta}  \in (0, a+\xi)$ (cf. \eqref{eqn:_Gi}), in the sense that the
deterministic solutions starting in $\mathcal{G}_{\bar{\delta}}$ do not leave this set for all times $t \geq 0$.
 We have $\Omega^-(\bar{\delta}) \cap \mathcal{G}_{i_{\bar{\delta}}}^c(\bar{\delta}) = \emptyset $ and $\Omega^+(\bar{\delta}) \cap   \mathcal{G}_{i_\delta}^c(\bar{\delta}) = \emptyset $.  The preceding statements follows from \cite{imkeller2010first}.

 Next we state exit time results from  the domain, $\mathcal{G}_i$ for a system defined by  \eqref{eqn:Levygeneric}.
\begin{theorem}\label{th:21}  Expressions 3.4, 3.8 \cite{imkeller2010first}
For $\bar{\delta} \in (0, \bar{\delta}_0)$ and initial state $\bbtheta_0  \in \mathcal{G}_i$,   $\bbtheta_t^{\epsilon}$ following \eqref{eqn:Levygeneric}  exits from the domain $\mathcal{G}_i$ in a little tube in the direction of $\alpha$,
Furthermore, for every $\bar{\delta}\in  (0, \bar{\delta}_0)$  the probability to exit in direction of perturbation $+ \bbr$ is given by
\begin{align}
&  \mathcal{P}^{\bbtheta_0} \left( \bbtheta_{\hat{\tau}}^{\epsilon}  \in \Omega^{+}_i (\delta) \right) =  \frac{2}{\epsilon^{\rho \alpha}} \epsilon^{\alpha} (d^+)^{-\alpha} ,
\end{align}
where  $(a,b) >0$,  $\rho \in (0,1)$, $\min (a,b) > \epsilon^{1-\rho}$.

Furthermore, for every $\bar{\delta} \in  (0, \bar{\delta}_0)$  the probabilities to exit in direction $\pm \bbr$ are given by
\begin{align}
& \lim_{\epsilon \to 0} \mathcal{P} \left( \bbtheta_{\hat{\tau}}^{\epsilon}  \in \Omega_i^+ (\bar{\delta}) \right) =\frac{p^+}{p_s} \\
& \lim_{\epsilon \to 0} \mathcal{P} \left( \bbtheta_{\hat{\tau}}^{\epsilon}  \in \Omega_i^- (\bar{\delta}) \right) =\frac{p^-}{p_s},
\end{align}
 for all $\bbtheta_t \in \mathcal{G}$
and 
\begin{align}
& p_s:= ((d^+)^{-\alpha} + (-d^-)^{-\alpha})\\
& p^+ :=  (d^+)^{-\alpha} \\
& p^- :=(-d^-)^{-\alpha},
\end{align}
$d^+$ and $d^-$ define distance from boundary of interest $\partial \mathcal{G}_i$ along $\pm \bbr$. 
\end{theorem}
 The above expression is obtained by using a single $\bbr$ in (3.4) of  \cite{imkeller2010first}.   It is to be noted that for a general process perturbed by finitely many  single dimensional L\'{e}vy processes with different tail indices,  the exit time depends on the smallest tail-index, and the system exits from the domain in the direction of the process with smallest $\alpha_i$.

In this section we  derive the first exit time behavior for the proposed heavy tailed setting of  {Algorithm 1}. We proceed by defining some key quantities of interest and lemmas used in the proof of  Theorem \ref{Theorem:Exit_time}. The first is the It\^{o} formula for stochastic differential equations, and then we present the Bellman-Gronwall inequality.
 \begin{definition}{(It\^{o} formula)\cite{xie2020ergodicity}} \label{Def: Ito}
 Let $N$ be a Poisson random measure  with intensity measure $dt \nu(dz)$, where $\nu$ is a L\'{e}vy measure on $\mathbb{R}^d$, ie., $ \int_{\mathbb{R}^d} (|z|^2 \wedge 1 ) \nu(dz) < + \infty, \, \nu(\{0\}) =0 $. The compensated Poisson random measure $\tilde{N}$ is defined as $ \tilde{N}(dt, dz):= N(dt,dz) - dt \nu(dz) $. Consider the following SDE  in $\mathbb{R}^d$ with jumps:
\begin{align*}
 dX_t = b_t(X_t) dt + \int_{|z| <R} g_t (X_{t-}, Z) \tilde{N} (dt, dz)  +  \int_{|z| \geq R} g_t (X_{t-}, Z) {N} (dt, dz),
\end{align*}
 where $R>0$ is a fixed constant.   Suppose $g(\bbx) \in \mathcal{C}^2(\mathbb{R})$ is a twice continuously differentiable function (in particular all second-partial-derivatives are continuous functions). Suppose $Y_t = g(\bbX_t) $ is again an It\^{o} process, then we have 
%
\begin{align*}
dY_t = [\mathcal{L}_{1}^{b_t} g + \mathcal{L}_{\nu}^{g} h ] (\bbX_t) dt + dM(t),
\end{align*}
 where $M_t$ is local martingale, $\mathcal{L}_{1}^{b_t}$ is the first order differential operator associated with drift ($b_t$), and $\mathcal{L}_{\nu}^{g} h $ is the non-local operator associated with jump coefficient $g()$ such that:
\begin{align*}
 \mathcal{L}_{\nu}^{g}  u(x) := & \int_{|z| <R} \left[  u(x+ g_t(x,z) ) - u(x) - g_t(x,z)\cdot \nabla u(x)  \right] \nu(dz)  \nonumber \\
 & \int_{|z| \geq R} \left[  u(x+ g_t(x,z) ) - u(x)  \right] \nu(dz)  
\end{align*}
\end{definition}
\begin{definition}{(Bellman-Gronwall inequality)} \label{Def: Gronwall}
Assume $\phi: [0,T] \to \mathbb{R}$ 
 is a bounded nonnegative measurable function, $C:[0,T] \to \mathbb{R}$
 is a nonnegative integrable function and $B \geq 0$
 is a constant with the property that
 \begin{align}
\phi(t) \leq B + \int_{0}^{t} C(\tau) \phi(\tau) d\tau \forall t \in  [0, T].
 \end{align}
 Then
 \begin{align}
\phi(t) \leq B \exp{\left( \int_{t=0}^{T} C(\tau) d\tau \right)} \, \forall t \in  [0, T]. 
 \end{align}
\end{definition}
Next we provide a lemma regarding the difference between a  L\'{e}vy process with two different tail indices $\alpha$ and $k\eta$.
\begin{lemma} \label{Lemma:Levynoise} \cite{thanh2019first}
For any $u > 0$, $\eta > 0$ and $K \in \mathbb{N}$, there exist a constant $C_{\alpha}$  such
that:
\begin{align}
\max_{k \in 0,\ldots, K-1} \mathcal{P}[ \sup_{t \in [k\eta, (k+1)\eta ]}
\| L^{\alpha}(t) -L^{k\eta}(t) \| \geq u   ]  \leq C_{\alpha} d^{1 + \frac{\alpha}{2}} \eta u^{-\alpha} 
\end{align}
and
\begin{align}
\mathcal{P}[ \max_{k \in 0,\ldots, K-1}  \sup_{t \in [k\eta, (k+1)\eta ]}
\| L^{\alpha}(t) -L^{k\eta}(t) \| \geq u   ]  \leq 1 -  (1-C_{\alpha} d^{1 + \frac{\alpha}{2}} \eta u^{-\alpha} )^K
\end{align}
\end{lemma}
Next we state a stochastic variant of Gronwall's inequality which is also used in the proof.
\begin{lemma}{Stochastic Gronwall’s inequality  \cite{scheutzow2013stochastic} } \label{Lemma:Stochastic GI}
Let $Z$ and  $H$ be nonnegative, adapted processes with continuous path and assume that $\psi$ is nonnegative and progressively measurable. Let $M$ be a continuous local martingale starting at 0. If 
\begin{align*}
Z(t) \leq \int_{0}^{t} \psi(s)Z(s) + M(t) + H(t) 
\end{align*}
holds for all $t \geq 0$, then for $p \in (0,1)$ and $\mu$ and $\nu >1$ such that  $\frac{1}{\mu} + \frac{1}{\nu} =1$ and $p\nu <1$, we have 
\begin{align*}
\EE \left(  \sup_{0 \leq s \leq t} Z^{p} (s)\right)  \leq (c_{p\nu} +1)^{1/\nu} \left(   \EE \exp{\left(  p \mu \int_{0}^{t} \psi(s) ds  \right)}\right)^{1/\mu} \left(  \EE (H^*(t))^{p\nu} \right)^{1/\nu},
\end{align*}
where, a real valued process $Y^*(t):= \sup_{0 \leq s \leq t} Y(s)$.
\end{lemma}
Next we provide a bound on the second moment of parameter vector $\bbtheta$ when integrated with respect to  L\'{e}vy measure $\nu$.
\begin{lemma}\label{Lemma:Levy1} \cite{thanh2019first}
Let $\nu$ be the L\'{e}vy measure of a d-dimensional L\'{e}vy process $L_{\alpha}$ whose components are 
independent scalar symmetric $\alpha$-stable L\'{e}vy processes $ L_1 , \, \ldots, \,  L_d$ . Then there exists a constant $C > 0$ such that the following inequality holds with $k_1 \geq  1$ and $2 > \alpha > 1$:
\begin{align*}
 \frac{1}{k^{2/\alpha}} \int_{\|\bbtheta\| <1} \| \bbtheta\|^2 \nu(d \bbtheta) + \frac{1}{2k_1^{1/\alpha}} \int_{\|\bbtheta\| \geq 1} \| \bbtheta \| \nu(d \bbtheta) \leq C \frac{d}{k_1^{1/\alpha}} 
\end{align*}
\end{lemma}
\begin{lemma} \label{Lemma:Gurbuzb}
Given the Assumptions \ref{Assum:Q0}- \ref{Assump_Levy} and  the proposed Heavy tailed setting of \eqref{eqn:SGD_L}, for  $ \lambda \in (0, 1)$,   $\bbtheta  \in  \mathbb{R}^d$, there exists  constants  $C_1$ and $C_2$ depending on $\lambda$, dissipativity constants $(m, \, b)$, and  $\bbtheta_{0}$, 
 the following holds on the  expected value of  $\bbtheta$ for all $t>0$ such that
\begin{align}
\EE \left(  \sup_{s \in  [0,t] } (\|\bbtheta_s\|)^{\lambda} \right) \leq C_1  \left( 1 + C_2 \left(  \frac{U_R}{(1-\gamma)^2}  (m +b)/2   + C \frac{d}{k_1^{1/\alpha}}   \right)t \right)^{\lambda}, \, k_1 \geq 1, \, 1 <\alpha<2, 
\end{align}
where $(m,b)$ are the dissipativity constants from Assumption \ref{Assump: disspitat} and $k_1$ is a function of stepsize $\eta$, $k_1: = 1/\eta^{\alpha-1}$.
\end{lemma}
\begin{proof}
{Here we derive an upper bound on $\EE \left(  \sup_{s \in  [0,t] } (\|\bbtheta_s\|)^{\lambda} \right) $ for the heavy-tailed policy gradient setting of \eqref{eqn:SGD_L} perturbed by a single dimensional L\'{e}vy process  in the direction of unit vector $\bbr$. We build up on the existing results from \cite{thanh2019first,xie2020ergodicity} for L\'{e}vy process and SDE. Using It\^{o} formula, the whole expression is divided into two terms.  We simplify the second term using properties of a L\'{e}vy process from \cite{thanh2019first,xie2020ergodicity}.  Further we build upon the heavy-tailed setting and its  properties such as dissipativity of score function and H\'{o}lder continuity  to simplify the first term. Further the combined expression is simplified using direct application of stochastic Gronwall's inequality.}
We start with the continuous equivalent of  \eqref{eqn:SGD_L} defined by 
$ d\bbtheta_t = b(\bbtheta_{t-}) dt  + k_1^{-1/\alpha}  dL_t^{\alpha}
$ ($t_{-}$ denote left limit of the process)  and $k_1: = 1/\eta^{\alpha-1}$ ($k_1$ is always greater than 1 as $\eta^{\alpha-1}<1$), we use $F =-J$ for simplicity,  $b(\cdot) := - \nabla F(\cdot)$ using \eqref{eqn:SGD_L1}.  Note that this representation is equivalent to \eqref{eqn:SGD_L1} and helps to invoke some of the existing results in the analysis. Here the direction of perturbation $r$ is absorbed   into $L_t^{\alpha}$  without the loss of generality (cf.  \eqref{eqn:Levygeneric}).  In order to upper bound $\EE \left(  \sup_{s \in  [0,t] } (\|\bbtheta_s\|)^{\lambda} \right)$, we  start by defining a function, $g_1(\bbtheta) \triangleq  (1+ \|\bbtheta\|^2)^{1/2}$. Using direct application of It\^{o}'s formula from Definition \ref{Def: Ito}  with jump coefficient $g$ being $k_1^{-1/\alpha}$, we can write It\^{o} formula for  \eqref{eqn:SGD_L} as follows
\begin{align} \label{eqn:ito}
dg_1(\bbtheta_t ) =  & \underbrace{\left< b(\bbtheta_t), \nabla g(\bbtheta_t )) \right>}_{T_1} dt \nonumber \\
& + \underbrace{ \int_{\mathbb{R}^d} \left( g_1(\bbtheta_t + k_1^{-1/\alpha} \bbtheta) -g_1(\bbtheta_t)-\mathbb{I}_{\|\bbtheta \| <1}\left<k_1^{-1/\alpha}, \nabla g_1(\bbtheta_t)  \right>) \nu (d\bbtheta)\right) }_{T_2} dt + dM(t),
\end{align}
here, $M(t)$ is defined as local martingale and $\nu$  be the L\'{e}vy measure of a d-dimensional L\'{e}vy process $L^{\alpha}$ .  We have $\partial_i g_1(\bbtheta) = \bbtheta_i (1+\|\bbtheta\|^2)^{-1/2}/2$. 

Next we unfold the expressions for $T_1$ and $T_2$ using the expression for the policy gradient in \eqref{eq:policy_gradient_iteration}.
\paragraph{Expression for $T_1$:}
 Expression for  $T_1$
 along with the policy gradient of \eqref{eq:policy_gradient_iteration} takes the form
\begin{align}
T_1 = \left< -\nabla F(\bbtheta), \nabla g_1(\bbtheta) \right>  = \left< - \frac{1}{1-\gamma} \EE_{(s,a) \sim \rho_{\bbtheta}}\left[ \nabla \log \pi_{\bbtheta}(s,a) Q_{\pi_{\bbtheta}}(s,a)  \right] , \bbtheta  \right> (1+ \| \bbtheta\|^2)^{-1/2}/2
\end{align} 
Now using Assumption \ref{Assum:Q0} to  upper bound    $\|\hat{Q}_{\pi_{{\bbtheta}}}(s,a)\|$, expression reduces to
\begin{align}
T_1 &  \leq \frac{U_R}{1-\gamma} \left< - \frac{1}{1-\gamma} \EE_{(s,a) \sim \rho_{\bbtheta}}\left[ \nabla \log \pi_{\bbtheta}(s,a)  \right] , \bbtheta  \right> (1+ \| \bbtheta\|^2)^{-1/2}/2 \\
& \leq \frac{U_R}{1-\gamma} \EE_{(s,a) \sim \rho_{\bbtheta}} \underbrace{\left[ \left< -\frac{1}{1-\gamma}  \nabla \log \pi_{\bbtheta}(s,a)  , \bbtheta  \right> \right] }_{T_3} (1+ \| \bbtheta\|^2)^{-1/2}/2
\end{align} 
The expression, $T_3$ is simplified using $(m,b, \beta)$-dissipative assumption of  Assumption \ref{Assump: disspitat} ($T_3 =  - \left<  \nabla \log \pi_{\bbtheta}(s,a)  , \bbtheta  \right> \leq -m \| \bbtheta\|^{1+\beta}+b $).  Unfolding the expectation operator defined with respect to the occupancy measure of the MDP under policy $\pi_\theta$ allows us to write 
\begin{align}
T_1 = \left< -\nabla F(\bbtheta), \nabla g_1(\bbtheta) \right>   \leq &  \frac{U_R}{(1-\gamma)^2} \left( \int_{ \mathcal{S} \times \mathcal{A}}   \left(  -m\left|  \bbtheta \right\|^{1+\beta} +b  \right) (1-\gamma)  {\rho_{\pi_{{\bbtheta}}}(s) } \cdot  \pi_{{\bbtheta}}(a|s)     \,  ds  \, da )\right) \nonumber \\
& \times (1+ \| \bbtheta\|^2)^{-1/2}/2
\end{align} 
The integral $\left(   \int_{ \mathcal{S} \times \mathcal{A}}  (1-\gamma)  {\rho_{\pi_{{\bbtheta}}}(s) } \cdot  \pi_{{\bbtheta}}(a|s)     \,  ds  \, da \right)   $ from the preceding expression  is a valid probability measure and hence integrates to unit. Therefore, we can simplify the right-hand side as
\begin{align}\label{eqn:intrim}
T_1& \leq  \frac{U_R}{(1-\gamma)^2}  \underbrace{\left(  -m\left|  \bbtheta \right\|^{1+\beta} +b  \right)}_{T_4} (1+ \| \bbtheta\|^2)^{-1/2}/2
\end{align} 
Adding and subtracting $m$ inside $T_4$ gives 
\begin{align}\label{eqn:intrim}
T_1& \leq  \frac{U_R}{(1-\gamma)^2}  {\left(  -m \underbrace{\left(\left\|  \bbtheta \right\|^{1+\beta}+1\right)}_{T_5}+m +b  \right)}(1+ \| \bbtheta\|^2)^{-1/2}/2
\end{align} 
In order to simplify the term $T_5$, we evaluate $(1+ \|\bbtheta\|^2)^{\beta_1} $.

 It turns out that the application of Bernoulli's inequality is advantageous when we split the evaluation of $T_5$ into  two cases, namely, $\|\bbtheta\|^2 < 1$ and $\|\bbtheta\|^2 >  1$ and relate the resultant inequality to $g_1 (\bbtheta)$. Consider the case  where $\|\bbtheta\|^2 <1$,  using Bernouli's inequality for  $0< \beta_1= (1+\beta)/2<1 $, 
\begin{align}
(1+ \|\bbtheta\|^2)^{\beta_1} \leq 1+ \beta_1 \|\bbtheta\|^2
\end{align}
As $0< \beta_1<1$, we get
\begin{align*}
(1+ \|\bbtheta\|^2)^{\beta_1} \leq 1+ \beta_1 \|\bbtheta\|^2 \leq 1+ \|\bbtheta\|^2
\end{align*}
Similarly, as   $0< \|\bbtheta\|^2 <1$, $ \|\bbtheta\|^2 <\|\bbtheta\|^{2\beta_1}$ and we have
\begin{align*}
(1+ \|\bbtheta\|^2)^{\beta_1} \leq 1+\|\bbtheta\|^2 \leq 1+ \|\bbtheta\|^{2\beta_1}
\end{align*}
Next, consider the case, $\|\bbtheta\|^2>1$, therefore $1/\|\bbtheta\|^2 <1$. 
Further following the same argument from previous case with Bernouli's inequality
\begin{align}
\left(1+ \frac{1}{\|\bbtheta\|^2}\right)^{\beta_1} \leq  1+\frac{\beta_1}{\|\bbtheta\|^2} \leq 1+\frac{1}{\|\bbtheta\|^2} \leq 1+\frac{1}{\|\bbtheta\|^{2\beta_1}}
\end{align}
Multiplying both side of the above inequality by $\|\bbtheta\|^{2{\beta_1}}$
\begin{align}
\left(\|\bbtheta\|^2+ 1\right)^{\beta_1}  \leq \|\bbtheta\|^{2\beta_1}+ 1
\end{align}
Therefore, using $\beta_1 = (1+\beta)/2$, we have 
\begin{align}
\left(\|\bbtheta\|^2+ 1\right)^{(1+ \beta)/2}  \leq \|\bbtheta\|^{1+ \beta}+ 1
\end{align}
Now we can write the above inequality for all values of $\|\bbtheta\|$ as 
\begin{align}\label{eqn:intrim1}
-m (\|\bbtheta\|^2+1)^{(1+\beta)/2} \geq -m( \|\bbtheta\|^{1+\beta} +1 )
\end{align}
Now we substitute this inequality \eqref{eqn:intrim1} back into the expression for the dissipativity-based upper-bound on the policy gradient in \eqref{eqn:intrim} to write
%
\begin{align}
\left< -\nabla F(\bbtheta), \nabla g_1(\bbtheta) \right> & \leq \frac{U_R}{(1-\gamma)^2}  \left(  -m(\left\|  \bbtheta \right\|^{2}+1)^{(1+\beta)/2} +m + b  \right) (1+ \| \bbtheta\|^2)^{-1/2}/2 \\
& = \frac{U_R}{2(1-\gamma)^2} \left(-m( \|\bbtheta\|^2+1)^{\beta/2} + (m+b)\underbrace{(1+\|\bbtheta\|^2)^{-1/2}}_{\in (0,1]}\right) \\
& \leq  \frac{U_R}{2(1-\gamma)^2}  \left (-m( \|\bbtheta\|^2+1)^{\beta/2} + m+b\right) \\
& \leq  \frac{U_R}{2(1-\gamma)^2}   \left(-m g_1( \bbtheta)^{\beta}+m +b\right). \label{eqn:itofrstterm}
\end{align} 
The last inequality is obtained by plugging in the expression for $g_1(\bbtheta) = (1+\|\bbtheta\|^2)^{1/2}$.
\paragraph{Expression for $T_2$:}
 Similarly, we analyze the second term, $T_2$. Note that $T_2$ does not depend on the heavy-tailed gradients and  simplification is based on the standard results  on properties of L\'{e}vy process  from \cite{xie2020ergodicity}.
We get
\begin{align}
T_2 = &\int_{\mathbb{R}^d} \left( g_1(\bbtheta_t + k^{-1/\alpha} \bbtheta) -g_1(\bbtheta_t)-\mathbb{I}_{\|\bbtheta \| <1}\left<k^{-1/\alpha}, \nabla g_1(\bbtheta_t)  \right>) \nu (d\bbtheta)\right)  dt\label{eqn:itso} \\
& \leq
 \frac{1}{2k_1^{2/\alpha}} \int_{\|\bbtheta\| <1} \| \bbtheta\|^2 \nu(d \bbtheta) + \frac{1}{2k_1^{1/\alpha}} \int_{\|\bbtheta\| \geq 1} \| \bbtheta \| \nu(d \bbtheta) \leq C \frac{d}{k_1^{1/\alpha}}  
\end{align}
Note that the above expression is a standard result for L\'{e}vy process  (\cite{xie2020ergodicity}, Expression 7.6). Using Lemma \ref{Lemma:Levy1} on the right-hand-side inequality, we get 
\begin{align}
T_2  
& \leq
 \frac{1}{2k_1^{2/\alpha}} \int_{\|\bbtheta\| <1} \| \bbtheta\|^2 \nu(d \bbtheta) + \frac{1}{2k_1^{1/\alpha}} \int_{\|\bbtheta\| \geq 1} \| \bbtheta \| \nu(d \bbtheta) \leq C \frac{d}{k_1^{1/\alpha}}  \label{eqn:itosecndterm}
\end{align}
 Using \eqref{eqn:itofrstterm}  and \eqref{eqn:itosecndterm} in \eqref{eqn:ito} and integrating the expression  from 0 to $t$ gives
\begin{align} \label{eqn:ito1}
g_1(\bbtheta_t )- g_1(\bbtheta_0) \leq & \int_{0}^{t}  \left(  \frac{U_R}{(1-\gamma)^2}  (-m g_1( \bbtheta_t)^{\beta}+m +b)/2   + C \frac{d}{k_1^{1/\alpha}}   \right) ds +  M(t) \\
& \leq \int_{0}^{t}  \left( \frac{U_R}{(1-\gamma)^2}  (m +b)/2   + C \frac{d}{k_1^{1/\alpha}}   \right) ds +  M(t)
\end{align}
%
The above expression can be simplified using  Gronwall's inequality for stochastic equations (Lemma 3.8 of \cite{scheutzow2013stochastic}, Lemma \ref{Lemma:Stochastic GI}  in  the Appendix) {to  upper bound  $\EE \left(  \sup_{s \in  [0,t] }  g_1(\bbtheta_s) \right)$ in the given interval. Now following the similar argument of \cite{thanh2019first}, an upper bound on  expression on $\EE \left(  \sup_{s \in  [0,t] }  g_1(\bbtheta_s) \right)$ also upper bounds $\EE \left(  \sup_{s \in  [0,t] } \|\bbtheta_s\|\right)$  as $\|\bbtheta\| $ is always less than $\|1+\bbtheta\|$. Therefore, let us first proceed with Gronwall's inequality and obtain an upper bound on    $\EE \left(  \sup_{s \in  [0,t] }  g_1(\bbtheta_s) \right)$ and further relate it to $\EE \left(  \sup_{s \in  [0,t] } \|\bbtheta_s\|\right)$. Upon comparing the above expression with  Stochastic Gronwall's inequality  of Lemma \ref{Lemma:Stochastic GI} (Theorem 4, \cite{scheutzow2013stochastic} ) for nonnegative adapted processes $Z$ and $H$ such that 
\begin{align*}
Z(t) \leq \int_{0}^{t} \psi(s)Z(s) + M(t) + H(t) 
\end{align*}
we have $g_1(\cdot)$ equivalent to $Z(\cdot)$,   $H(\cdot):=  \int_{0}^{t}  \left(  \frac{U_R}{(1-\gamma)^2}  (m +b)/2   + C \frac{d}{k_1^{1/\alpha}}   \right) ds$, and $\sup_{s \in [0, \, t]} H^*(s) = \left(  \frac{U_R}{(1-\gamma)^2}  (m +b)/2   + C \frac{d}{k_1^{1/\alpha}}   \right)t $, $p=\lambda$.  As  \eqref{eqn:ito1} holds for all $t \geq 0$, using Lemma \ref{Lemma:Stochastic GI}  we have  
\begin{align}
\EE \left(  \sup_{s \in  [0,t] } g_1(\bbtheta_s)^{\lambda} \right) \leq  (c_{p\nu} +1)^{1/\nu}  \left(  \EE g_1(\bbtheta_0) + \left(  \frac{U_R}{(1-\gamma)^2}  (m +b)/2   + C \frac{d}{k_1^{1/\alpha}}   \right)t \right)^{\nu}
\end{align}
where,  $ \nu>0, \, c_{p\nu}:= (4 \wedge  \frac{1}{p\nu}) \frac{\sin \pi\nu}{\pi p\nu}$ (Proposition 1, \cite{scheutzow2013stochastic}).
Now for  $\lambda \nu <1, \, p \in (0,1)$ from Lemma \ref{Lemma:Stochastic GI}, we get
\begin{align}
\EE \left(  \sup_{s \in  [0,t] } g_1(\bbtheta_s)^{\lambda} \right) \leq c_{\lambda}  \left(  \EE g_1(\bbtheta_0) + \left(  \frac{U_R}{(1-\gamma)^2}  (m +b)/2   + C \frac{d}{k_1^{1/\alpha}}   \right)t \right)^{\lambda}
\end{align}
where,  $c_{\lambda}:= (c_{p\nu} +1)^{\lambda}$.} 
 As we have $g_1(\bbtheta) \geq  \|\bbtheta\|$, the above inequality is lower bounded by $\EE \left(  \sup_{s \in  [0,t] } \|\bbtheta\|)^{\lambda} \right)$ and  we get
\begin{align}
\EE \left(  \sup_{s \in  [0,t] } \|\bbtheta\|)^{\lambda} \right) \leq  C_1\left( 1+ C_2  \left(  \frac{U_R}{(1-\gamma)^2}  (m +b)/2   + C \frac{d}{k_1^{1/\alpha}}   \right)t \right)^{\lambda},
\end{align}
where $C_1$ and $C_2$ are  positive constants depending on $c_{\lambda}$ and $ \EE g_1(\bbtheta_0)$.
\end{proof}
The following lemma upper bounds the error between \eqref{eqn:SGD_L} and its continuous time equivalent SDE \eqref{eqn:SGD_L1} for  $t \in [k\eta, (k+1) \eta]$ and derives its probabilistic interpretation. Note that this result is useful in translating exit time  results for the proposed setting as defined  in  Theorem \ref{Theorem:Exit_time}.  
\begin{lemma}\label{lemma:bound}
Given the proposed heavy-tailed setting of \eqref{eq:policy_gradient_iteration}  is initialized at $\bbtheta_0$ with step-size $\eta: \left(  \exp{(M_J\eta)} \eta (B + M_J) \leq \xi/3   \right), \, \xi>0$,  there exist a set of  positive constants $C_{\alpha}$, $C_1$, and $C_2$  such that the following holds
\begin{align}\label{eqn:bound}
&  \mathcal{P}^{\bbtheta_0}  \left(  \max_{0 \leq k \leq K-1}   \sup_{t \in [k\eta, (k+1) \eta]} \| \theta_t-\bbtheta_{k\eta} \| \geq \xi \right) \nonumber \\
& \leq \exp{(M_J \eta)}  M_J \eta  \frac{   C_1 \left(1 +   C_2 \left(\frac{U_R}{(1-\gamma)^2}  (m +b)/2   + C \frac{d}{k_1^{1/\alpha}}   \right)K \eta \right)^{\beta}}{\xi/3}    \nonumber \\
&\quad+  1- \left(1-C_{\alpha}d^{1+\frac{\alpha}{2}}\eta \exp{(\alpha M_J\eta)}\epsilon^{\alpha}\left(\frac{\xi}{3}\right)^{-\alpha}\right)^{K}.
\end{align}
\end{lemma}
\begin{proof}
In order to upper bound the error between the discrete and continuous equivalents, we start by analyzing the difference in dynamics for   $t \in [k\eta, (k+1)\eta]$ and $t=k\eta$. Using the continuous equivalent  \eqref{eqn:Levygeneric} and integrating the expression over $0$ to $t$ gives  
\begin{align}
& \bbtheta_t = \bbtheta_0 + \int_{0}^t \nabla J (\bbtheta_t^{\epsilon} (u)) du + \epsilon r  L^{\alpha}(t) 
\end{align}
For $t=k\eta$, ie., the beginning of the interval considered, we have 
\begin{align}
& \bbtheta_{k\eta} = \bbtheta_0 + \int_{0}^{k\eta} \nabla J (\bbtheta_t^{\epsilon} (u)) du + \epsilon r  L^{\alpha}(k\eta) 
\end{align}
Using the above two expressions, $\|\bbtheta_t-\bbtheta_{k\eta}\|$ takes the form
%
\begin{align}
& \|\bbtheta_t-\bbtheta_{k\eta} \| \leq \int_{k\eta}^{t}\|  \nabla J (\bbtheta_u) \| du + \epsilon \|r\| \| L^{\alpha}(t) - L^{\alpha}(k\eta)\| 
\end{align} 
Here $\bbtheta_u$ denotes $\bbtheta_t^{\epsilon}(u)$, $r$ is the unit vector defining the direction of perturbation.  We are interested in the dynamics when $t \in [k\eta, (k+1)\eta)$. 
Adding and subtracting $\nabla J (\bbtheta_{k\eta})$ inside the integral of the above expression yields
\begin{align}
& \|\bbtheta_t-\bbtheta_{k\eta} \| \leq \int_{k\eta}^{t} \| (\nabla J (\bbtheta_u)  -\nabla J(\bbtheta_{k\eta}) \| du + \eta \| \nabla J(\bbtheta_{k\eta})\| +  \epsilon \| L^{\alpha}(t) - L^{\alpha}(k\eta)\| 
\end{align} 
Next, Apply Lemma \ref{Lemma:PGL} and the  H\"{o}lder  continuity of the policy gradients inside the integral to write:
\begin{align} \label{eqn:interm5}
& \|\bbtheta_t-\bbtheta_{k\eta} \| \leq \int_{k\eta}^{t} M_J \| \bbtheta_u  -\bbtheta_{k\eta} \|^{\beta} du + \eta \| \nabla J(\bbtheta_{k\eta})\| +  \epsilon \| L^{\alpha}(t) - L^{\alpha}(k\eta)\| 
\end{align} 
By employing H\"{o}lder continuity and Cauchy Schwartz inequality for the second term, we obtain
\begin{align}
 \| \nabla J(\bbtheta_{k\eta}) \| - \|  \nabla J (\bbtheta_{0})\| \leq  \| \nabla J(\bbtheta_{k\eta}) -  \nabla J(\bbtheta_{0}) \| \leq M_J \|\bbtheta_{k\eta}\|^{\beta}
\end{align} 
Using $\|\nabla J(0)\| \leq B$ in the above expression, we get 
\begin{align}\label{eqn:bound}
 \| \nabla J(\bbtheta_{k\eta}) \|  \leq M_J \|\bbtheta_{k\eta}\|^{\beta} +B
\end{align} 
Substituting \eqref{eqn:bound} for the second term of  \eqref{eqn:interm5} results in 
  \begin{align}
& \|\bbtheta_t-\bbtheta_{k\eta} \| \leq \int_{k\eta}^{t} M_J \| \bbtheta_u  -\bbtheta_{k\eta} \|^{\beta} du + \eta\left(   M_J \| \bbtheta_{k\eta}\|^{\beta} +B  \right) +  \epsilon \| L^{\alpha}(t) - L^{\alpha}(k\eta)\| 
\end{align} 
For $\beta<1$, $\| \bbtheta_u  -\bbtheta_{k\eta} \|^{\beta}  < \| \bbtheta_u  -\bbtheta_{k\eta} \|+1$. This fact allows us to rewrite the above expression as 
\begin{align}
\| \theta_t-\bbtheta_{k\eta} \| &\leq \int_{k\eta}^{t} M_J ( \| \bbtheta_u  -\bbtheta_{k\eta} \|) du + \eta\left(   M_J \| \bbtheta_{k\eta}\|^{\beta} +B +M_J \right) +  \epsilon \| L^{\alpha}(t) - L^{\alpha}(k\eta)\| \nonumber \\
&\leq \int_{k\eta}^{t} M_J ( \| \bbtheta_u  -\bbtheta_{k\eta} \|) du + \eta\left(   M_J \| \bbtheta_{k\eta}\|^{\beta} +B +M_J \right) +  \epsilon \sup_{t \in [k\eta, (k+1) \eta]} \| L^{\alpha}(t) - L^{\alpha}(k\eta)\|.\nonumber 
\end{align} 
Now this is in the exact form of  Gronwall's inequality of Definition \ref{Def: Gronwall} with $\phi(\cdot)$ being $\| \theta_t-\bbtheta_{k\eta} \|$ and $B:= \eta\left(   M_J \| \bbtheta_{k\eta}\|^{\beta} +B +M_J \right) +  \epsilon \sup_{t \in [k\eta, (k+1) \eta]} \| L^{\alpha}(t) - L^{\alpha}(k\eta)\|$ and $C(\cdot) : = M_J$, therefore direct application of the inequality yields
\begin{align}
 \sup_{t \in [k\eta, (k+1) \eta]} \| \theta_t-\bbtheta_{k\eta} \| \leq  \exp{(M_J \eta)} \left(  \left( \eta\left(   M_J  \| \bbtheta_{k\eta}\|^{\beta} +B +M_J \right) \right) +  \epsilon \sup_{t \in [k\eta, (k+1) \eta]} \| L^{\alpha}(t) - L^{\alpha}(k\eta)\|\right).
\end{align} 
Evaluating the maximum of the above expression  in the interval  $k \in [0,\, K-1]$
\begin{align} \label{eqn:intrLemma9}
& \max_{0 \leq k \leq K-1}  \sup_{t \in [k\eta, (k+1) \eta]} \| \theta_t-\bbtheta_{k\eta} \| \leq  \exp{(M_J\eta)} \left(  \left( \eta\left(   M_J \max_{0 \leq k \leq K-1}   \| \bbtheta_{k\eta}\|^{\beta} +B +M_J \right) \right) \right. \nonumber \\
& \left.  +  \epsilon  \max_{0 \leq k \leq K-1} \sup_{t \in [k\eta, (k+1) \eta]} \| L^{\alpha}(t) - L^{\alpha}(k\eta)\|\right).
\end{align} 
%
As we are interested in evaluating the probability of $\|\bbtheta_t -\bbtheta_{k\eta} \| \in B^{C}$ (cf. \eqref{eqn:B}),  consider the cases where error between  continuous and discrete process exceeds $\|\xi\|$, ie., $ \max_{0 \leq k \leq K-1}  \sup_{t \in [k\eta, (k+1) \eta]} \| \theta_t-\bbtheta_{k\eta} \| \geq \xi$, 
\begin{align}
 \mathcal{P}^{\bbtheta_0}\left(\max_{0 \leq k \leq K-1}  \sup_{t \in [k\eta, (k+1) \eta]} \| \theta_t-\bbtheta_{k\eta} \|  \geq \xi \right)  .
\end{align}
Now with the assumption that each term on the right hand side of \eqref{eqn:intrLemma9} contributes equally, we have
\begin{align} 
\max_{0 \leq k \leq K-1}  \sup_{t \in [k\eta, (k+1) \eta]} \| \theta_t-\bbtheta_{k\eta} \|  \leq&  \left(   \exp{(M_J\eta)}  \eta  M_J \max_{0 \leq k \leq K-1}   \| \bbtheta_{k\eta}\|^{\beta} \geq \frac{\xi}{3} \right)  \\& + \left(   \exp{(M_J\eta)}  \eta (B +M_J)  \geq  \frac{\xi}{3} \right) \nonumber \\
 & \hspace{-1cm}
+  \left(\exp{(M_J\eta)}  \left(   \epsilon  \max_{0 \leq k \leq K-1} \sup_{t \in [k\eta, (k+1) \eta]} \| L^{\alpha}(t) - L^{\alpha}(k\eta)\|\right) \geq \frac{\xi}{3} \right). \nonumber
\end{align} 
Next we evaluate the probability that the above expression holds  given the process is initialized at $\bbtheta_0$ 
%
\begin{align}
& \mathcal{P}^{\bbtheta_0}\left(\max_{0 \leq k \leq K-1}  \sup_{t \in [k\eta, (k+1) \eta]} \| \theta_t-\bbtheta_{k\eta} \|  \geq \xi \right) 
  \nonumber \\
 & \leq  \mathcal{P}^{\bbtheta_0}\left( \exp{(M_J\eta)}   \eta M_J \max_{0 \leq k \leq K-1}   \| \bbtheta_{k\eta}\|^{\beta}  \geq \xi/3\right) +    \mathcal{P}^{\bbtheta_0}\left( \exp{(M_J \eta)}\eta(B +M_J ) \geq \xi/3 \right)  \nonumber \\
& \quad  +  \epsilon \exp{(M_J \eta)}  \mathcal{P}^{\bbtheta_0}\left( \max_{0 \leq k \leq K-1} \sup_{t \in [k\eta, (k+1) \eta]} \| L^{\alpha}(t) - L^{\alpha}(k\eta)\| \geq \xi/3 \right).  \label{eqn:intrmm}
\end{align} 
%
From   Markov's inequality, we can write
\begin{align}
\mathcal{P}(X  \geq u ) \leq \mathbb{E} \left[ X \right] / u .
\end{align}
%
Using  Markov's inequality for first term on the right hand side of \eqref{eqn:intrmm}  gives 
%
\begin{align} \label{eqn:t2}
& \mathcal{P}^{\bbtheta_0}  \left(  \max_{0 \leq k \leq K-1}   \sup_{t \in [k\eta, (k+1) \eta]} \| \theta_t-\bbtheta_{k\eta} \| \geq \xi \right) \leq  \exp{(M_J \eta)} M_J  \eta  \frac{  \mathbb{E} \left[  \max_{0 \leq k \leq K-1}   \| \bbtheta_{k\eta}\|^{\beta} \right]}{\xi/3} \nonumber \\
& +  \mathcal{P}^{\bbtheta_0} \left(  \exp{(M_J\eta)} \eta (B + M_J) \geq \xi/3  \right)+ \nonumber \\
&  \underbrace{ \mathcal{P}^{\bbtheta_0}  \left( \exp{(M_J\eta)}   \max_{0 \leq k \leq K-1}    \sup_{t \in [k\eta, (k+1) \eta]} \| L^{\alpha}(t) - L^{\alpha}(k\eta)\| \geq   \epsilon^{-1}   \xi/3 \right)}_{T_6}.
\end{align} 
Using the choice of $\eta$ such that
 $\left(  \exp{(M_J\eta)} \eta (B + M_J)   \right)$ is always less than $\xi/3$, second term of the right hand side inequality  equals to zero. Using properties of L\'{e}vy process from Lemma \ref{Lemma:Levynoise} (Lemma 3, \cite{thanh2019first}),   the last  term of the above inequality  simplifies to
\begin{align}
T_6  \leq 1- \left(1-C_{\alpha}d^{1+\frac{\alpha}{2}}\eta \exp{(\alpha M_J \eta)}\epsilon^{\alpha}\left(\frac{\xi}{3}\right)^{-\alpha}\right)^{K}.
\end{align}
Using the above expression in \eqref{eqn:t2}
\begin{align} 
 \mathcal{P}^{\bbtheta_0}  \left(  \max_{0 \leq k \leq K-1}   \sup_{t \in [k\eta, (k+1) \eta]} \| \theta_t-\bbtheta_{k\eta} \| \geq \xi \right) & \leq   \exp{(M_J \eta)}  M_J \eta  \frac{  \mathbb{E} \left[  \max_{0 \leq k \leq K-1}   \| \bbtheta_{k\eta}\|^{\beta} \right]}{\xi/3}   \nonumber \\
& +  1- \left(1-C_{\alpha}d^{1+\frac{\alpha}{2}}\eta \exp{(\alpha M_J \eta)}\epsilon^{\alpha}\left(\frac{\xi}{3}\right)^{-\alpha}\right)^{K}.
\end{align}
Now the expression for $\EE \left(  \max_{0 \leq k \leq K-1}  \|\bbtheta_{k\eta}\|)^{\beta} \right) $ can be obtained by the direct substitution of inequality from Lemma \ref{Lemma:Gurbuzb}.  Therefore,  we get 
\begin{align} 
& \mathcal{P}^{\bbtheta_0}  \left(  \max_{0 \leq k \leq K-1}   \sup_{t \in [k\eta, (k+1) \eta]} \| \theta_t-\bbtheta_{k\eta} \| \geq \xi \right) \nonumber \\
& \leq \exp{(M_J \eta)}  M_J \eta  \frac{   C_1 \left(1 +   C_2 \left(\frac{U_R}{(1-\gamma)^2}  (m +b)/2   + C \frac{d}{k_1^{1/\alpha}}   \right)K \eta \right)^{\beta}}{\xi/3}    \nonumber \\
&\quad+  1- \left(1-C_{\alpha}d^{1+\frac{\alpha}{2}}\eta \exp{(\alpha M_J\eta)}\epsilon^{\alpha}\left(\frac{\xi}{3}\right)^{-\alpha}\right)^{K},
\end{align}
where $ \beta \in (0, 1)$,  $U_R$ is the upper bound on the reward function  from Assumption \ref{Assum:Q0}, constants  $C_1$ and $C_2$ depends on $\beta$, dissipativity constants $(m, \, b)$, and  tail index $\alpha$,  $k_1: = 1/\eta^{\alpha-1}$.
\end{proof}

%% file: Appendix_E.tex




\section{Proof of Theorem \ref{Theorem:Exit_time}: Exit time Analysis for Heavy-tailed Policy Search}
\label{Proof:exittime}
Next, we present results on the first exit time for the proposed heavy-tailed policy gradient setting.
\begin{proof}

We start along the lines of \cite{simsekli2019tail, thanh2019first, tzen2018local}  and relate  the proposed setting of  \eqref{eqn:SGD_L} to the continuous SDE of \eqref{eqn:SGD_L1}.  
We define the set of $K$ points, $k=1 \ldots, K$  obtained  from  \eqref{eqn:SGD_L} which are at a maximum distance of $a$ from the local maxima  of interest, $\bar{\bbtheta}_i$ is defined as 
\begin{align}\label{eqn:aset}
A \triangleq \{  (\bbtheta^1, \ldots, \bbtheta^k): \max_{k \leq K} \| \bbtheta^k -\bar{\bbtheta}_i \| \leq a    \} 
\end{align}
Next we define a set $N_a$ as the neighborhood in Euclidean distance centered at $\bar{\bbtheta}_i$:
\begin{align}\label{eqn:Na}
N_a \triangleq \{ \bbtheta \in \mathbb{R}^d: \| \bbtheta - \bar{\bbtheta}_i \| \leq a \}.
\end{align}
As we initialized both discrete and continuous process at $\bbtheta_0$, processes defined by  \eqref{eqn:SGD_L} and \eqref{eqn:SGD_L1} are considered close enough if both of their   exit times from $N_a$ are close. 
For the moment, we assume  step-size $\eta_k =\eta$ is constant,  and consider the linearly  interpolated version of \eqref{eqn:SGD_L} given by 
\begin{align}\label{eqn:SGDint}
d \hat{\bbtheta}_t = b(\hat{\bbtheta}_t) dt  + \epsilon d \bbL_{\alpha}^t, 
\end{align}
Note that the unit vector direction of perturbation is absorbed into $d \bbL_{\alpha}^t$ (cf. \eqref{eqn:Levygeneric}).  Here, $\hat{\bbtheta} = \{\hat{\bbtheta}_t\}_{t \geq 0}$ denotes the whole process with drift term, $b$ defined as 
\begin{align}\label{eq:interpolated_gradient}
b(\hat{\bbtheta}_t) \triangleq  \sum_{k=0}^{\infty} \nabla J(\hat{\bbtheta}_{k \eta})  \mathbb{I}_{[k\eta , (k+1) \eta )} (t). 
\end{align}
{where $\mathbb{I}$ denotes the indicator function such that $\mathbb{I} = 1 $ if $t \in [k\eta, \, (k+1) \eta)$.  The above expression for the drift (deterministic) component of the continuous-time process can also be expressed as 
 $$  b(\hat{\bbtheta}_t) \triangleq  \frac{1}{1-\gamma} \sum_{k=0}^{\infty} \EE_{s, a \sim \rho_{\hat{\bbtheta}_{k\eta}} } [\nabla \log \pi_{\hat{\bbtheta}_{k\eta}} Q_{\pi_{\hat{\bbtheta}_{k\eta}}}(s,a)] \mathbb{I}_{[k\eta , (k+1) \eta )} (t).
 $$ 
}
First exit time of a continuous SDE \eqref{eqn:SGD_L1}   is summarized in Theorem \ref{th:21} \cite{imkeller2010first}.   To invoke this result, we first more rigorously establish the connection between  the exit times of the discrete [cf. \eqref{eqn:SGD_L}] and continuous-time [cf. \eqref{eqn:SGD_L1}] processes. Then, we may invoke this result to obtain the desired statement associated with the  heavy-tailed policy search scheme in discrete time given in \eqref{eqn:SGD_L}.

%
For the time being, we assume the distance (divergence) between the underlying distributions of processes  \eqref{eqn:SGD_L} and \eqref{eqn:SGD_L1}  sampled at discrete instants, $\eta,\, k\eta,\, \ldots..,\,K\eta$ is  bounded by $\delta$.  Since \eqref{eqn:SGDint} is the interpolated version of \eqref{eqn:SGD_L}, this sampling specification  also implies that  there exists an optimal transport plan or optimal coupling between $\{\bbtheta_s\}_{s \in [0,\ldots, K] }$ and $\{\hat{\bbtheta}_s\}_{s \in [0,\ldots, K]}$. Therefore, using optimal coupling argument, data processing inequality for the relative entropy,   and Pinsker's inequality, there exists a coupling between the random variables of both the processes such that coupling $M$ between  $\{\bbtheta_s\}_{s \in [0, K\eta]}$ and $\{ \hat{\bbtheta}_s \}_{s \in [0, K\eta]}$ satisfies \cite{tzen2018local}
\begin{align}
M (\{\bbtheta_s\}_{s \in [0, K\eta]} \neq \{ \hat{\bbtheta}_s \}_{s \in [0, K\eta]} )  \leq \delta. 
\end{align}
Now, probability that the interpolated process and the continuous-time process are not equivalent is upper bounded by   $\delta$, i.e.,
\begin{align} \label{eqn:thmintrm2}
\mathcal{P}\left( \{ \bbtheta_s \}_{s \in [0, K\eta] } \neq   \{ \hat{\bbtheta}_s\}_{s \in [0, K \eta]}  \right) \leq \delta 
\end{align}
Typically, this upper-bound depends on the choice of algorithm step-size, $\eta$ (See Assumption 6, \cite{simsekli2019tail}), and depends on the KL-divergence between the underlying  distributions.

Next we relate the expressions for the exit time of  continuous and discrete-time processes, \eqref{eqn:SGD_L} and \eqref{eqn:SGD_L1}, respectively.    Let $\{ \hat{\bbtheta}_s \in A, \, \textrm{for} \, k=1,\ldots,\, K\}$. 
 As $\hat{\bbtheta}_s$ is the linearly interpolated version of the discrete process defined by \eqref{eqn:SGD_L1} we get  $\bbtheta^1, \ldots, \bbtheta^K \in A$.  The statement   $\bbtheta^1, \ldots, \bbtheta^K \in A$ implies  exit time $\bar{\tau}$ of the discrete process from   a domain characterized by $a$ is greater than $K$, $ \mathcal{P}(\bar{\tau}_{0,a}(\epsilon) > K)$. Therefore, using \eqref{eqn:thmintrm2}, the following events can happen with a  probability dependent on $\delta$ if we have $ \bar{\tau}_{0,a}(\epsilon) > K$
\[
\bbtheta_{t} \begin{cases}
 \in A & \textrm{for } \, t = \eta, \ldots, K \eta \\
  \notin A:  & \textrm{with\, probability } \, <  \delta 
\end{cases}
\]
%
This fact allows us to write the probability of $\bar{\tau}_{0,a}(\epsilon) > K$ for some constant $K$ as follows: 
\begin{align}\label{thm:intrm1}
\mathcal{P}^{\bbtheta_0}\left(  \bar{\tau}_{0,a} (\epsilon) >K) \leq \mathcal{P}^{\bbtheta_0}\left( (\bbtheta_{\eta}, \ldots, \bbtheta_{K\eta} ) \in A    \right)   \right) + \delta
\end{align}
%
Note that  $\left( (\bbtheta_{\eta}, \ldots, \bbtheta_{K\eta} ) \in A    \right) $ implies maximum distance from local minimum is $a$. 

Now we have an  expression connecting the probability for continuous time processes, $\bbtheta_t$  at $t=k \eta$, $k=1,\ldots,K$ being  inside  $A$ to  the exit time of  discrete equivalents  given both of them are initialized with same value.  However,  the identity of   continuous  random variables $\bbtheta_t$ between $ [k\eta, (k+1)\eta]$ are still unknown. To understand their behavior within this range, we study their behavior to sets $N_a$ [cf. \eqref{eqn:Na}].   Once we augment that into  \eqref{thm:intrm1},  we are in the position to derive for exit time results for discrete process in terms of its continuous equivalent. These steps are formalized next.
\paragraph{Upper bound on $\mathcal{P}^{\bbtheta_0}\left( (\bbtheta_{\eta}, \ldots, \bbtheta_{K\eta} ) \in A    \right) $: }

In order   to analyze  $\bbtheta_t$ for $t \in [k\eta, (k+1)\eta]$ and to link it with $N_a$, we follow the approach used in \cite{thanh2019first} and  impose an upper bound on   $\|\bbtheta_{t}-\bbtheta_{k\eta}\|$, $\forall t \in [k\eta, (k+1)\eta]$.  Let us assume   $\|\bbtheta_t-\bbtheta_{k\eta} \| $ is bounded by $\xi$ for the time between $k\eta$ and $(k+1)\eta$. We define a set B  such that
\begin{align} \label{eqn:B}
B \triangleq \{   \max_{0 \leq k \leq K-1}  \sup_{t \in [k\eta, (k+1)\eta]} \| \bbtheta_t-\bbtheta_{k\eta} \| \leq \xi  \} 
\end{align} 
For a clear understanding of the process, we illustrate the  previously defined  sets in Fig. \ref{fig:illust}
\begin{figure}[h]
	{\includegraphics[scale=0.4]{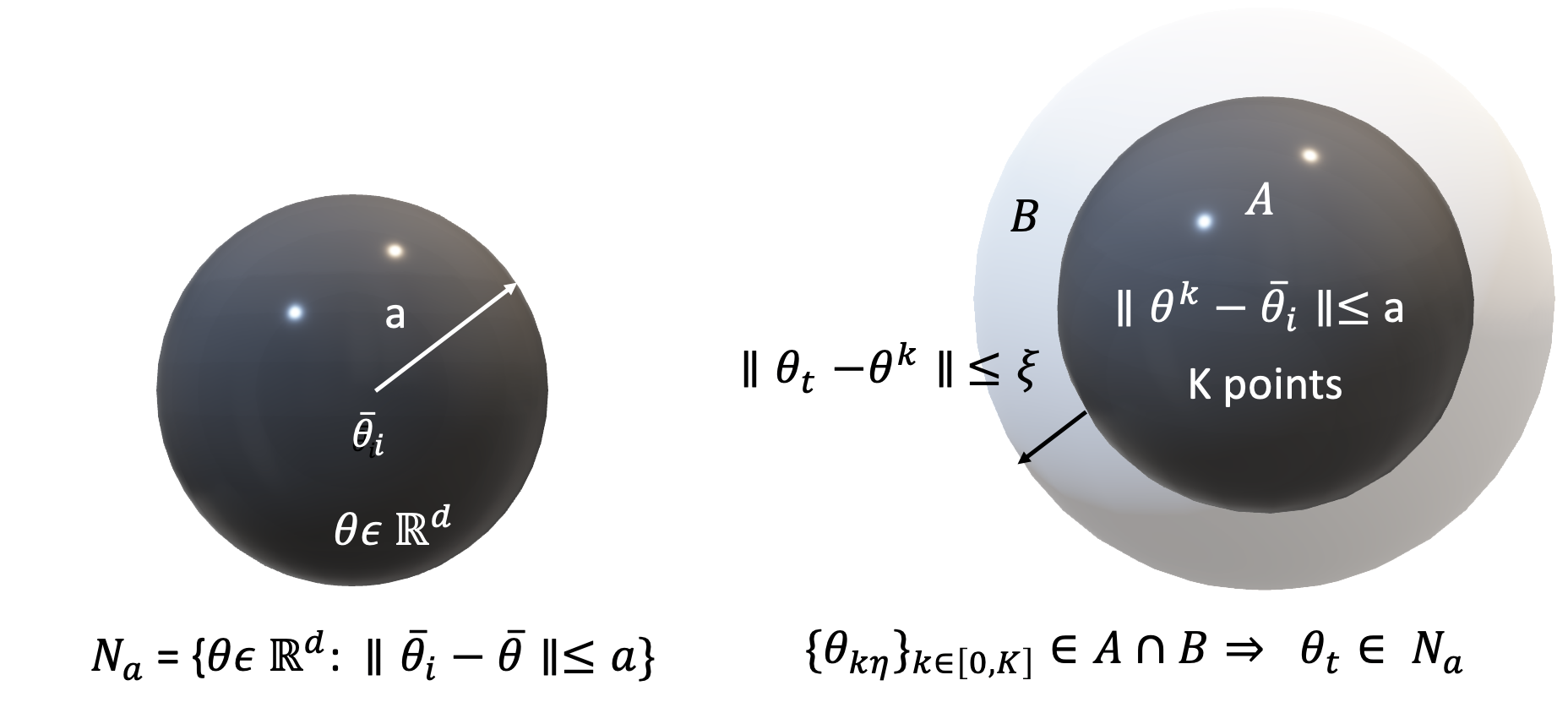}}
	\vspace{0mm}   
	{{\centering  
			\caption{Sketch of intuition for the analysis of the exit time behavior of the discrete time process. We study the behavior of discrete and continuous process initialized at same point.  On the left we depict $N_a$, the neighborhood of the local extrema for a generic element $\bbtheta \in \mathbb{R}^d$ with respect to which we study  the  exit time behavior.  In order to do that we first  analyze the behavior of the continuous-time process at discrete time indices which is piecewise constant within intervals $t=\eta,\cdots,K\eta$ within a neighborhood of local maximum $\bar{\bbtheta}_i$ defined in Assumption \ref{assmp:domaing}(3.). This interpolation process bridges the gap between the continuous-time process and the discrete time process. On the right we depict an equivalent neighborhood for the discrete-time process. We show that the probabilistic behavior of the continuous and discrete neighborhoods are comparable up to constant factors that depend on algorithm step-size and interpolation parameters  (cf. \eqref{thm:intrm1}). Characterizing this probabilistic behavior does not bound the behavior of the continuous process between the  instants $t \in [k\eta, (k+1)\eta]$. To do so, we introduce event  $B$ and study the behavior of continuous process between  $ [k\eta, (k+1)\eta]$  (Lemma \ref{lemma:bound}).  Now relating events $A \cap B$ and $B^c$, we obtain an upper bound on the exit time of discrete process. } \label{fig:illust}}}\vspace{-2mm} 
\end{figure} 

 Let $A$ defines  a  hyper sphere with    $K$ points $\bbtheta^k \in \mathbb{R}^d$  from discrete process \eqref{eqn:SGD_L}. The radius of $A$ is defined by $a$ and $\bar{\bbtheta}_i$ be its center. Now  we have another hyper sphere (shown in gray), $B$ of radius $a+\xi$ such that $B$ defines  the event $\bbtheta_t, \, t \in [k\eta, (k+1)\eta]$ such that  the  maximum error between $\|\bbtheta_t -\bbtheta^{k}\| \leq \xi$.    Now, the event $[(\bbtheta_{\eta}, \ldots, \bbtheta_{K\eta} ) \in A ] \cap B    )$ ensures that $\bbtheta_t$ is close to $N_a$ for $t= \eta,\ldots, K\eta$.  Now let us relate  $\mathcal{P}^{\bbtheta_0}\left( (\bbtheta_{\eta}, \ldots, \bbtheta_{K\eta} ) \in A    \right) $ from \eqref{eqn:thmintrm2} to  exit time of the continuous process.   
 
 For $ ( \bbtheta_{\eta}, \ldots, \bbtheta_{K\eta}) \in A$,  we can have two possibilities: $\|\bbtheta_t-\bbtheta_{k\eta}\|, \, t \in [k\eta, (k+1) \eta] $ can be either inside or outside $B$. Therefore, 
\begin{align}
\mathcal{P}^{\bbtheta_0} ( ( \bbtheta_{\eta}, \ldots, \bbtheta_{K\eta}) \in A) \leq \mathcal{P}^{\bbtheta_0} ( ( \bbtheta_{\eta}, \ldots, \bbtheta_{K\eta}) \in A \cap B ) + \mathcal{P}^{\bbtheta_0}  ( ( \bbtheta_{\eta}, \ldots, \bbtheta_{K\eta}) \in B^c )
\end{align}
Note that  $ \mathcal{P}^{\bbtheta_0} ( ( \bbtheta_{\eta}, \ldots, \bbtheta_{K\eta}) \in A \cap B ) 
\implies \bbtheta_t \in N_a$, implies exit time of the corresponding continuous process is greater than $K\eta$.  Therefore, first term on the right hand side of the above expression is equivalent to  probability that the exit time of the continuous process  from a domain characterized by $a+\xi$ is greater than $K\eta$. And we get, 
\begin{align}
\mathcal{P}^{\bbtheta_0} ( ( \bbtheta_{\eta}, \ldots, \bbtheta_{K\eta}) \in A) \leq  \mathcal{P}^{\bbtheta_0}(\tau_{\xi,a}(\epsilon) \geq K \eta)  + \mathcal{P}^{\bbtheta_0}  ( ( \bbtheta_{\eta}, \ldots, \bbtheta_{K\eta}) \in B^c )
\end{align}
Using the above expression in \eqref{thm:intrm1}
\begin{align}
& \mathcal{P}^{\bbtheta_0} \left(  \bar{\tau}_{0,a} (\epsilon) >K\right) 
\leq \mathcal{P}^{\bbtheta_0}(\tau_{\xi,a}(\epsilon) \geq K \eta) + \mathcal{P}^{\bbtheta_0} ( ( \bbtheta_{\eta}, \ldots, \bbtheta_{K\eta}) \in B^c ) + \delta 
\end{align}
 Now the second term on the right hand side of the expression  defines the probability that $\{ \bbtheta_{k\eta}\}_{k=1,\ldots,K}$ are not confined in the set $B$, $\mathcal{P}^{\bbtheta_0} ( ( \bbtheta_{\eta}, \ldots, \bbtheta_{K\eta}) \in B^c ) $.  Using upper bound on $\mathcal{P}^{\bbtheta_0} ( ( \bbtheta_{\eta}, \ldots, \bbtheta_{K\eta}) \in B^c ) $  from  Lemma \ref{lemma:bound}  in the above inequality,  we get
\begin{align}\label{eq:proof_key_step}
 \mathcal{P}^{\bbtheta_0} \left(  \bar{\tau}_{0,a} (\epsilon) >K\right) 
 \leq &\mathcal{P}(\hat{\tau}_{\xi,a}(\epsilon) \geq K \eta) \nonumber \\
 & + \exp{(M_J \eta)}  M_J \eta  \frac{   C_1  \left( 1 + C_2 \left(\frac{U_R}{(1-\gamma)^2}  (m +b)/2   + C \frac{d}{k_1^{1/\alpha}}   \right)K\eta \right)^{\beta}}{\xi/3}     \nonumber \\
 & +  1- \left(1-C_{\alpha}d^{1+\frac{\alpha}{2}}\eta \exp{(\alpha M\eta)}\epsilon^{\alpha}\left(\frac{\xi}{3}\right)^{-\alpha}\right)^{K} +\delta  .
\end{align}
%
%
%
From here, we invoke the results for exit time in multi-dimensional space \cite{imkeller2010first} formalized in Theorem \ref{th:21}. To do so, let $\bbtheta_0 \in \mathcal{G}_i$ ($\mathcal{G}_i$ be a domain containing the $i$-th  local minima $\bar{\bbtheta}_i$ such that $\mathcal{G}_i \subset \mathbb{R}^d$) and  the process escapes to the delta-tube, $\Omega^+_i(\bar{\delta})$  (cf. \eqref{eqn:omegatube})  from  $\mathcal{G}_i$ using jumps  initiated by  L\'{e}vy process of tail index, $\alpha$ (cf. \eqref{eqn:omegatube}).  

Therefore, the  first term of the inequality \eqref{eq:proof_key_step} corresponds to the time at which $\bbtheta$ exits $\mathcal{G}_i$, i.e., the time at which  $\bbtheta \in \Omega^{+}_i(\bar{\delta})$. Note that the probability defined in the  above expression is in terms of the time at which the process exits a given domain, i.e. the probability that the exit time is greater than $K\eta$. Using Assumption \ref{assmp:domaing},   there exists a $K$ such that $\hat{\tau}_{\xi,a}$ greater than $K \eta,\, K>0$ and $\bbtheta_{\hat{\tau}}^{\epsilon}  \in \Omega^+_i$ for $\hat{\tau}_{\xi,a}(\epsilon) \geq K \eta$ (meaning, probability that the exit time is greater than $K\eta$  is proportional to the probability of  the process entering the desired delta tube), specifically, that $\mathcal{P}(\hat{\tau}_{\xi,a}(\epsilon) \geq K \eta) \propto \mathcal{P}^{\bbtheta_0} \left( \bbtheta_{\hat{\tau}}^{\epsilon}  \in \Omega^+_i (\bar{\delta})  \right)$.

 Therefore, substituting $\mathcal{P}(\hat{\tau}_{\xi,a}(\epsilon) \geq K \eta) $ with $  \mathcal{P}^{\bbtheta_0} \left( \bbtheta_{\hat{\tau}}^{\epsilon}  \in \Omega^{+}_i (\bar{\delta})  \right)$ in \eqref{eq:proof_key_step} allows us to write
 \begin{align}
 \mathcal{P}^{\bbtheta_0} \left(  \bar{\tau}_{0,a} (\epsilon) >K\right) 
 \leq &  \mathcal{P}^{\bbtheta_0} \left( \bbtheta_{\hat{\tau}}^{\epsilon}  \in \Omega^{+}_i (\bar{\delta})  \right)  \nonumber \\
 & + \exp{(M_J \eta)}  M_J \eta  \frac{   C_1  \left( 1 + C_2 \left(\frac{U_R}{(1-\gamma)^2}  (m +b)/2   + C \frac{d}{k_1^{1/\alpha}}   \right)K\eta \right)^{\beta}}{\xi/3}     \nonumber \\
 &  +  1- \left(1-C_{\alpha}d^{1+\frac{\alpha}{2}}\eta \exp{(\alpha M\eta)}\epsilon^{\alpha}\left(\frac{\xi}{3}\right)^{-\alpha}\right)^{K} +\delta 
 \end{align}
 where, $\mathcal{P}^{\bbtheta_0} \left( \bbtheta_{\hat{\tau}}^{\epsilon}  \in \Omega^{+}_i \right) $ indicates the first exit time for a continuous SDE from $\mathcal{G}_i$.  Note that $ \Omega^+_i(\cdot)$ indicates delta-tube outside the boundary of $\mathcal{G}_i$  with perturbations along the positive basis vector. Now direct application of  Theorem \ref{th:21} for the first term on the right-hand-side of the inequality, we get
 \begin{align}
 \mathcal{P}^{\bbtheta_0} \left(  \bar{\tau}_{0,a} (\epsilon) >K\right) 
  &\leq  \frac{2}{\epsilon^{\rho \alpha}} \epsilon^{\alpha} (d^+)^{-\alpha}   \nonumber \\
  &\quad +  \exp{(M_J \eta)}  M_J \eta  \frac{   C_1  \left( 1 + C_2 \left(\frac{U_R}{(1-\gamma)^2}  (m +b)/2   + C \frac{d}{k_1^{1/\alpha}}   \right)K\eta \right)^{\beta}}{\xi/3}    \nonumber \\
  & \quad+  1- \left(1-C_{\alpha}d^{1+\frac{\alpha}{2}}\eta \exp{(\alpha M_J\eta)}\epsilon^{\alpha}\left(\frac{\xi}{3}\right)^{-\alpha}\right)^{K} +\delta   \\
 &\leq  \frac{2}{\epsilon^{\rho \alpha}} \epsilon^{\alpha} (d^+)^{-\alpha}     + \mathcal{O} \left( \frac{d}{k_1^{1/\alpha}} K \eta\right)^{\beta}  \nonumber \\
 & \quad + \mathcal{O} \left( 1- \left(1-C_{\alpha}d^{1+\frac{\alpha}{2}}\eta \exp{(\alpha M_J\eta)}\epsilon^{\alpha}\left(\frac{\xi}{3}\right)^{-\alpha}\right)^{K} +\delta  \right)
\end{align}
%
%
where, $d^+$  denotes distance function to the boundary along  the positive $r$ unit vector (cf.  \eqref{eqn:distm}),  $\rho$ is a positive constant such that $\rho \in (0,1)$,   $a+\xi  > \epsilon^{1-\rho}$,  $d$ is the dimension of $\bbtheta$, H\"{o}lder continuity constant, $\beta \in (0,1)$, $\delta >0$, $C_{\alpha}>0$, $\xi>0$, $\eta$ is the step-size, $U_R$, $\gamma$ are the parameters of proposed RL setting, positive constants, $C, \, C_1\, C_2$,   are functions of dissipativity constants of score function.  Note that first term of the above inequality denotes the probability that the  stochastic process exits from a given domain $\mathcal{G}_i$ when perturbed by a jump process along unit vector $\bbr$ with tail index $\alpha$. Observe that the exit time is only a function of distance between the point at which the process is initialized and the boundary of the domain around local minima (extrema),  but does not depend on  the function values (height of the local minima (extrema)). 
\end{proof}

%% file: Appendix_F.tex


\section{Proof for Theorem \ref{theorem: transitiontym}: Transition time for the Proposed Heavy-tailed setting} \label{sec:proofthm4}

In this section, we derive transition time results for the proposed heavy-tailed setting from one local maxima to another as defined in Theorem  \ref{Theorem:Exit_time}.   Suppose  the domain $\mathcal{G}_i$ satisfies the Assumption \ref{assmp:domaing}  defined in Section \ref{subsec:exit_time}.  Further there exists a unit vector, $+ \bbr$ in the direction connecting the domains $\mathcal{G}_i$ and $\mathcal{G}_{i+1}$ and we define  the distance function to the the intersection of boundaries $\mathcal{G}_{i}$ and $\mathcal{G}_{i+1} $ by:
\begin{align}\label{eqn:dista}
& d_{ij}^+(\bbtheta):= \inf \{ t >0: g_{{i}_{\bbtheta}}(t) \in \partial \mathcal{G}_i  \cap \partial \mathcal{G}_{i+1}\},  \\
&  d_{ij}^-(\bbtheta):= \sup \{ t <0: g_{{i}_{\bbtheta}}(t) \in \partial \mathcal{G}_i  \cap  \partial \mathcal{G}_{i-1}\}.
\end{align}

Let $\bbtheta_0 \in \mathcal{G}_{i}$.  As $\epsilon$ is a coefficient that multiplies the noise process in  \eqref{eqn:SGD_L1}, it is clear that the error between both the processes are lower bounded by $\epsilon$, ie. $  \left( \| \bbtheta^{k} - \bbtheta_{k\eta} \| > \epsilon    \right)$.  Using Markov's inequality, the above statement can be expressed in the language of probability as  
\begin{align} 
\mathcal{P}^{\bbtheta_0} \left( \| \bbtheta^{k} - \bbtheta_{k\eta} \| > \epsilon    \right) \leq \frac{\EE_{\bbtheta_0}\left[  \| \bbtheta^{k} - \bbtheta_{k\eta} \| \right] }{\epsilon} \frac{}{}   \leq \frac{\delta}{2}. 
\end{align} 
Let us say the probability of above event given $\bbtheta^0 = \bbtheta_0$ is  $\delta/2$. 
Given the continuous time instant $t=k\eta$, there exists a unit vector, $\bbr$ in the direction connecting the domains $\mathcal{G}_i$ and $\mathcal{G}_{i+1}$.   
 and let $\bbtheta_{k\eta} \in \Omega_i^+ (\bar{\delta}) \cap \partial \mathcal{G}_{i+1}, \, \forall \bar{\delta} \in (0,\bar{\delta}_0)$,  the following events can happen with probabilities dependent on $\delta$:
\[
\bbtheta^k \begin{cases}
 \in\Omega_i^+ (\bar{\delta}) \cap \partial \mathcal{G}_{i+1}  & \textrm{with} \, \mathcal{P} ( \bbtheta^{k} \in \Omega_i^+ (\bar{\delta}) \cap \partial \mathcal{G}_{i+1}) \\
  \notin \Omega_i^+ (\bar{\delta}) \cap \partial \mathcal{G}_{i+1}:  \|\bbtheta^k -\bbtheta_{k\eta}\| >\epsilon & \textrm{with\, probability } \, <  \delta/2 \\
   \notin \Omega_i^+ (\bar{\delta}) \cap \partial \mathcal{G}_{i+1}:  \max \|\bbtheta^k -N_a \| < \epsilon & \textrm{with\, probability } \, < \delta/2 \\
\end{cases}
\]
The above case-by-case study illuminates that either $\bbtheta^{k} \in \Omega_i^+ (\bar{\delta}) \cap \partial \mathcal{G}_{i+1}$ or $\bbtheta^k \notin \Omega_i^+ (\bar{\delta}) \cap \partial \mathcal{G}_{i+1}$ (for which the probability adds up to $\delta$). Note that $\Omega_i^+ (\bar{\delta}) $ denotes $\bar{\delta}$- tubes outside of $\mathcal{G}_i$ (cf. \eqref{eqn:omegatube}). This fact allows us to write the probability that $\bbtheta_{k\eta} \in \Omega_i^+ (\bar{\delta}) \cap \partial \mathcal{G}_{i+1}$ as follows: 
\begin{align}
\mathcal{P}^{\bbtheta_0}( \bbtheta_{k\eta} \in \Omega_i^+ (\bar{\delta}) \cap \partial \mathcal{G}_{i+1}) \leq \mathcal{P}^{\bbtheta_0}( \bbtheta^{k} \in \Omega_i^+ (\bar{\delta}) \cap \partial \mathcal{G}_{i+1}) +\delta
\end{align}
 Evaluating the  above inequality  in the limit of $\epsilon\rightarrow 0$ yields 
 \begin{align} \label{thm4:intrm1}
\lim_{\epsilon \to 0} \mathcal{P}^{\bbtheta_0}( \bbtheta_{k\eta} \in \Omega_i^+ (\bar{\delta}) \cap \partial \mathcal{G}_{i+1}) \leq \lim_{\epsilon \to 0}  \mathcal{P}^{\bbtheta_0}( \bbtheta^{k} \in \Omega_i^+ (\bar{\delta}) \cap \partial \mathcal{G}_{i+1}) +\delta
\end{align}

Observe that the probability that $\bbtheta^k \notin \Omega_i^+ (\bar{\delta}) \cap \partial \mathcal{G}_{i+1}$ is ${\delta}$.  Note that the value of ${\delta}$ can be obtained analyzing KL-divergence between underlying distirbutions of  \eqref{eqn:SGD_L} and \eqref{eqn:SGDint}, the interpolated version of \eqref{eqn:SGD_L}. The consequence of the aforementioned analysis is an upper bound on the step-size, $\eta$ (See Assumption 6, \cite{simsekli2019tail}). Now the expression on the left-hand side of the above inequality is the probability that the process reaches to the intersection of  $\bar{\delta}$ tube  and $\mathcal{G}_{i+1}$ in the limit $\epsilon \to 0$.

Using Theorem \ref{th:21},  for every $\bar{\delta} \in  (0,a +\xi )$  the probabilities to exit in direction $\pm r$ are given by
\begin{align}
& \lim_{\epsilon \to 0} \mathcal{P} \left(  \Omega_i^+ (\bar{\delta}) \cap \partial \mathcal{G}_{i+1}  \right) =\frac{p_i^+}{p_s} 
\end{align}
 for all $\bbtheta_t \in \mathcal{G}_i$ 
and 
\begin{align}
& p_s:=  ((d_{ij}^+)^{-\alpha} + (-d_{ij}^-)^{-\alpha})\\
& p_i^+ :=  (d_{ij}^+)^{-\alpha} \\
& p_i^- :=(-d_{ij}^-)^{-\alpha},
\end{align}
where  $d_{ij}^+$ and $d_{ij}^-$ define distance from $\bbtheta_0$ to $ \Omega_i^+ (\bar{\delta}) \cap \partial \mathcal{G}_{i+1}$ along $\bbr$. Now using the above expression in \eqref{thm4:intrm1} yields lower bound on the transition probability from $\mathcal{G}_i$ to $\mathcal{G}_{i+1}$, $\mathcal{P}^{\bbtheta_0}( \bbtheta^{k} \in \Omega_i^+ (\bar{\delta}) \cap \partial \mathcal{G}_{i+1}) $ in the limit $\epsilon \to 0$
 \begin{align} \label{thm4:intrm2}
\lim_{\epsilon \to 0} \mathcal{P}^{\bbtheta_0}( \bbtheta^{k} \in \Omega_i^+ (\bar{\delta}) \cap \partial \mathcal{G}_{i+1}) \geq \frac{d_{ij}^{-\alpha}}{((d_{ij}^+)^{-\alpha} + (-d_{ij}^-)^{-\alpha}}  - \delta.
\end{align}

\section{Supplemental Experiments}\label{experiment_Details}

\subsection{Additional Experimental Instantiations}\label{subsec:experiments_supplement}
In this section, we   evaluate the performance of the proposed algorithm on more complex environments   from OpenAI gym and Roboschool, namely Inverted Pendulum from Roboschool and sparse versions of Roboschool REACHER-V2 \&
 HALFCHEETAH-V2. 
  \subsubsection{Inverted Pendulum}

Next, we evaluate the performance of HPG on Inverted Pendulum-v1 from Roboschool in Fig. \ref{fig:J4}.  The objective is to keep the pendulum balanced while keeping the cart upon which it is mounted away from the borders. Environment consists of a observation space of dimension $5$ and an action space of dimension $1$. Here we use a neural network with single hidden layer consisting of $64$ neurons. All the other parameters are same as given in the previous experiments.

\begin{figure*}[t!]
	\centering
	\subfigure[Inverted Pendulum]
	{\includegraphics[width=.33\columnwidth, height=3.5cm]{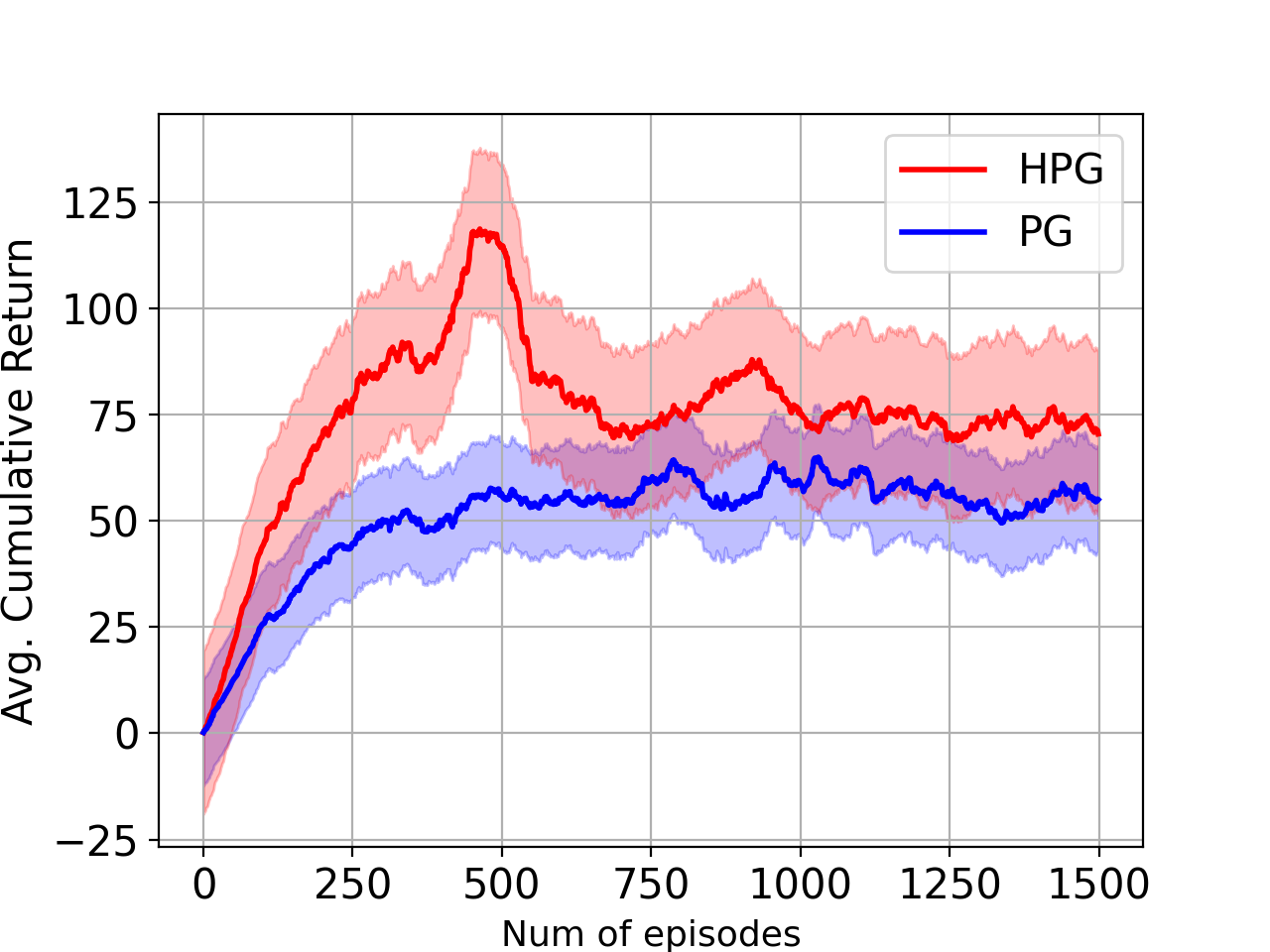} \label{fig:J4}}\hspace{-3mm}
		\subfigure[Sparse Half Chettah]
	{\includegraphics[width=.33\columnwidth, height=3.5cm]{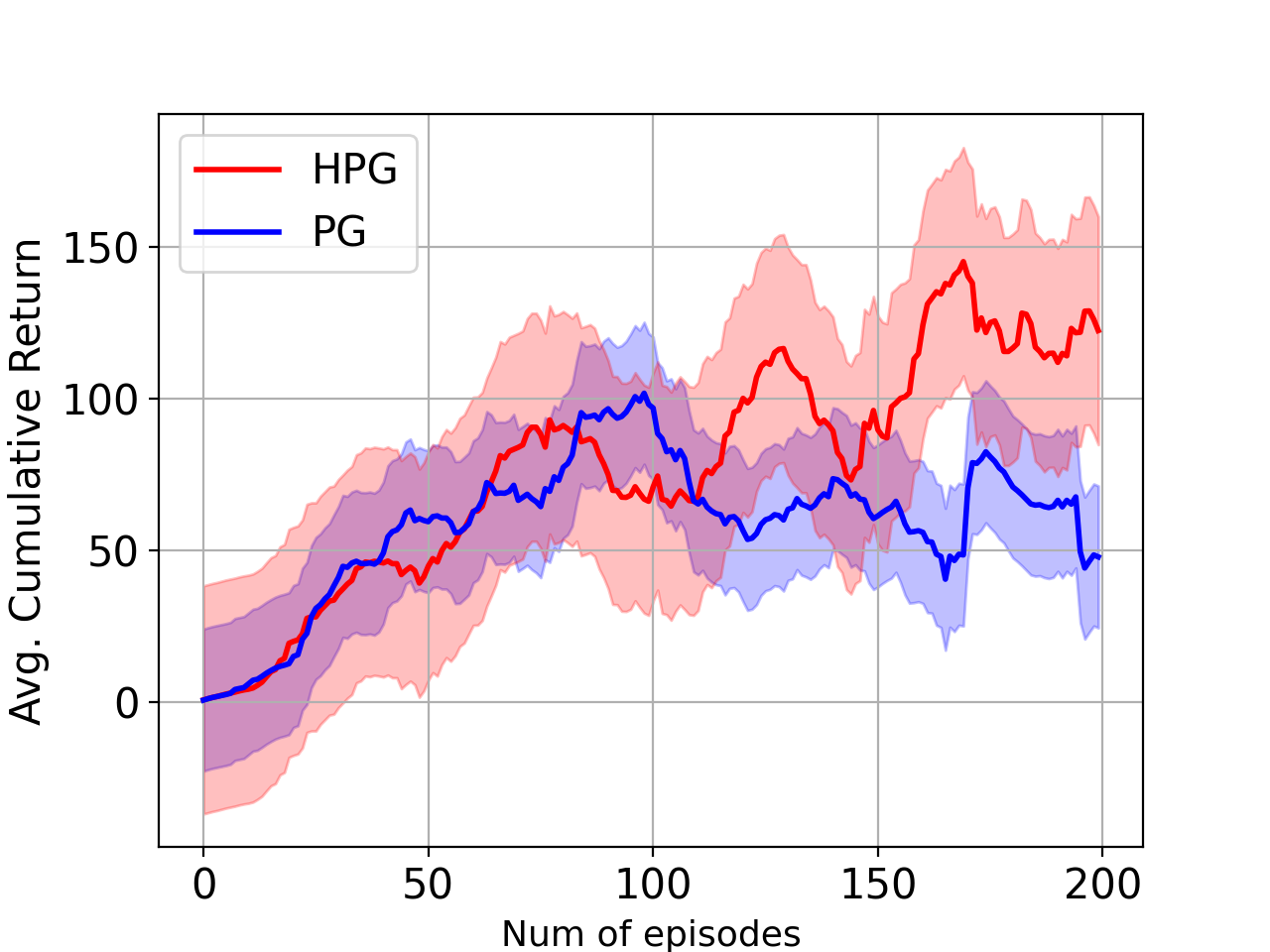} \label{fig:J5}}\hspace{-3mm}
	\subfigure[Sparse Reacher]
	{\includegraphics[width=.33\columnwidth, height=3.5cm]{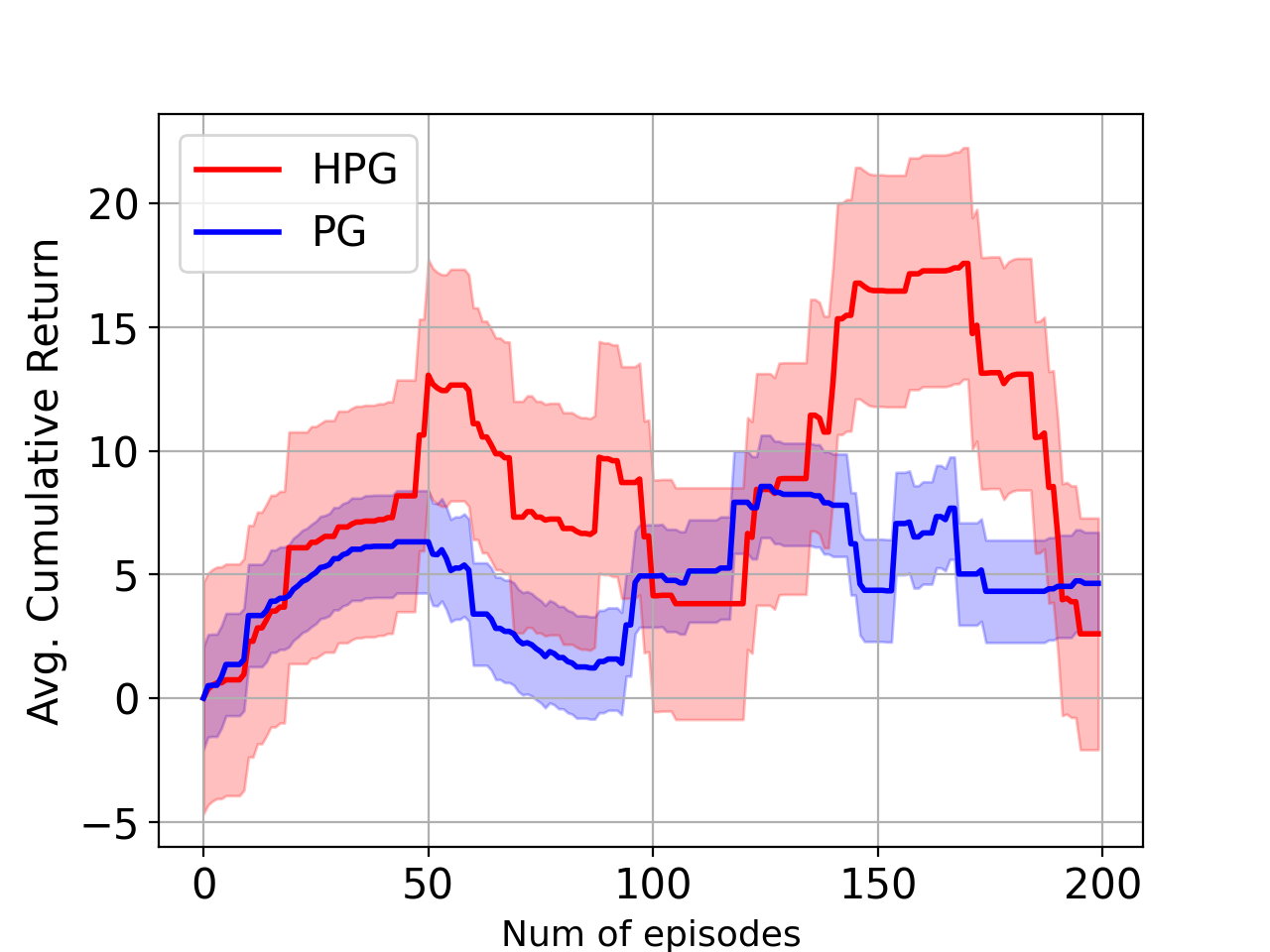}  \label{fig:J6}}
	\caption{ {\bf (a)} Average cumulative returns for Inverted Pendulum from Roboschool over latest $100$ episodes. HPG shows better ability to explore over light tailed policies. {\bf (b)} We plot the  initial average cumulative returns for sparse Chettah  over latest $25$ episodes. Observe that HPG better explores the environment and shows improved performance over PG. {\bf (c)}  We plot the initial average cumulative returns for sparse Reacher over latest $25$ episodes for HPG which shows  better ability to explore  over PG. For clarity of only distinguishing the performance differences associated with policy parameterization, we omit the additional comparators of Figure \ref{comparisons2}.
	}
	\label{comparisons6}
\end{figure*} 

\subsubsection{Sparse Variant of HalfCheetah}
The  environment (shown in Fig. \ref{cheetah}) consists  of planar biped robot in a  continuous environment with $\mathcal{S} \in \mathbb{R}^{26}$ and  $\mathcal{A} \in \mathbb{R}^6$.  \begin{figure}
	\centering
		\subfigure[HalfCheetah environment.]
	{\includegraphics[scale=0.4]{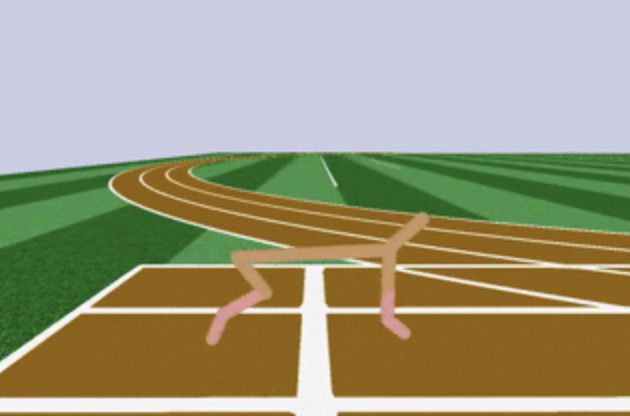}\label{cheetah}} \hspace{5mm}
			\subfigure[Reacher environment.]
	{\includegraphics[scale=0.4]{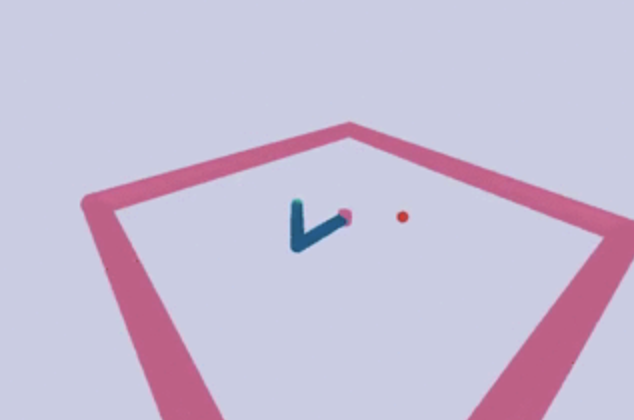} \label{racher}}	
	\caption{Additional environments}
\end{figure}   The task runs on the Roboschool simulator. In the task, the agent is a two-dimensional cheetah controlled by seven actuators.   
The state representation corresponds to the position and velocity of the cheetah, current angle and angular velocity of its joints as well as the torque on the actuators. The reward function in this environment is a mixture between the amount of contact the limbs have on the floor and the displacement of the agent from the starting position. The reward structure is made sparse by removing the control penalty to each actions. In addition, a step reward of 2 is provided when the x-component of  speed of the Cheetah is greater than 2 as in  \cite{matheron2019problem}. We plot the  initial average cumulative returns for sparse Chettah  over latest $25$ episodes in Fig. \ref{fig:J5}. Observe that HPG, by virtue of its extreme action selection according to a heavy-tailed distribution, causes more jumps in the state space, which results in better exploration of the environment. The result is faster policy improvement than vanilla PG.

 \subsubsection{Sparse Variant of Roboschool Reacher}
The  environment (shown in Fig. \ref{racher}) consists  of an robot arm in a 2D space which tries to  reach a  target.   Reacher is a continuous environment with $\mathcal{S} \in \mathbb{R}^9$ and  $\mathcal{A} \in \mathbb{R}^2$. A sparse variant of REACHER-V2 is  created by imposing a  step reward of 1 when the distance between the arm and the target is less than $0.06$, and $0$ otherwise \cite{matheron2019problem}. 
 The target is fixed to a position of $[0.19,  0]$ and  removed distance to the target from the observations.
In addition, the reward structure is  sparsified by removing the control penalty.  We implement the proposed algorithm using policies with $\alpha=1,2$. Here we choose the mode of the policy: $\mathcal{S} \to \mathcal{A}$ to be a neural network that has
two hidden layers. Each hidden layer consists of 10 neurons  with tanh being  the  activation function.  Other parameters of the Algorithm are  same as the previous experiments. Fig. \ref{fig:J6} shows that HPG, by virtue of its extreme action selection according to a heavy-tailed distribution, causes more jumps in the state space, which results in better exploration of the environment. The result is faster policy improvement than vanilla PG.